\documentclass{article} %
\usepackage{iclr2027_conference,times}
\usepackage{microtype}

\usepackage{amsmath,amsfonts,bm}

\def\eqref#1{equation~\ref{#1}}

\def\1{\bm{1}}

\DeclareMathAlphabet{\mathsfit}{\encodingdefault}{\sfdefault}{m}{sl}
\SetMathAlphabet{\mathsfit}{bold}{\encodingdefault}{\sfdefault}{bx}{n}

\usepackage{graphicx}
\usepackage{booktabs}
\usepackage{array}
\usepackage{multirow}
\usepackage{amsmath,amssymb,amsthm}
\makeatletter
\renewcommand{\eqref}[1]{\textup{\tagform@{\ref{#1}}}}
\makeatother
\usepackage{bm}
\usepackage{algorithm}
\usepackage{algorithmic}
\usepackage{hyperref}
\usepackage{url}
\hypersetup{
  pdftitle={CLEARMIND: Closed-Form Laplacian Embeddings of Arakelov-Green Resistance for Morphological Intrinsic Neuronal Distances},
  pdfauthor={Yuyang Zhang, Weihan Xu, Xuehai Zhou, Shucheng Cao, Qihuang Zhang}
}

\theoremstyle{plain}
\newtheorem{theorem}{Theorem}[section]
\newtheorem{proposition}[theorem]{Proposition}
\newtheorem{lemma}[theorem]{Lemma}
\newtheorem{corollary}[theorem]{Corollary}
\theoremstyle{definition}
\newtheorem{definition}[theorem]{Definition}
\theoremstyle{remark}
\newtheorem{remark}[theorem]{Remark}

\title{CLEARMIND: Closed-Form Laplacian\\ Embeddings of Arakelov--Green\\ Resistance for Morphological\\ Intrinsic Neuronal Distances}

\author{Yuyang Zhang, Weihan Xu, Xuehai Zhou, Shucheng Cao \& Qihuang Zhang\thanks{Corresponding author.} \\
McGill University \\
Montreal, QC, Canada \\
{\small\texttt{\{yuyang.zhang, weihan.xu, xuehai.zhou, shucheng.cao\}@mail.mcgill.ca}} \\
{\small\texttt{qihuang.zhang@mcgill.ca}}}

\begin{document}

\maketitle

\begin{abstract}
Tropical geometry turns a metric graph into a flat torus, its tropical Jacobian, and measures distances between the images of its points by a closest vector problem on the period lattice. We prove that for points of the graph this problem is solved by geodesics: the squared tropical polarization distance equals the path metric minus the Arakelov--Green distance, an effective-resistance kernel given in closed form by the Laplacian pseudoinverse. Equivalently, the Arakelov--Green distance and the path metric are the minimal energies of real and of integral unit flows, and the squared tropical distance is exactly their integrality gap. The full tropical distance matrix is thus computable in cubic time. On this identity we build CLEARMIND, a training-free descriptor of 3D neuronal morphology. A reconstruction is reduced to its branching skeleton, neurite tips near the soma are joined to it, short bridges are contracted, and the Arakelov--Green matrix of the result is summarized by its spectrum. Every step has an exact algebraic description; for instance, the period matrix records the shared soma-to-tip path lengths, and the spectrum has exactly one positive eigenvalue, so its absolute values determine it. The signature needs no lattice search, is invariant to vertex order, rigid motions and subdivision, and is Lipschitz in the edge lengths. On ACT-4, JML-4 and BIL-6, appending the 64-dimensional spectrum improves a point-cloud GNN and MorphVAE on every dataset, by up to 20.3 points; a Tree-LSTM with the spectrum outperforms every reimplemented baseline; and an MLP on the spectrum alone is a competitive morphology classifier that surpasses every reimplemented deep baseline on ACT-4. On BREC, a GIN with Arakelov--Green eigenvector encodings distinguishes 70.0\% of the pairs beyond the 1-WL limit, including 33\% of the CFI pairs, more than PPGN (23\%) and I$^2$-GNN (21\%), and the spectrum alone separates 217 of the 400 pairs without training.
\end{abstract}

\section{Introduction}

A reconstructed neuron is a spatial tree: thousands of sample points in $\mathbb{R}^3$, joined by the branching pattern of its dendrites and axon. Learning cell types from such trees requires representations that see both the topology of the arbor and its geometry in space, and that do not depend on how the vertices are numbered or where the neuron sits. Message-passing networks on the tree are bounded by the 1-Weisfeiler--Leman (1-WL) test and see only tree adjacency \citep{xu2018how,morris2019weisfeiler}; persistence-based descriptors summarize branching through barcodes that must be vectorized before learning \citep{kanari2018topological,adams2017persistence}.

Tropical geometry offers a canonical alternative. Every metric graph $\Gamma$ carries a flat torus, its tropical Jacobian $\operatorname{Jac}(\Gamma)=\mathbb{R}^g/\Lambda$, together with the Abel--Jacobi map $\Phi:\Gamma\to\operatorname{Jac}(\Gamma)$ and the polarization metric \citep{mikhalkin2008tropical,baker2011metric}. The resulting tropical polarization distance between $\Phi(x)$ and $\Phi(y)$ measures a displacement in the graph against all of its loops at once. \citet{cao2025computing} made this geometry computable: they express the distance as a closest vector problem (CVP) on the period lattice and solve it with exact enumeration, mixed-integer programming, or lattice reduction followed by Babai's algorithm \citep{Babai1986}.

We show that for points of the graph no lattice search is needed. Our central result (Theorem~\ref{thm:main}) is the identity
\begin{equation*}
d_{\mathrm{Trop}}\big(\Phi(x),\Phi(y)\big)^2\;=\;d_\Gamma(x,y)\;-\;M_{AG}(x,y),
\end{equation*}
where $d_\Gamma$ is the path metric and $M_{AG}$ is the Arakelov--Green distance, read off in closed form from the Laplacian pseudoinverse $L^+$. In terms of flows, $M_{AG}(x,y)$ and $d_\Gamma(x,y)$ are the minimal energies of real and of integral unit flows from $x$ to $y$, and the squared tropical distance is exactly the gap between them: real flows can split across the loops of the graph, integral flows cannot, and the tropical Jacobian measures what splitting saves. Geodesic lifts solve the CVP, and the full matrix of tropical distances between all vertices follows from one pseudoinverse and one all-pairs shortest-path computation in $\mathcal{O}(|V|^3)$ time, without lattice reduction, enumeration, or integer programming (Corollary~\ref{cor:complexity}).

CLEARMIND turns this identity into a descriptor of neuronal morphology. The tropical Jacobian of a tree is a point, so we close loops where the arbor folds back towards the soma. A reconstruction is reduced to its branching skeleton by Kron reduction, neurite tips that lie within $\epsilon$ of the soma are joined to it, bridges shorter than $\tau$ are contracted, and the AG matrix of the result is summarized by its eigenvalues (Figure~\ref{fig:transformation_pipeline}). Each step has an exact algebraic description: the period matrix of the augmented skeleton is the matrix of shared soma-to-tip path lengths, and the AG distance is the tree path metric corrected by the loops that proximal tips close. The signature is invariant to vertex order, rigid motions and subdivision, and it is Lipschitz in the edge lengths (Proposition~\ref{prop:stability}).

Our contributions are:
\begin{itemize}
\item \textbf{Exact tropical distances in closed form.} We prove that the squared tropical polarization distance between Abel--Jacobi images of two points equals the path metric minus the Arakelov--Green distance, identify it as the integrality gap of Thomson's principle, and show that the geodesic lifts are exactly the solutions of the CVP (Theorem~\ref{thm:main}). This extends the single-edge formula of \citet{baker2011metric} to all pairs of points, sharpens the H\"older bound of \citet{cao2025computing} to an identity, and characterizes when lattice rounding fails: exactly when the rounded lift covers more than half of some loop (Proposition~\ref{prop:voronoi}).
\item \textbf{A morphology descriptor with a complete proof chain.} We give exact statements for every step of CLEARMIND (Appendix~\ref{sec:appendix}): Kron reduction preserves the AG distance (Theorem~\ref{thm:kron}), the period matrix records shared soma-to-tip path lengths (Eq.~\eqref{eq:period_main}), contracting a bridge lowers the AG distance between points on opposite sides by exactly its length and leaves the tropical Jacobian unchanged (Proposition~\ref{prop:bridge}), and the spectrum has exactly one positive eigenvalue, so its absolute values lose no information (Theorem~\ref{thm:spectral}).
\item \textbf{Gains on every dataset and expressivity beyond 1-WL.} Appending the 64-dimensional CLEARMIND spectrum improves a point-cloud GNN and MorphVAE on each of ACT-4, JML-4 and BIL-6, by up to $20.3$ points, a Tree-LSTM with the spectrum outperforms every reimplemented baseline on all three datasets, and an MLP on the spectrum alone is a competitive morphology classifier that surpasses every reimplemented deep baseline on ACT-4 ($63.79\%$). On BREC \citep{wang2024brec}, a GIN with AG eigenvector encodings distinguishes $70.0\%$ of the pairs beyond the 1-WL limit, including $33\%$ of the CFI pairs, more than PPGN ($23\%$) and I$^2$-GNN ($21\%$), and the training-free spectrum separates $217$ of the $400$ pairs (Section~\ref{subsec:brec}).
\end{itemize}

\section{Related Work}
\label{sec:related_work}

\textbf{Neuronal morphology representations.} Learned descriptors of neuron morphology use random walks (MorphVAE, \citealp{laturnus2021morphvae}), point clouds (MorphoGNN, \citealp{zhu2023morphognn}), tree recursions (the Tree-LSTM architecture, \citealp{Tai2015}; TreeMoCo, \citealp{chen2022treemoco}), graph self-supervision \citep{weis2023graphdino}, and transformers with persistence images (GraPHFormer, \citealp{shah2026graphformer}) or spatial encodings (SGTMorph, \citealp{sheng2025sgtmorph}). The Topological Morphology Descriptor summarizes branching by a radial persistence barcode \citep{kanari2018topological,kanari2019objective}. CLEARMIND is a training-free, exactly computable metric invariant that complements all of these methods and can be appended to any of the encoders.

\textbf{Expressivity, spectra and effective resistance.} Message passing is bounded by 1-WL \citep{xu2018how,morris2019weisfeiler}. Laplacian eigenvectors as positional encodings \citep{dwivedi2023benchmarking} suffer from sign and basis ambiguities \citep{lim2023sign,ma2023laplacian}. Effective resistance \citep{klein1993resistance} lifts message passing beyond 1-WL while remaining bounded by 2-FWL \citep{zhang2023rethinking}, and resistance features have been injected into MPNNs \citep{velingker2023affinity}. We use the resistance matrix $M_{AG}$ of a metric graph $\Gamma$ built from the morphology, prove that $d_\Gamma-M_{AG}$ is the squared tropical polarization distance (Theorem~\ref{thm:main}), and summarize $M_{AG}$ by a permutation-invariant spectrum whose sign structure is fixed by the Euclidean-distance-matrix property \citep{schoenberg1935remarks}.

\textbf{Tropical Jacobians and Arakelov--Green functions.} The tropical Jacobian of a metric graph and its Abel--Jacobi map \citep{mikhalkin2008tropical,kotani2000jacobian} were studied metrically by \citet{baker2011metric}, who showed that $\Phi$ contracts exactly the bridges and maps an edge $e$ of length $\ell_e$ to a segment of length $\sqrt{\ell_eF(e)}$ in the polarization metric, with $F(e)$ its Foster coefficient; Theorem~\ref{thm:main} extends this local description to any two points of the graph (Theorem~\ref{thm:exact_trop} and Remark~\ref{rem:special_cases}). \citet{cao2025computing} turned the Abel--Jacobi transform and the tropical polarization distance into computational tools, and \citet{cao2025torelli} built graph kernels from the tropical Torelli map. On metrized graphs, the Arakelov--Green functions of \citet{chinburg1993capacity} and \citet{zhang1993admissible} are expressed through the resistance function \citep{baker2007harmonic}, and the orthogonal decomposition of $1$-chains underlies Thomson's principle \citep{biggs1997potential,dejong2022crossratios}. CVP is NP-hard for general lattices \citep{vanEmdeBoas1981} but polynomial for lattices of integral flows \citep{mccormick2021polynomial}; our Theorem~\ref{thm:main} shows that for targets coming from points of the graph it has an explicit solution, given by the geodesic lift.

\section{Methodology}
\label{sec:methodology}

CLEARMIND maps a 3D spatial tree to a permutation-invariant and $E(3)$-invariant spectral signature in four steps: Kron reduction to the branching skeleton, cycle augmentation at the soma, a quotient by short bridges, and the evaluation of the Arakelov--Green (AG) distance, whose spectrum is the signature (Figure~\ref{fig:transformation_pipeline}, Algorithm~\ref{alg:ag_feature_extraction}). Section~\ref{subsec:ag_measure} proves that the AG distance and the path metric together give the tropical polarization distance in closed form. All proofs are in Appendix~\ref{sec:appendix}; Table~\ref{tab:proof_roadmap} maps the statements of the main text to their proofs.

\begin{figure}[t]
\begin{center}
\includegraphics[width=\linewidth]{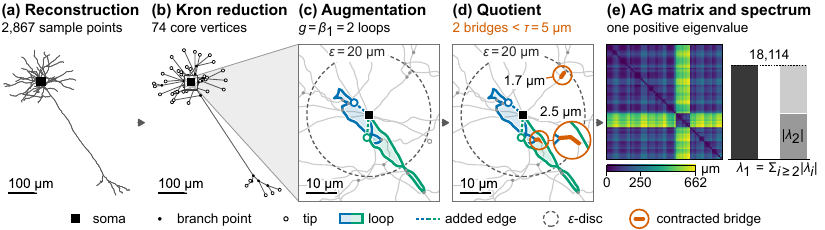}
\end{center}
\caption{CLEARMIND on a spiny layer-4 neuron of ACT-4 (primary visual cortex), with $\epsilon=20\,\mu$m and $\tau=5\,\mu$m.
\textbf{(a)} The reconstruction has $2{,}867$ sample points; the square marks the soma.
\textbf{(b)} Kron reduction keeps the $74$ branch points, tips and the soma (Theorem~\ref{thm:kron}); every unbranched segment becomes one edge whose length is the segment's path length. The box marks the soma region enlarged in (c) and (d).
\textbf{(c)} Two dendritic tips lie within $\epsilon$ of the soma. Joining them to the soma (dashed) closes two loops (coloured), so $g=\beta_1=2$.
\textbf{(d)} The two bridges shorter than $\tau$ (vermilion, ringed, with their lengths) are contracted; the loops are untouched. The lower bridge is also shown enlarged: it passes behind the green loop without meeting it.
\textbf{(e)} The AG matrix on the $74$ core vertices, in depth-first order (the long descending neurite forms the bright band), and its spectrum: exactly one eigenvalue, $\lambda_1$ (dark bar), is positive, and it equals the sum of the absolute values of all the others (stacked: $|\lambda_2|$, then the rest; Theorem~\ref{thm:spectral}).}
\label{fig:transformation_pipeline}
\end{figure}

\subsection{Metric Graph Formulation and Kron Reduction}
\label{subsec:retraction}

Let $\mathcal{T}=(V,E)$ be a 3D spatial tree with embedding $\iota:V\to\mathbb{R}^3$ and root $v_{\text{root}}$ at the soma. We regard it as a metric graph $\Gamma$ with edge lengths $\ell(u,v)=\|\iota(u)-\iota(v)\|_2$ and work on its largest connected component $\mathcal{T}_{\text{LCC}}$; from now on $V$, $E$ and $\deg$ refer to $\mathcal{T}_{\text{LCC}}$. For a metric graph with weighted Laplacian $L=\operatorname{diag}(W\mathbf{1})-W$, $W_{u,v}=1/\ell(u,v)$, the \emph{Arakelov--Green (AG) distance} is
\begin{equation}
\label{eq:mag}
M_{AG}(x,y)=L^+_{x,x}+L^+_{y,y}-2L^+_{x,y},
\end{equation}
where $L^+$ is the Moore--Penrose pseudoinverse. The AG distance is the effective resistance of $\Gamma$ read as an electrical network with resistances $\ell(u,v)$, and it is the Arakelov--Green function of $\Gamma$ written as a distance: for the canonical measure $\mu_{\mathrm{can}}$ of $\Gamma$, $g_{\mu_{\mathrm{can}}}+\frac12M_{AG}$ is constant, and for every measure $\mu$ of total mass one, $M_{AG}(x,y)=g_\mu(x,x)+g_\mu(y,y)-2g_\mu(x,y)$ (Proposition~\ref{prop:arakelov}). By Thomson's principle, $M_{AG}(x,y)$ is the minimal energy $\sum_e\ell_ef_e^2$ of a unit flow from $x$ to $y$ (Theorem~\ref{thm:thomson}).

Reconstructions sample unbranched neurites densely. Kron reduction \citep{Dorfler2013} removes these samples without changing the AG distances between the vertices that remain.

\begin{theorem}[Kron reduction; Corollary~\ref{cor:kron_tree}]
\label{thm:kron}
Let $V_{\text{reg}}=\{v\neq v_{\text{root}}:\deg(v)=2\}$ and $V_{\text{core}}=V\setminus V_{\text{reg}}$. The Schur complement of $L$ with respect to $V_{\text{reg}}$ is the Laplacian of the branching skeleton $\Gamma_{\min}$, the tree on $V_{\text{core}}$ in which every unbranched segment is a single edge whose length is the path length of the segment. It preserves $M_{AG}$ on $V_{\text{core}}$, does not depend on the elimination order, and commutes with the cycle augmentation below.
\end{theorem}

The skeleton therefore depends on each unbranched neurite only through its path length and is unchanged by any subdivision of the reconstruction, and all subsequent computations cost $\mathcal{O}(|V_{\text{core}}|^3)$ instead of $\mathcal{O}(|V|^3)$; on ACT-4 and BIL-6 the median $|V|$ is $4{,}677$ and $18{,}966$, and the median $|V_{\text{core}}|$ is $125$ and $464$. On a tree, $M_{AG}$ coincides with the path metric (Corollary~\ref{cor:kron_tree}).

\subsection{Cycle Space Augmentation}
\label{subsec:augmentation}

A tree has genus zero, so its tropical Jacobian is a point. Dendrites, however, curve back towards the soma, and this 3D folding is invisible to the tree metric. Given $\epsilon>0$, we join every leaf $v\neq v_{\text{root}}$ of $\Gamma_{\min}$ with $a_v=\|\iota(v)-\iota(v_{\text{root}})\|_2<\epsilon$ to the root by an edge of length $a_v$, which gives $\Gamma_\epsilon=\Gamma_{\min}\cup E_\epsilon$. Each added edge closes one loop: $\beta_1(\Gamma_\epsilon)=|E_\epsilon|$, and the loops through the added edges, each run from the soma out to its tip and back, form a $\mathbb{Z}$-basis of $H_1(\Gamma_\epsilon;\mathbb{Z})$. We write them as the rows of a matrix $C$, and write $[\gamma]\in\{0,\pm1\}^E$ for the signed edge vector of a path $\gamma$. The period matrix $Q=CD_\ell C^\top$, with $D_\ell$ the diagonal matrix of edge lengths, can be read off from the morphology (Proposition~\ref{prop:augmentation}):
\begin{equation}
\label{eq:period_main}
Q_{uv}=h(u\wedge v)\ \ (u\neq v),\qquad Q_{vv}=h(v)+a_v,
\end{equation}
where $h$ is the path distance from the soma and $u\wedge v$ is the branch point at which the paths to the tips $u$ and $v$ separate. Thus $Q$ records the shared soma-to-tip path lengths of the proximal tips, and $Q\succeq\operatorname{diag}(a)\succ0$. A single proximity edge shortens exactly those AG distances whose tree paths run along the soma-to-tip path it closes, by the squared overlap divided by the length of the loop (Proposition~\ref{prop:rank_one}); several edges act jointly through $Q^{-1}$ (Eq.~\eqref{eq:morph_main}).

\subsection{Quotient by Short Bridges}
\label{subsec:quotient}

Short terminal branches often reflect reconstruction artefacts. A bridge of $\Gamma_\epsilon$ lies on no loop, so contracting it cannot change the cycle space. Given $\tau>0$, let $E_{<\tau}$ be the set of bridges of $\Gamma_\epsilon$ shorter than $\tau$, and let $\tilde\Gamma_\epsilon=\Gamma_\epsilon/\!\sim_\tau$ be the metric graph obtained by contracting them.

\begin{definition}[Quotient metric graph]
Two vertices satisfy $u\sim_\tau v$ if they are joined by a path of edges in $E_{<\tau}$; $\tilde\Gamma_\epsilon$ identifies each equivalence class to a single vertex.
\end{definition}

Contraction is exact: contracting a bridge $e$ lowers $M_{AG}(x,y)$ and the path distance between $x$ and $y$ by $\ell_e$ when $e$ separates them and leaves both unchanged otherwise, and it leaves the period matrix, the tropical Jacobian, the Abel--Jacobi map and every tropical distance unchanged; contracting the edges of $E_{<\tau}$ one at a time, in any order, yields $\tilde\Gamma_\epsilon$ (Proposition~\ref{prop:bridge}). In Algorithm~\ref{alg:ag_feature_extraction} we set $W_{u,v}=10^9$ on contracted bridges; by the same proposition this shifts each $M_{AG}(x,y)$ by exactly $10^{-9}$ per contracted bridge separating $x$ and $y$.

\subsection{The Tropical Jacobian and the Polarization Distance}
\label{subsec:polarization}

The graph $\tilde\Gamma_\epsilon$ has genus $g=\beta_1(\tilde\Gamma_\epsilon)$. Its tropical Jacobian is the flat torus $\operatorname{Jac}(\tilde\Gamma_\epsilon)=\mathbb{R}^g/\Lambda$, with lattice $\Lambda=Q\mathbb{Z}^g$ and the polarization metric $\langle t,t'\rangle=t^\top Q^{-1}t'$ \citep{mikhalkin2008tropical,baker2011metric}. The Abel--Jacobi map $\Phi$ sends a vertex $x$ to the period vector $t=CD_\ell[\gamma]$ of any path $\gamma$ from the root to $x$, modulo $\Lambda$. The tropical polarization distance between the images of $x$ and $y$ is (\citealp{cao2025computing}; Lemma~\ref{lem:period_coords})
\begin{equation}
\label{eq:dtrop}
d_{\mathrm{Trop}}(x,y)=\min_{n\in\mathbb{Z}^g}\big((t-Qn)^\top Q^{-1}(t-Qn)\big)^{1/2},\qquad t=CD_\ell[\gamma_{xy}],
\end{equation}
for any path $\gamma_{xy}$ from $x$ to $y$. The distance $d_{\mathrm{Trop}}$ weighs the displacement from $x$ to $y$ against every loop of the graph and does not depend on a single shortest path. Eq.~\eqref{eq:dtrop} is a Closest Vector Problem (CVP). CVP is NP-hard on general lattices \citep{vanEmdeBoas1981} and polynomial on lattices of integral flows \citep{mccormick2021polynomial}, and explicit pipelines solve it by enumeration, integer programming, or Babai's rounding \citep{Babai1986}, which lifts along a spanning tree and rounds $\hat n=\lfloor Q^{-1}t\rceil$ coordinatewise, with residual $\delta=Q^{-1}t-\hat n$ and value $\delta^\top Q\delta$ (Figure~\ref{fig:distance_evaluation}e).

\subsection{Closed Form and Exact Tropical Distance}
\label{subsec:ag_measure}

The targets of the CVP instances in Eq.~\eqref{eq:dtrop} are lifts of paths in the graph, and for such targets the problem can be solved in closed form.

\begin{theorem}[AG distance and exact tropical distance; Theorems~\ref{thm:closed_form} and~\ref{thm:exact_trop}]
\label{thm:main}
Let $x,y$ be vertices of $\tilde\Gamma_\epsilon$ and let $d_{\tilde\Gamma}$ be its path metric.
\begin{enumerate}
\item[(i)] \textbf{Lift invariance.} For every path $\gamma$ from $x$ to $y$, with length $\ell(\gamma)$ and period vector $t_\gamma=CD_\ell[\gamma]$,
\begin{equation}
\label{eq:closed_form_main}
M_{AG}(x,y)=\ell(\gamma)-t_\gamma^\top Q^{-1}t_\gamma .
\end{equation}
\item[(ii)] \textbf{Exact tropical distance.} A path $\gamma$ is a geodesic if and only if $t_\gamma^\top Q^{-1}t_\gamma=d_{\mathrm{Trop}}(x,y)^2$: the geodesic lifts are exactly the closest ones. Hence
\begin{equation}
\label{eq:exact_main}
d_{\mathrm{Trop}}(x,y)^2=d_{\tilde\Gamma}(x,y)-M_{AG}(x,y),
\end{equation}
and the solutions of the CVP in Eq.~\eqref{eq:dtrop} are exactly the lattice vectors that carry $t$ to a geodesic lift, one for each geodesic from $x$ to $y$ (Proposition~\ref{prop:voronoi}).
\item[(iii)] \textbf{Integrality gap.} $M_{AG}(x,y)$ and $d_{\tilde\Gamma}(x,y)$ are the minimal energies $\sum_e\ell_ef_e^2$ of real and of integral unit flows from $x$ to $y$, respectively; the squared tropical distance is the gap between the two.
\end{enumerate}
\end{theorem}

Theorem~\ref{thm:main} holds for any two points of any compact connected metric graph $\Gamma$ and any $\mathbb{Z}$-basis of $H_1(\Gamma;\mathbb{Z})$ (Theorems~\ref{thm:closed_form} and~\ref{thm:exact_trop}, Lemma~\ref{lem:subdivision}). The AG distance is the part of the path energy that no lattice translation can change: replacing $\gamma$ by any other path moves $t_\gamma$ by a vector of $\Lambda$ and changes $\ell(\gamma)$ and $t_\gamma^\top Q^{-1}t_\gamma$ by the same amount, so no integer search is needed. Together with the path metric, $M_{AG}$ yields the exact matrix of tropical polarization distances of all vertices in $\mathcal{O}(|V_{\text{core}}|^3)$ time (Corollary~\ref{cor:complexity}), a problem previously attacked with lattice reduction and mixed-integer solvers \citep{cao2025computing}. Eq.~\eqref{eq:exact_main} also sharpens their bound $d_{\mathrm{Trop}}^2\le d_\Gamma$ to an identity (Remark~\ref{rem:nphard}) and, along a single edge, recovers the Foster-coefficient formula of \citet{baker2011metric} for $\Phi$ (Remark~\ref{rem:special_cases}). Applied to the tree path, Eq.~\eqref{eq:closed_form_main} gives the morphological closed form
\begin{equation}
\label{eq:morph_main}
M_{AG}(x,y)=d_{\tilde T}(x,y)-s_{xy}^\top Q^{-1}s_{xy},
\end{equation}
where $d_{\tilde T}$ is the path metric of the contracted skeleton and $(s_{xy})_v$ is the signed length along which the path from $x$ to $y$ follows the path from the soma to the proximal tip $v$ (Corollary~\ref{cor:morph}). Babai rounding, in contrast, depends on the chosen cycle basis and returns a closest lattice vector exactly when the rounded lift is the lift of a geodesic (Propositions~\ref{prop:babai} and~\ref{prop:voronoi}).

\begin{figure}[t]
\begin{center}
\includegraphics[width=\linewidth]{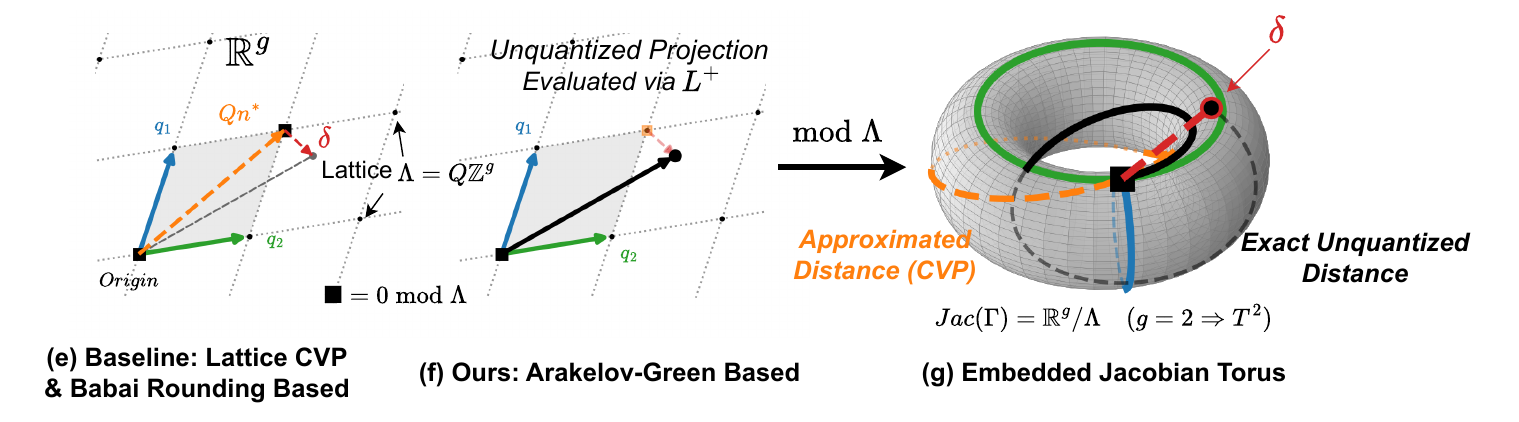}
\end{center}
\caption{Distances on the tropical Jacobian.
\textbf{(e)} Explicit pipelines lift a pair of points along a spanning tree and round the period vector to the lattice $\Lambda=Q\mathbb{Z}^g$ (Babai); the rounded lattice point (square) is in general not the closest one, and the residual $\delta$ (red) inflates the distance.
\textbf{(f)} CLEARMIND evaluates $M_{AG}=\ell(\gamma)-t_\gamma^\top Q^{-1}t_\gamma$ from $L^+$; this quantity is identical for all lifts, and the geodesic lift is a closest lattice translate, so the exact distance follows as $d_{\mathrm{Trop}}^2=d_\Gamma-M_{AG}$ (Theorem~\ref{thm:main}).
\textbf{(g)} On the torus $\operatorname{Jac}(\Gamma)=\mathbb{R}^g/\Lambda$, rounding can select a longer representative (orange) than the exact geodesic class (black).}
\label{fig:distance_evaluation}
\end{figure}

\subsection{Spectral Feature Extraction}
\label{subsec:spectral}

Since $L^+$ is positive semidefinite, it factors as $L^+=XX^\top$, and with $b_x$ the rows of $X$ we get $M_{AG}(x,y)=\|b_x-b_y\|_2^2$: $M_{AG}$ is a Euclidean distance matrix. Its spectrum $\Sigma=\{\lambda_1\ge\dots\ge\lambda_{|V_{\text{core}}|}\}$ is invariant under vertex permutations, and it has a rigid sign pattern: $\lambda_1>0>\lambda_2$, so exactly one eigenvalue is positive, all others are negative, and $\lambda_1=\sum_{i\ge2}|\lambda_i|$ (Theorem~\ref{thm:spectral}). Taking absolute values therefore loses no information. The signature is
\begin{equation*}
\mathcal{S}_\Gamma=\operatorname{TopK}\big(\operatorname{SortDescending}(|\Sigma|),K\big)\in\mathbb{R}^K,
\end{equation*}
zero-padded when $|V_{\text{core}}|<K$. It is invariant under relabelling, rigid motions and reflections, and subdivision of the reconstruction (Theorem~\ref{thm:spectral}). At fixed topology it is homogeneous of degree one and Lipschitz in the edge lengths: every entry of $M_{AG}$ is nondecreasing and concave in $\ell$ and changes by at most $\|\ell-\ell'\|_1$, and $\|\mathcal{S}_\Gamma-\mathcal{S}_{\Gamma'}\|_\infty\le(|V_{\text{core}}|-1)\,\|\ell-\ell'\|_1$ (Proposition~\ref{prop:stability}). The signature is concatenated with node or graph features in neural architectures (Section~\ref{subsec:morphology}).

\begin{algorithm}[t]
\footnotesize
\caption{CLEARMIND spectral signature}
\label{alg:ag_feature_extraction}
\textbf{Input}: Spatial tree $\mathcal{T}=(V, E, \iota)$, proximity threshold $\epsilon$, contraction threshold $\tau$, signature dimension $K$. \\
\textbf{Output}: Spectral signature $\mathcal{S}_{\Gamma} \in \mathbb{R}^K$.
\begin{algorithmic}[1]
\STATE \textbf{Kron reduction.} $\Gamma_{\min} \gets \text{KronReduction}(\mathcal{T}_{\text{LCC}})$
\STATE \textbf{Cycle augmentation.} $E_{\epsilon} \gets \{ (v, v_{\text{root}}) : v\neq v_{\text{root}} \text{ a leaf of } \Gamma_{\min},\ \|\iota(v) - \iota(v_{\text{root}})\|_2 < \epsilon \}$, with $\ell(v,v_{\text{root}})=\|\iota(v)-\iota(v_{\text{root}})\|_2$; $\Gamma_{\epsilon} \gets \Gamma_{\min} \cup E_{\epsilon}$
\STATE \textbf{Quotient.} $W_{u,v}\gets10^9$ if $\{u,v\}$ is a bridge of $\Gamma_\epsilon$ with $\ell(u,v)<\tau$, else $W_{u,v}\gets 1/\ell(u,v)$
\STATE \textbf{AG distance.} $L \gets \operatorname{diag}(W \mathbf{1}) - W$; $M_{AG} \gets \operatorname{diag}(L^+)\mathbf{1}^\top + \mathbf{1}\operatorname{diag}(L^+)^\top - 2L^+$
\STATE \textbf{Spectrum.} \textbf{return} $\mathcal{S}_{\Gamma} \gets \operatorname{TopK}(\operatorname{SortDescending}(|\operatorname{Spectrum}(M_{AG})|), K)$
\end{algorithmic}
\end{algorithm}

\section{Experiments}
\label{sec:experiments}

We evaluate CLEARMIND as a source of expressive structural features on BREC, as a structural prior for classifying 3D neuronal morphologies, and against an explicit lattice pipeline.

\subsection{Expressivity on BREC}
\label{subsec:brec}

BREC \citep{wang2024brec} contains 400 pairs of non-isomorphic graphs that 1-WL cannot distinguish. We use the AG distance in two ways. \emph{AG-GIN} appends the sign-invariant absolute values of the 16 leading eigenvectors of $M_{AG}$ to constant node features and trains a Graph Isomorphism Network \citep{xu2018how} with jumping knowledge \citep{xu2018jk}; the network is evaluated with the Reliable Paired Comparison of BREC. The \emph{AG spectrum} $\mathcal{S}_\Gamma$ is used directly as a graph invariant, without any training: a pair is separated when the two spectra differ (Appendix~\ref{app:brec_protocol}).

\begin{table}[t]
\caption{Pair-distinguishing accuracy on BREC. Every pair is indistinguishable by 1-WL. Baseline rows are reported by \citet[Table~2]{wang2024brec}.}
\label{tab:brec_results}
\centering
\small
\setlength{\tabcolsep}{4pt}
\begin{tabular}{lccccc}
\toprule
\textbf{Model} & \textbf{Basic} & \textbf{Regular} & \textbf{Extension} & \textbf{CFI} & \textbf{Total} \\
\midrule
\textit{Number of pairs} & (60) & (140) & (100) & (100) & (400) \\
3-WL test & 100\% & 35.7\% & 100\% & 60.0\% & 67.5\% \\
\midrule
Graphormer \citep{ying2021transformers} & 26.7\% & 8.6\% & 41.0\% & 10.0\% & 19.8\% \\
NGNN \citep{zhang2021nested} & 98.3\% & 34.3\% & 59.0\% & 0.0\% & 41.5\% \\
PPGN \citep{maron2019provably} & 100\% & 35.7\% & 100\% & 23.0\% & 58.2\% \\
KP-GNN \citep{feng2022khop} & 100\% & 75.7\% & 98.0\% & 11.0\% & 68.8\% \\
I$^2$-GNN \citep{huang2023boosting} & 100\% & 71.4\% & 100\% & 21.0\% & 70.2\% \\
\midrule
\textbf{AG-GIN (ours)} & 100\% & 62.1\% & 100\% & \textbf{33.0\%} & 70.0\% \\
\textbf{AG spectrum, no training (ours)} & 100\% & 35.7\% & 100\% & 7.0\% & 54.3\% \\
\bottomrule
\end{tabular}
\end{table}

AG-GIN separates every Basic and Extension pair and $70.0\%$ of all pairs, and reaches $33.0\%$ on the CFI category, the highest CFI accuracy among the trained models of Table~\ref{tab:brec_results}, above higher-order and subgraph GNNs such as PPGN ($23.0\%$) and I$^2$-GNN ($21.0\%$). The AG features, computed once per graph in a single preprocessing step, account for every separated pair: a GIN is bounded by 1-WL \citep{xu2018how}, so on constant node features alone it separates none of the $400$ pairs. Even without any trained parameter, the AG spectrum separates all 60 Basic, all 50 simple-regular, all 100 Extension and 7 CFI pairs, $217$ in total, more than Graphormer and NGNN; on Basic, Regular and Extension it matches the 3-WL test. The computation is deterministic, and every threshold between $10^{-12}$ and $10^{-8}$ on the relative difference of the spectra yields the same $217$ pairs (Appendix~\ref{app:brec_protocol}). The behaviour of the spectrum on the remaining regular pairs is explained exactly: on distance-regular graphs, including strongly regular and 4-vertex-condition graphs, $M_{AG}$ is a strictly increasing function of the graph distance, given in closed form by the intersection array, so two such graphs with the same intersection array have the same AG spectrum (Proposition~\ref{prop:drg}).

\subsection{Classification of 3D Neuronal Morphologies}
\label{subsec:morphology}

\textbf{Datasets.} We use the benchmarks of \citet{chen2022treemoco}: \textbf{ACT-4} (Allen Cell Types), $495$ manually labeled neurons from the mouse visual cortex \citep{gouwens2019classification}; \textbf{JML-4} (Janelia MouseLight), $408$ automatically labeled, complete projection neurons \citep{winnubst2019reconstruction}; and \textbf{BIL-6} (BICCN fMOST), $1{,}279$ manually labeled, complete neurons from cortex, striatum and thalamus, of which the $1{,}200$ with released SWC reconstructions are used \citep{peng2021morphological}.

\textbf{Protocol.} We use the fixed fold assignment of \citet{chen2022treemoco}, training on folds 0--7 and evaluating on folds 8--9, and report the mean and standard deviation over five random seeds at the best epoch; optimizers, schedules and architectures are listed in Appendix~\ref{app:protocols}. Spectra use $\epsilon=20\,\mu$m, $\tau=5\,\mu$m and $K=64$.

\textbf{Fusion.} The spectrum $\mathbf{s}\in\mathbb{R}^{64}$ is concatenated with a backbone's global representation: through one linear layer with batch normalization and ReLU, $P_{\text{spec}}(\mathbf{s})$, with the root state of a Tree-LSTM \citep{Tai2015} and the pooled EdgeConv features of MorphoGNN \citep{zhu2023morphognn}, and directly with the frozen $32$-dimensional embedding of MorphVAE \citep{laturnus2021morphvae}. The descriptor itself has no trainable parameters.

\begin{table}[t]
\caption{Classification accuracy (\%) on 3D morphological benchmarks. Baselines: TRNN \citep{zhang2021trnn}, TreeMoCo \citep{chen2022treemoco}, GraPHFormer \citep{shah2026graphformer}, MorphoGNN \citep{zhu2023morphognn}, MorphVAE \citep{laturnus2021morphvae} and SGTMorph \citep{sheng2025sgtmorph}. All numbers are our reimplementation under the same pipeline, split and seeds, except SGTMorph (quoted under its native protocol, with different preprocessing). Small numbers are gains over the matched backbone; the spectrum-only MLP has 11k parameters.}
\label{tab:main_results}
\centering
\footnotesize
\setlength{\tabcolsep}{3pt}
\begin{tabular}{lccc}
\toprule
\textbf{Method} & \textbf{ACT-4} & \textbf{JML-4} & \textbf{BIL-6} \\
\midrule
\multicolumn{4}{l}{\textit{Sequence and tree-based}} \\
TRNN (translation-invariant) & $33.70{\pm}1.45$ & $42.90{\pm}1.62$ & $31.00{\pm}1.80$ \\
TreeMoCo ($k$-NN) & $60.21{\pm}1.73$ & $63.16{\pm}0.93$ & $77.36{\pm}0.89$ \\
TreeMoCo (fine-tuned) & $57.05{\pm}1.73$ & $72.89{\pm}1.50$ & $87.19{\pm}1.13$ \\
GraPHFormer (tree branch) & $54.52{\pm}1.37$ & $73.33{\pm}2.67$ & $82.97{\pm}0.54$ \\
\textbf{Tree-LSTM + AG (ours)} & $\mathbf{69.68{\pm}3.61}$ & $\mathbf{74.29{\pm}2.23}$ & $\mathbf{93.55{\pm}0.56}$ \\
\midrule
\multicolumn{4}{l}{\textit{Point cloud}} \\
MorphoGNN & $52.21{\pm}2.27$ & $65.19{\pm}2.52$ & $88.45{\pm}0.75$ \\
\textbf{MorphoGNN + AG (ours)} & $\mathbf{54.32{\pm}2.71}$\,{\scriptsize$+2.11$} & $\mathbf{76.36{\pm}0.97}$\,{\scriptsize$+11.17$} & $\mathbf{92.15{\pm}0.45}$\,{\scriptsize$+3.70$} \\
\midrule
\multicolumn{4}{l}{\textit{Generative}} \\
MorphVAE & $44.63{\pm}1.76$ & $51.95{\pm}5.19$ & $55.04{\pm}1.97$ \\
\textbf{MorphVAE + AG (ours)} & $\mathbf{48.84{\pm}1.60}$\,{\scriptsize$+4.21$} & $\mathbf{61.56{\pm}4.07}$\,{\scriptsize$+9.61$} & $\mathbf{75.29{\pm}0.61}$\,{\scriptsize$+20.25$} \\
\midrule
\multicolumn{4}{l}{\textit{Spectrum only}} \\
\textbf{MLP on $\mathcal{S}_\Gamma$ (ours)} & $\mathbf{63.79{\pm}1.20}$ & $60.26{\pm}1.72$ & $73.31{\pm}0.86$ \\
\midrule
\multicolumn{4}{l}{\textit{Quoted spatial transformer (different preprocessing)}} \\
SGTMorph & -- & $72.40$ & $88.90$ \\
\bottomrule
\end{tabular}
\end{table}

\textbf{Results.} The spectrum improves every backbone it is added to. It raises MorphoGNN by $2.1$, $11.2$ and $3.7$ points and MorphVAE by $4.2$, $9.6$ and $20.3$ points on ACT-4, JML-4 and BIL-6 (Table~\ref{tab:main_results}). The two MorphVAE classifiers share the frozen embedding, the architecture and the training schedule and differ only in the $64$ appended spectral inputs (Appendix~\ref{appendix:morphvae_integration}), so the gains over MorphVAE are due to the spectrum alone. A Tree-LSTM with the spectrum, with $97{,}792$ parameters, outperforms every baseline on all three datasets, including fine-tuned TreeMoCo with $335{,}494$ parameters, and on BIL-6 its $93.55\%$ exceeds the $87.19\%$ of TreeMoCo and the published $88.90\%$ of SGTMorph. On every dataset the best accuracy is attained by a model with the spectrum: Tree-LSTM + AG on ACT-4 and BIL-6, and MorphoGNN + AG on JML-4 ($76.36\%$). The spectrum is informative on its own: an MLP with $11$k parameters on $\mathcal{S}_\Gamma$ alone reaches $63.79\%$ on ACT-4 (Appendix~\ref{appendix:mlp_integration}), above every reimplemented deep baseline, including contrastively pre-trained TreeMoCo ($60.21\%$). The largest gains occur for the two backbones that do not operate on the tree itself, point clouds and random walks, which is where an exact summary of the arbor's path geometry adds the most.

\textbf{Transformer-based architectures.} The graph branch of GraPHFormer \citep{shah2026graphformer} is a TreeLSTM-Double encoder; the full model adds persistence images, whose generator is not part of the public release. Reproduced under our split (Appendix~\ref{appendix:graphformer_reproduction}), the graph branch reaches $54.52\%$, $73.33\%$ and $82.97\%$, and Tree-LSTM + AG outperforms it on all three datasets. SGTMorph \citep{sheng2025sgtmorph} normalizes every neuron into a unit box; a positional probe on the normalized soma coordinates alone reaches $61.0\%$ on JML-4 and $74.8\%$ on BIL-6 (Appendix~\ref{appendix:sgtmorph_probe}).

\subsection{Closed Form versus Lattice Rounding}
\label{subsec:explicit_vs_implicit}

Theorem~\ref{thm:main} makes lattice search unnecessary, and Proposition~\ref{prop:voronoi} predicts exactly when rounding fails: when the rounded lift is not a geodesic, that is, when it runs along more than half of some loop. On random graphs, rounding a spanning-tree lift misses the closest lattice vector for $11.4\%$ of vertex pairs (Remark~\ref{rem:numerical_check}). We compare the closed-form AG signature with a descriptor built from Babai-rounded tropical distances on spanning-tree lifts, with a Floyd--Warshall distance on the bridge backbone (Appendix~\ref{appendix:proof_tropical_approx}), under the same frozen $k$-NN classifier at $\epsilon=50\,\mu$m and $\tau=10\,\mu$m, where up to $106$ (ACT-4), $35$ (JML-4) and $30$ (BIL-6) neurite tips lie within $\epsilon$ of the soma.

\begin{table}[t]
\caption{Closed-form AG signature versus the explicit lattice pipeline (period matrix and Babai rounding) with a frozen $k$-NN classifier at $\epsilon=50$, $\tau=10$. Time is the preprocessing cost in seconds.}
\label{tab:cvp_vs_ag_selected}
\centering
\small
\setlength{\tabcolsep}{4pt}
\begin{tabular}{lcccccc}
\toprule
& \multicolumn{3}{c}{\textbf{Babai-rounded lattice pipeline}} & \multicolumn{3}{c}{\textbf{Closed-form AG (ours)}} \\
\cmidrule(lr){2-4}\cmidrule(lr){5-7}
\textbf{Dataset} & Acc & Macro-F1 & Time (s) & Acc & Macro-F1 & Time (s) \\
\midrule
ACT-4 & 47.37 & 45.20 & 109.2 & \textbf{49.47} & \textbf{47.79} & \textbf{16.8} \\
JML-4 & \textbf{49.35} & 32.49 & 376.7 & 48.05 & \textbf{36.21} & \textbf{103.0} \\
BIL-6 & 52.89 & 41.27 & 908.1 & \textbf{64.05} & \textbf{46.65} & \textbf{219.1} \\
\bottomrule
\end{tabular}
\end{table}

The closed form has two advantages. \textbf{Exactness and basis independence.} Rounding $\hat n=\lfloor Q^{-1}t\rceil$ depends on the cycle basis, already through the choice of spanning tree (Proposition~\ref{prop:babai}), whereas the closed form is exact and basis free (Theorem~\ref{thm:closed_form}); it yields higher macro-F1 on all three datasets and higher accuracy on ACT-4 ($+2.1$) and BIL-6 ($+11.2$), while the lattice pipeline is slightly more accurate on the imbalanced JML-4 (Table~\ref{tab:cvp_vs_ag_selected}). \textbf{Efficiency.} A single pseudoinverse replaces homology extraction, rounding and backbone shortest paths, reducing preprocessing time by $3.7\times$ to $6.5\times$.

\subsection{Ablation of Augmentation and Quotient}
\label{sec:ablation}

With the same frozen $k$-NN classifier we switch augmentation ($\epsilon$) and contraction ($\tau$) on and off (Table~\ref{tab:ablation_study}). At $\epsilon=\tau=0$ the signature is the spectrum of the resistance distance of the skeleton. Augmentation encodes the proximity of neurite tips to the soma in the cycle space, and contraction removes short terminal variation. Each operation helps on its own on ACT-4 and BIL-6, and the combination ($\epsilon=50$, $\tau=10$) is the strongest there, lifting ACT-4 from $41.05\%$ to $49.47\%$ and BIL-6 from $52.07\%$ to $64.05\%$. On JML-4, whose complete projection neurons carry long terminal branches, contraction raises macro-F1 from $32.48\%$ to $39.69\%$ while accuracy drops from $49.35\%$ to $45.45\%$.

\begin{table}[t]
\caption{Ablation of augmentation ($\epsilon$) and bridge contraction ($\tau$) with a frozen $k$-NN classifier: accuracy and macro-F1 (\%).}
\label{tab:ablation_study}
\centering
\footnotesize
\setlength{\tabcolsep}{4pt}
\begin{tabular}{cccccccc}
\toprule
& & \multicolumn{2}{c}{\textbf{ACT-4}} & \multicolumn{2}{c}{\textbf{JML-4}} & \multicolumn{2}{c}{\textbf{BIL-6}} \\
\cmidrule(lr){3-4}\cmidrule(lr){5-6}\cmidrule(lr){7-8}
$\bm{\epsilon}$ & $\bm{\tau}$ & Acc & Macro-F1 & Acc & Macro-F1 & Acc & Macro-F1 \\
\midrule
0 & 0 & 41.05 & 38.71 & 49.35 & 32.48 & 52.07 & 38.87 \\
50 & 0 & 45.26 & 43.09 & 49.35 & 32.48 & 52.48 & 40.12 \\
0 & 10 & 46.32 & 44.73 & 45.45 & 39.69 & 63.64 & 47.24 \\
50 & 10 & \textbf{49.47} & \textbf{47.79} & 48.05 & 36.21 & \textbf{64.05} & 46.65 \\
\bottomrule
\end{tabular}
\end{table}

\section{Conclusion}

The squared tropical polarization distance between the Abel--Jacobi images of two points of a metric graph is the gap between the minimal energies of integral and of real unit flows, $d_{\mathrm{Trop}}^2=d_\Gamma-M_{AG}$. This identity replaces lattice search by a closed form: the geodesic lifts are exactly the closest ones, and one Laplacian pseudoinverse and one all-pairs shortest-path computation give the exact tropical distances between all vertices in $\mathcal{O}(|V|^3)$ time. CLEARMIND builds on it a training-free descriptor of neuronal morphology whose every step is exact: Kron reduction, cycle augmentation with a period matrix of shared soma-to-tip lengths, contraction of short bridges, and a spectral summary of $M_{AG}$, which has exactly one positive eigenvalue, equal to the sum of the absolute values of all others. The eigenvectors of $M_{AG}$ separate BREC pairs beyond the 1-WL limit, $70.0\%$ of all pairs with a GIN, including $33.0\%$ of the CFI pairs, more than PPGN ($23.0\%$) and I$^2$-GNN ($21.0\%$). The spectrum is invariant and stable, and it is informative: it improves point-cloud and generative encoders on ACT-4, JML-4 and BIL-6 by up to $20.3$ points, a Tree-LSTM with the spectrum outperforms every reimplemented baseline, an MLP on the spectrum alone is a competitive morphology classifier that surpasses every reimplemented deep baseline on ACT-4, and the spectrum separates 217 BREC pairs without training. The Pythagorean structure behind the identity makes it possible to define message-passing operators directly on the universal cover of the tropical Jacobian (Appendix~\ref{appendix:future}).

\subsection*{AI use statement}

In this work, we used generative AI tools to help develop the theoretical framework, to formulate mathematical claims, to provide ingredients for proofs and assist in writing them (in particular for Theorem~\ref{thm:main} and the proof chain of Appendix~\ref{sec:appendix}), to refine hypotheses, to give feedback on the methodology and the experimental protocol, to implement the numerical verification scripts and the code that generates Figure~\ref{fig:transformation_pipeline}, and to help interpret results. We have not used generative AI tools to generate synthetic data sets or to clean or reformat data sets, and assistance with translation and qualitative or thematic data analysis are not applicable to this work. Additionally, we used generative AI tools to draft and edit text, to improve readability, to format tables, references and \LaTeX{}, and to prepare figures. We have reviewed all AI-assisted work: the authors checked every AI-assisted proof, and the scripts in the supplementary material verify the results of Appendix~\ref{sec:appendix} numerically. We take responsibility for the final content of this work, including text, claims or artifacts produced with the aid of generative AI.

\subsection*{Ethics statement}

This work uses publicly released neuronal morphology reconstructions (Allen Cell Types, Janelia MouseLight, and BICCN fMOST) that were collected by the originating consortia under their own institutional protocols. We do not collect new human or animal data, and we do not release additional raw reconstructions. The descriptors are computed from existing SWC trees for cell-type classification and graph-isomorphism analysis; they do not identify individual subjects and are not intended for clinical decision-making. All co-authors have read and commit to adhering to the ICLR Code of Ethics.

\subsection*{Reproducibility statement}

All theorems are proved in Appendix~\ref{sec:appendix}, which starts from definitions and states every assumption. Algorithm~\ref{alg:ag_feature_extraction} specifies the CLEARMIND pipeline, and Appendix~\ref{app:numerics} describes its numerical implementation. Section~\ref{sec:experiments} and Appendix~\ref{app:protocols} record the data splits, seeds, optimizers, model-selection rule, baseline reproductions, fusion architectures and parameter counts, and Appendix~\ref{app:brec_protocol} gives the BREC protocol. Source code, including scripts that verify Theorem~\ref{thm:main} numerically, is available at \url{https://github.com/Yyuzrah/CLEAR-MIND}.

\bibliography{iclr2027_conference}
\bibliographystyle{iclr2027_conference}

\clearpage
\appendix
\setlength{\parskip}{9pt plus 2pt minus 1pt}
\setlength{\jot}{5pt}
\allowdisplaybreaks[2]
\makeatletter
\g@addto@macro\normalsize{%
  \setlength\abovedisplayskip{10pt plus 3pt minus 2pt}%
  \setlength\belowdisplayskip{10pt plus 3pt minus 2pt}%
  \setlength\abovedisplayshortskip{6pt plus 2pt}%
  \setlength\belowdisplayshortskip{8pt plus 2pt minus 1pt}}
\def\thm@space@setup{\thm@preskip=12pt plus 3pt minus 2pt \thm@postskip=\thm@preskip}
\renewenvironment{proof}[1][\proofname]{\par\pushQED{\qed}\normalfont
  \topsep10\p@\@plus3\p@\relax\trivlist\item[\hskip\labelsep\itshape#1\@addpunct{.}]\ignorespaces}
  {\popQED\endtrivlist\@endpefalse}
\makeatother
\normalsize

\section{Mathematical Foundations and Proofs}
\label{sec:appendix}

This appendix develops the mathematics behind CLEARMIND from first principles. We proceed in the order in which the pipeline acts on a neuron: we first fix notation (Section~\ref{app:conventions}), then identify the Arakelov--Green (AG) distance as the solution of a variational problem (Section~\ref{app:potential}), and then follow the three structural operations (Kron reduction, cycle augmentation, and the quotient by short bridges; Sections~\ref{app:kron}--\ref{app:quotient}). Using these operations we prove the closed form on the universal cover (Section~\ref{app:closed_form}) and the exact identity linking the AG distance to the tropical polarization distance (Section~\ref{app:exact_trop}), analyse the explicit lattice baseline (Section~\ref{appendix:proof_tropical_approx}), and finish with the properties of the spectral signature and its values on distance-regular graphs (Sections~\ref{app:spectral}--\ref{app:drg}). Table~\ref{tab:proof_roadmap} records which result supports which statement of the main text.

\begin{table}[h]
\caption{Proof roadmap: statements of the main text and the results that establish them.}
\label{tab:proof_roadmap}
\begin{center}
\small
\begin{tabular}{>{\raggedright\arraybackslash}p{0.58\linewidth} >{\raggedright\arraybackslash}p{0.36\linewidth}}
\toprule
\textbf{Statement in the main text} & \textbf{Established by} \\
\midrule
$M_{AG}$ is the minimal energy of a unit flow (Thomson's principle) & Theorem~\ref{thm:thomson} \\
$M_{AG}$ is the Arakelov--Green function written as a distance; Euclidean distance matrix & Propositions~\ref{prop:arakelov}, \ref{prop:metric} \\
Kron reduction (Theorem~\ref{thm:kron}); $M_{AG}=d_T$ on trees & Theorem~\ref{thm:schur}, Corollary~\ref{cor:kron_tree} \\
$\beta_1(\Gamma_\epsilon)=|E_\epsilon|$, cycle basis, period matrix (Eq.~\eqref{eq:period_main}) & Proposition~\ref{prop:augmentation} \\
Effect of a single augmented edge & Proposition~\ref{prop:rank_one} \\
Exact contraction of short bridges; invariance of $Q$, $\operatorname{Jac}$ and $\Phi$; the surrogate $W_{u,v}=10^9$ & Proposition~\ref{prop:bridge}, Remark~\ref{rem:quotient_impl} \\
Tropical Jacobian $\mathbb{R}^g/\Lambda$, Abel--Jacobi map, and $d_{\mathrm{Trop}}$ as a closest vector problem (Eq.~\eqref{eq:dtrop}) & Lemma~\ref{lem:period_coords} \\
Lift invariance, Theorem~\ref{thm:main}(i) (Eq.~\eqref{eq:closed_form_main}); independence of the model & Theorem~\ref{thm:closed_form}, Lemma~\ref{lem:subdivision} \\
Exact tropical distance, geodesic lifts, solutions of the CVP, integrality gap, Theorem~\ref{thm:main}(ii)--(iii); extension to any two points of $\Gamma$; sharpening of \citet{cao2025computing}; single-edge case \citep{baker2011metric} & Lemma~\ref{lem:integral_flows}, Theorem~\ref{thm:exact_trop}, Lemma~\ref{lem:subdivision}, Proposition~\ref{prop:voronoi}, Remarks~\ref{rem:nphard}, \ref{rem:special_cases} \\
Exact tropical distance matrices in $\mathcal{O}(|V|^3)$, and in $\mathcal{O}(|V_{\text{core}}|^3)$ after Kron reduction & Corollaries~\ref{cor:kron_tree}, \ref{cor:complexity}, \ref{cor:signature} \\
Morphological closed form (Eq.~\eqref{eq:morph_main}) & Corollary~\ref{cor:morph} \\
Babai rounding: basis dependence and exactly when it fails & Propositions~\ref{prop:babai}, \ref{prop:voronoi} \\
Sign pattern of the spectrum; invariances of the signature & Theorem~\ref{thm:spectral} \\
Invariance of the signature to vertex order, rigid motions and subdivision; Lipschitz bound; cost $\mathcal{O}(|V|+|V_{\text{core}}|^3)$ without lattice search & Corollary~\ref{cor:signature}, Proposition~\ref{prop:stability} \\
Distance-regular graphs on BREC, including strongly regular and $4$-vertex-condition graphs & Proposition~\ref{prop:drg}, Corollary~\ref{cor:srg} \\
\bottomrule
\end{tabular}
\end{center}
\end{table}

\subsection{Standing conventions}
\label{app:conventions}

Throughout, $\Gamma$ is a compact connected metric graph and $G=(V,E)$ is a combinatorial model of $\Gamma$: a finite connected multigraph without self-loops in which every edge $e\in E$ carries a length $\ell_e>0$. We allow parallel edges because cycle augmentation may join a leaf to the root to which it is already adjacent. Every edge $e$ is given a fixed orientation, with tail $e^-\in V$ and head $e^+\in V$; no statement below depends on this choice, and none depends on the model either (Lemma~\ref{lem:subdivision}). We write $n=|V|$ and $m=|E|$, and we use $\mathbf{1}$ for the all-ones vector.

\begin{definition}[Chains, boundary, and energy]
\label{def:chains}
The real chain spaces are $C_0=\mathbb{R}^V$ and $C_1=\mathbb{R}^E$, with standard bases $\{\delta_v\}_{v\in V}$ and $\{\mathbf{1}_e\}_{e\in E}$; the integral chains are $\mathbb{Z}^V\subset C_0$ and $\mathbb{Z}^E\subset C_1$. The boundary map $\partial:C_1\to C_0$ is the linear map with $\partial\mathbf{1}_e=\delta_{e^+}-\delta_{e^-}$; its matrix is the incidence matrix $B\in\{-1,0,1\}^{V\times E}$. Thus $(\partial f)(v)$ is the net inflow of $f$ at $v$. The \emph{energy inner product} on $C_1$ is
\begin{equation*}
\langle f,h\rangle_\ell=\sum_{e\in E}\ell_e f_e h_e=f^\top D_\ell h,\qquad D_\ell=\operatorname{diag}(\ell_e)_{e\in E},
\end{equation*}
with norm $\|f\|_\ell$. A \emph{path} $\gamma$ from $x$ to $y$ is a sequence of distinct vertices $x=v_0,\dots,v_k=y$ together with edges $e_i$ joining $v_{i-1}$ and $v_i$; its chain is $[\gamma]=\sum_{i=1}^k\sigma_i\mathbf{1}_{e_i}$, where $\sigma_i=+1$ if $e_i$ is traversed from tail to head and $\sigma_i=-1$ otherwise. Then $\partial[\gamma]=\delta_y-\delta_x$, and since the coefficients of $[\gamma]$ are $\pm1$ on distinct edges,
\begin{equation}
\label{eq:path_energy}
\|[\gamma]\|_\ell^2=\sum_{i=1}^k\ell_{e_i}=\ell(\gamma),
\end{equation}
the length of $\gamma$. The \emph{path metric} is
\begin{equation*}
d_\Gamma(x,y)=\min\{\ell(\gamma):\gamma\text{ a path from }x\text{ to }y\};
\end{equation*}
a path attaining the minimum is a \emph{geodesic}. For $x,y\in V$ we call
\begin{align*}
\mathcal{F}(x,y)&=\{f\in C_1:\partial f=\delta_y-\delta_x\}\quad\text{and}\\
\mathcal{F}_{\mathbb{Z}}(x,y)&=\mathcal{F}(x,y)\cap\mathbb{Z}^E
\end{align*}
the sets of \emph{unit flows} and \emph{integral unit flows} from $x$ to $y$.
\end{definition}

\begin{definition}[Laplacian and Arakelov--Green distance]
\label{def:laplacian}
The affinity matrix $W\in\mathbb{R}^{V\times V}$ has entries $W_{u,v}=\sum_{e\text{ joining }u,v}\ell_e^{-1}$ for $u\neq v$ and $W_{u,u}=0$; parallel edges therefore add their conductances. The weighted graph Laplacian (the discrete Laplace--Beltrami operator of $\Gamma$) is
\begin{equation}
\label{eq:laplacian_factor}
L=\operatorname{diag}(W\mathbf{1})-W=B D_\ell^{-1}B^\top .
\end{equation}
The \emph{Arakelov--Green distance} between $x,y\in V$ is
\begin{equation*}
M_{AG}(x,y)=(\delta_x-\delta_y)^\top L^+(\delta_x-\delta_y)=L^+_{x,x}+L^+_{y,y}-2L^+_{x,y},
\end{equation*}
where $L^+$ is the Moore--Penrose pseudoinverse. The factorization in Eq.~\eqref{eq:laplacian_factor} holds because
\begin{equation*}
(BD_\ell^{-1}B^\top)_{u,v}=\sum_e B_{u,e}B_{v,e}/\ell_e ,
\end{equation*}
for $u\neq v$ only edges joining $u$ and $v$ contribute, each with $B_{u,e}B_{v,e}=-1$, while for $u=v$ each edge incident to $u$ contributes $B_{u,e}^2/\ell_e=1/\ell_e$ (there are no self-loops), for a total of $(W\mathbf{1})_u$.

In electrical language, $M_{AG}$ is the effective resistance of the network in which each edge $e$ is a resistor of resistance $\ell_e$ \citep{klein1993resistance}; Proposition~\ref{prop:arakelov} explains the name we use.
\end{definition}

\begin{definition}[Cycle lattice, period matrix, tropical Jacobian, Abel--Jacobi map]
\label{def:jacobian}
The real cycle space is $Z_1=\ker\partial\subset C_1$, and the cycle lattice is $H_1(\Gamma;\mathbb{Z})=Z_1\cap\mathbb{Z}^E$; since a graph has no $2$-cells, this is the first integral homology group. Because $G$ is connected, the vectors $u\in\mathbb{R}^V$ with $B^\top u=0$, that is, with $u(e^+)=u(e^-)$ on every edge, are the constants; hence $\partial$ has rank $n-1$, and the genus~is
\begin{equation*}
g=\dim Z_1=m-n+1=\beta_1(\Gamma).
\end{equation*}
Let $\pi:C_1\to Z_1$ denote the $\langle\cdot,\cdot\rangle_\ell$-orthogonal projection. Fix a $\mathbb{Z}$-basis $z_1,\dots,z_g$ of $H_1(\Gamma;\mathbb{Z})$, let $C\in\mathbb{Z}^{g\times E}$ be the matrix with rows $z_i^\top$, and let
\begin{equation*}
Q=CD_\ell C^\top\in\mathbb{R}^{g\times g}
\end{equation*}
be the \emph{period matrix}. The \emph{tropical Jacobian} is the flat torus $\operatorname{Jac}(\Gamma)=Z_1/H_1(\Gamma;\mathbb{Z})$ with the distance
\begin{equation*}
d_{\operatorname{Jac}}([a],[b])=\min_{\lambda\in H_1(\Gamma;\mathbb{Z})}\|a-b-\lambda\|_\ell .
\end{equation*}
For a base vertex $q$, the \emph{Abel--Jacobi map} is $\Phi_q(x)=[\pi[\gamma]]$, where $\gamma$ is any path from $q$ to $x$. The \emph{tropical polarization distance} between $x,y\in V$ is
\begin{equation*}
d_{\mathrm{Trop}}(x,y)=d_{\operatorname{Jac}}(\Phi_q(x),\Phi_q(y)).
\end{equation*}
\end{definition}

Three points deserve comment. First, $\Phi_q$ is well defined: two paths from $q$ to $x$ differ by an integral cycle, which $\pi$ fixes.

Second, $d_{\mathrm{Trop}}$ does not depend on $q$, because for any path $\gamma$ from $x$ to $y$ the chain $[\gamma_{qy}]-[\gamma_{qx}]-[\gamma]$ is an integral cycle, whence $\Phi_q(y)-\Phi_q(x)=[\pi[\gamma]]$.

Third, $H_1(\Gamma;\mathbb{Z})$ is a full-rank lattice in $Z_1$, and every $\mathbb{Z}$-basis of it is an $\mathbb{R}$-basis of $Z_1$. Fix a spanning tree $T$ of $G$ and, for each of the $g$ edges $e\notin T$, let $z_e$ be the chain of the cycle formed by $e$ and the path in $T$ between its ends, oriented so that $z_e(e)=1$; then $z_e(e')=\delta_{ee'}$ for $e,e'\notin T$. For $z\in Z_1$, the cycle $z-\sum_{e\notin T}z(e)\,z_e$ is supported on $T$, and a nonzero cycle supported on a forest would violate conservation at a leaf of its support, so
\begin{equation*}
z=\sum_{e\notin T}z(e)\,z_e .
\end{equation*}
These coefficients are unique, and they are integers exactly when $z$ is integral. Hence the fundamental cycles are at once a $\mathbb{Z}$-basis of $H_1(\Gamma;\mathbb{Z})$ and an $\mathbb{R}$-basis of $Z_1$ (Proposition~\ref{prop:augmentation}(ii) is the case of an augmented tree). Any other $\mathbb{Z}$-basis is the image of this one under a matrix in $GL_g(\mathbb{Z})$, which is invertible over $\mathbb{R}$, so it is an $\mathbb{R}$-basis of $Z_1$ as well.

This is the Albanese torus of \citet{kotani2000jacobian}, the model that \citet[Section~4.3]{cao2025computing} use for the tropical Jacobian. The integration map $a\mapsto(\omega\mapsto\int_a\omega)$ carries it isometrically onto the model $\Omega(\Gamma)^*/H_1(\Gamma;\mathbb{Z})$ built from harmonic $1$-forms and equipped with the tropical polarization (\citealp{mikhalkin2008tropical}; \citealp[Remark~2.3(2)]{baker2011metric}), and it intertwines the two Abel--Jacobi maps \citep[Corollary~4.12]{cao2025computing}. In the coordinates $t$ of Lemma~\ref{lem:period_coords} below, Eq.~\eqref{eq:dtrop_coords} is the square of \citet[Eq.~(5.1)]{cao2025computing}; in the coefficient coordinates $Q^{-1}t$ it is the square of their Eq.~(5.2). The next lemma converts these intrinsic definitions into the coordinates used in the main text.

\begin{lemma}[Period coordinates]
\label{lem:period_coords}
For $f\in C_1$ put $t(f)=CD_\ell f\in\mathbb{R}^g$, so that $t_i(f)=\langle z_i,f\rangle_\ell$. Then $Q$ is symmetric positive definite, and
\begin{equation}
\label{eq:projection_coords}
\pi f=C^\top Q^{-1}t(f),\qquad \|\pi f\|_\ell^2=t(f)^\top Q^{-1}t(f).
\end{equation}
The restriction $t|_{Z_1}:Z_1\to\mathbb{R}^g$ is a linear isomorphism that maps $H_1(\Gamma;\mathbb{Z})$ onto the lattice $\Lambda=Q\mathbb{Z}^g$ and satisfies $\|a\|_\ell^2=t(a)^\top Q^{-1}t(a)$. It therefore induces an isometry from $\operatorname{Jac}(\Gamma)$ onto the flat torus $\mathbb{R}^g/\Lambda$ with the polarization metric $\langle t,t'\rangle=t^\top Q^{-1}t'$, which sends $\Phi_q(x)$ to $t([\gamma])\bmod\Lambda$ for any path $\gamma$ from $q$ to $x$; these are the tropical Jacobian and the Abel--Jacobi map of Section~\ref{subsec:polarization}. For every path $\gamma$ from $x$ to $y$,
\begin{equation}
\label{eq:dtrop_coords}
d_{\mathrm{Trop}}(x,y)^2=\min_{n\in\mathbb{Z}^g}\,(t-Qn)^\top Q^{-1}(t-Qn),\qquad t=t([\gamma]),
\end{equation}
whose square root is Eq.~\eqref{eq:dtrop} of the main text.
\end{lemma}

\begin{proof}
The matrix $Q$ is the Gram matrix of the $\mathbb{R}$-linearly independent vectors $z_1,\dots,z_g$ with respect to an inner product, so it is symmetric positive definite. Let $f\in C_1$. Since $\pi f\in Z_1$ and the $z_i$ span $Z_1$, we can write $\pi f=C^\top c$ for a unique $c\in\mathbb{R}^g$. The defining property of the orthogonal projection is that $f-\pi f$ is orthogonal to every $z_i$, that is, $CD_\ell(f-C^\top c)=0$. This says $t(f)=Qc$, so $c=Q^{-1}t(f)$, and consequently
\begin{equation*}
\|\pi f\|_\ell^2=c^\top Qc=t(f)^\top Q^{-1}t(f).
\end{equation*}
This proves Eq.~\eqref{eq:projection_coords}.

For $a=C^\top c\in Z_1$ we have $t(a)=Qc$, and $Q$ is invertible, so $t|_{Z_1}$ is an isomorphism. Because the $z_i$ form a $\mathbb{Z}$-basis, $a$ is integral exactly when $c\in\mathbb{Z}^g$, so $t$ maps $H_1(\Gamma;\mathbb{Z})$ onto $Q\mathbb{Z}^g$. The identity $\|a\|_\ell^2=t(a)^\top Q^{-1}t(a)$ for $a\in Z_1$ is Eq.~\eqref{eq:projection_coords} with $\pi a=a$, so $t|_{Z_1}$ descends to an isometry $Z_1/H_1(\Gamma;\mathbb{Z})\to\mathbb{R}^g/\Lambda$. Moreover $t(\pi[\gamma])=t([\gamma])$ because $[\gamma]-\pi[\gamma]$ is orthogonal to each $z_i$, so this isometry sends $\Phi_q(x)=[\pi[\gamma]]$ to $t([\gamma])\bmod\Lambda$.

Finally, for a path $\gamma$ from $x$ to $y$ we have $\Phi_q(y)-\Phi_q(x)=[\pi[\gamma]]$ by the second comment after Definition~\ref{def:jacobian}, and $d_{\operatorname{Jac}}([a],[b])$ depends only on $a-b$; hence
\begin{align*}
d_{\mathrm{Trop}}(x,y)^2&=\min_{\lambda\in H_1(\Gamma;\mathbb{Z})}\|\pi[\gamma]-\lambda\|_\ell^2\\
&=\min_{n\in\mathbb{Z}^g}(t-Qn)^\top Q^{-1}(t-Qn),
\end{align*}
as claimed.
\end{proof}

\begin{remark}[Change of basis]
\label{rem:basis_change}
If $C$ is replaced by $UC$ with $U\in GL_g(\mathbb{Z})$, then $Q\mapsto UQU^\top$, $t\mapsto Ut$, $\Lambda\mapsto U\Lambda$, and every quadratic form of the type $t^\top Q^{-1}t$ is unchanged. The distance $d_{\mathrm{Trop}}$ is defined intrinsically and is therefore basis free; this will contrast with the behaviour of lattice rounding (Proposition~\ref{prop:babai}).
\end{remark}

\subsection{Discrete potential theory: the AG distance as a minimal energy}
\label{app:potential}

The main text defines $M_{AG}$ algebraically, through $L^+$. We first identify the optimization problem it solves. The answer, Thomson's principle, is classical \citep{doyle1984random,biggs1997potential,lyons2016probability}, and the orthogonal decomposition behind it appears in the tropical literature as well (\citealp[Section~7]{dejong2022crossratios}; \citealp[Proposition~3.17]{shokrieh2019canonical}). We include a complete proof because we need the version that keeps track of the Pythagorean remainder, through which the tropical Jacobian enters.

\begin{lemma}[Laplacian facts]
\label{lem:laplacian}
The Laplacian $L$ is symmetric positive semidefinite with
\begin{equation*}
u^\top Lu=\sum_{e\in E}(u(e^+)-u(e^-))^2/\ell_e.
\end{equation*}
Its kernel is $\operatorname{span}(\mathbf{1})$ and its image is $\mathbf{1}^\perp=\{b\in\mathbb{R}^V:\mathbf{1}^\top b=0\}$. Moreover $LL^+=L^+L=I-\frac1n\mathbf{1}\mathbf{1}^\top$, $L^+\mathbf{1}=0$, and $L^+$ is positive definite on $\mathbf{1}^\perp$.
\end{lemma}

\begin{proof}
The formula for $u^\top Lu$ follows from Eq.~\eqref{eq:laplacian_factor}, since $u^\top BD_\ell^{-1}B^\top u=\|B^\top u\|^2_{D_\ell^{-1}}$ and $(B^\top u)_e=u(e^+)-u(e^-)$. Because $L$ is positive semidefinite, $Lu=0$ holds exactly when $u^\top Lu=0$, that is, when $u$ takes equal values at the two ends of every edge; since $G$ is connected, this happens exactly when $u$ is constant, so $\ker L=\operatorname{span}(\mathbf{1})$. Since $L$ is symmetric, its image is the orthogonal complement of its kernel.

For a symmetric matrix, $L^+$ has the same eigenvectors as $L$ with the nonzero eigenvalues inverted, so $LL^+=L^+L$ is the orthogonal projector onto $\operatorname{im}L=\mathbf{1}^\perp$, which is $I-\frac1n\mathbf{1}\mathbf{1}^\top$; the remaining statements follow from the same spectral description.
\end{proof}

\begin{theorem}[Thomson's principle with Pythagorean remainder]
\label{thm:thomson}
Let $x,y\in V$. Put $u^*=L^+(\delta_y-\delta_x)$ and $f^*=D_\ell^{-1}B^\top u^*$, that is, $f^*_e=(u^*(e^+)-u^*(e^-))/\ell_e$. Then:
\begin{enumerate}
\item[(i)] $f^*\in\mathcal{F}(x,y)$ and $\|f^*\|_\ell^2=M_{AG}(x,y)$;
\item[(ii)] $f^*$ is $\langle\cdot,\cdot\rangle_\ell$-orthogonal to $Z_1$;
\item[(iii)] for every unit flow $f\in\mathcal{F}(x,y)$ we have $\pi f=f-f^*$ and
\begin{equation}
\label{eq:pythagoras}
\|f\|_\ell^2=M_{AG}(x,y)+\|\pi f\|_\ell^2 .
\end{equation}
\end{enumerate}
In particular,
\begin{equation*}
M_{AG}(x,y)=\min_{f\in\mathcal{F}(x,y)}\|f\|_\ell^2,
\end{equation*}
and $f^*$ is the unique minimizer.
\end{theorem}

\textit{Proof idea.} The set $\mathcal{F}(x,y)$ is an affine subspace whose direction is $Z_1$, because two unit flows differ by a cycle. The point of minimal norm in an affine subspace is the unique point orthogonal to its direction, and every other point is obtained from it by adding a vector of the direction, which adds its squared norm. We exhibit $f^*$ as that orthogonal point.

\begin{proof}
(i) By Eq.~\eqref{eq:laplacian_factor},
\begin{equation*}
\partial f^*=BD_\ell^{-1}B^\top u^*=Lu^*=LL^+(\delta_y-\delta_x).
\end{equation*}
Since $\delta_y-\delta_x\in\mathbf{1}^\perp$, Lemma~\ref{lem:laplacian} gives $\partial f^*=\delta_y-\delta_x$, so $f^*\in\mathcal{F}(x,y)$. For the energy,
\begin{align*}
\|f^*\|_\ell^2&=(u^*)^\top BD_\ell^{-1}D_\ell D_\ell^{-1}B^\top u^*\\
&=(u^*)^\top Lu^*\\
&=(\delta_y-\delta_x)^\top L^+LL^+(\delta_y-\delta_x)\\
&=M_{AG}(x,y),
\end{align*}
where the last step uses the identity $L^+LL^+=L^+$ of the pseudoinverse.

(ii) Let $z\in Z_1$. Then
\begin{equation*}
\langle f^*,z\rangle_\ell=(u^*)^\top BD_\ell^{-1}D_\ell z=(u^*)^\top\partial z=0.
\end{equation*}

(iii) Let $f\in\mathcal{F}(x,y)$. Both $f$ and $f^*$ have boundary $\delta_y-\delta_x$, so $f-f^*\in Z_1$. Therefore $f=f^*+(f-f^*)$ writes $f$ as the sum of a vector orthogonal to $Z_1$ (by (ii)) and a vector in $Z_1$. Orthogonal decompositions are unique, hence $\pi f=f-f^*$, and Pythagoras' theorem gives
\begin{equation*}
\|f\|_\ell^2=\|f^*\|_\ell^2+\|f-f^*\|_\ell^2=M_{AG}(x,y)+\|\pi f\|_\ell^2.
\end{equation*}
The final assertion follows because $\|\pi f\|_\ell^2\ge0$, with equality exactly when $\pi f=0$, that is, when $f=f^*$.
\end{proof}

\begin{remark}[Hodge decomposition]
\label{rem:hodge}
A chain $h$ is orthogonal to $Z_1=\ker B$ exactly when $D_\ell h\in(\ker B)^\perp=\operatorname{im}B^\top$, so $Z_1^\perp=D_\ell^{-1}\operatorname{im}B^\top$ is the space of \emph{gradient flows} (Ohm's law). Theorem~\ref{thm:thomson} is therefore the discrete Hodge decomposition $C_1=Z_1^\perp\oplus Z_1$ applied to a unit flow: the gradient part is the same for all unit flows and carries energy $M_{AG}(x,y)$, and the cyclic part is $\pi f$. The cyclic part lies in $Z_1$, whose quotient by $H_1(\Gamma;\mathbb{Z})$ is the tropical Jacobian (Lemma~\ref{lem:period_coords}).
\end{remark}

We record two consequences that will be used later. The first bounds the current through a single edge; it is the only analytic estimate needed for the stability result of Section~\ref{app:spectral}.

\begin{lemma}[Unit currents are bounded by one]
\label{lem:current_bound}
Let $x\neq y$ and let $u^*,f^*$ be as in Theorem~\ref{thm:thomson}. Then $u^*$ attains its minimum over $V$ at $x$ and its maximum at $y$, $u^*(y)-u^*(x)=M_{AG}(x,y)$, and $|f^*_e|\le1$ for every edge $e$. For $x=y$ these conclusions hold trivially, since then $u^*=0$ and $f^*=0$.
\end{lemma}

\begin{proof}
For $a\notin\{x,y\}$ we have $(Lu^*)(a)=(\delta_y-\delta_x)(a)=0$, that is, $u^*(a)$ is the weighted average $\sum_bW_{a,b}u^*(b)/\sum_bW_{a,b}$ of its neighbours; we say $u^*$ is harmonic at $a$.

Let $A$ be the set of vertices at which $u^*$ attains its minimum, and suppose $x\notin A$. If $y\in A$, then $(Lu^*)(y)=\sum_bW_{y,b}(u^*(y)-u^*(b))\le0$, contradicting $(Lu^*)(y)=1$. Hence $A$ consists of harmonic vertices, and a weighted average with positive weights equals the minimum only if all neighbours attain it; so $A$ is closed under taking neighbours, and connectedness gives $A=V$, contradicting $x\notin A$. Thus the minimum is attained at $x$, and the symmetric argument shows that the maximum is attained at $y$. Next,
\begin{equation*}
u^*(y)-u^*(x)=(\delta_y-\delta_x)^\top L^+(\delta_y-\delta_x)=M_{AG}(x,y).
\end{equation*}

Now fix an edge $e$. If $u^*(e^-)=u^*(e^+)$ then $f^*_e=0$. Otherwise, reversing the orientation of $e$ if necessary (which only changes the sign of $f^*_e$), we may assume $u^*(e^-)<u^*(e^+)$. Choose $s$ with $u^*(e^-)<s<u^*(e^+)$ and put $S=\{a\in V:u^*(a)<s\}$. By the first part, $x\in S$ and $y\notin S$. Summing $\partial f^*=\delta_y-\delta_x$ over $S$, the contributions of edges with both ends in $S$ cancel, and we get
\begin{equation*}
-1=\sum_{a\in S}(\partial f^*)(a)=-\sum_{e'\text{ crossing }S}\varphi_{e'},
\end{equation*}
where, for an edge $e'$ with one end $a\in S$ and the other end $b\notin S$, $\varphi_{e'}=(u^*(b)-u^*(a))/\ell_{e'}$ is the flow from $a$ to $b$. Each $\varphi_{e'}$ is positive because $u^*(b)\ge s>u^*(a)$. Hence the positive numbers $\varphi_{e'}$ sum to $1$, and $e$ is one of the crossing edges, with $\varphi_e=|f^*_e|$. Therefore $|f^*_e|\le1$.
\end{proof}

The second consequence is that $M_{AG}$ is a metric. This is classical \citep{klein1993resistance}; we include the short proof because it uses only the maximum principle just established.

\begin{proposition}[Metric properties]
\label{prop:metric}
The Arakelov--Green distance $M_{AG}$ is a metric on $V$. Moreover, $\sqrt{M_{AG}}$ is Euclidean: there are points $b_v\in\mathbb{R}^{n-1}$, $v\in V$, with $M_{AG}(x,y)=\|b_x-b_y\|_2^2$ for all $x,y$.
\end{proposition}

\begin{proof}
Symmetry and $M_{AG}(x,x)=0$ follow from the definition. If $x\neq y$, then $\delta_x-\delta_y$ is a nonzero vector of $\mathbf{1}^\perp$, on which $L^+$ is positive definite (Lemma~\ref{lem:laplacian}), so $M_{AG}(x,y)>0$.

For the triangle inequality let $x,y,z$ be distinct (otherwise there is nothing to prove). Splitting $\delta_x-\delta_z=(\delta_x-\delta_y)+(\delta_y-\delta_z)$ in the right-hand factor gives
\begin{gather*}
M_{AG}(x,z)=(\delta_x-\delta_z)^\top v+(\delta_x-\delta_z)^\top w,\\
v=L^+(\delta_x-\delta_y),\quad w=L^+(\delta_y-\delta_z).
\end{gather*}
By Lemma~\ref{lem:current_bound} applied to the pair $(y,x)$, the potential $v$ is minimal at $y$ and $v(x)-v(y)=M_{AG}(x,y)$; hence
\begin{equation*}
(\delta_x-\delta_z)^\top v=v(x)-v(z)\le v(x)-v(y)=M_{AG}(x,y).
\end{equation*}
By the same lemma applied to the pair $(z,y)$, the potential $w$ is maximal at $y$ and $w(y)-w(z)=M_{AG}(y,z)$, so
\begin{equation*}
(\delta_x-\delta_z)^\top w=w(x)-w(z)\le w(y)-w(z)=M_{AG}(y,z).
\end{equation*}
Adding the two bounds gives $M_{AG}(x,z)\le M_{AG}(x,y)+M_{AG}(y,z)$.

For the Euclidean statement, let $\phi_1,\dots,\phi_{n-1}$ be orthonormal eigenvectors of $L$ for its positive eigenvalues $\kappa_1,\dots,\kappa_{n-1}$, so that $L^+=\sum_{k}\kappa_k^{-1}\phi_k\phi_k^\top$, and let $X\in\mathbb{R}^{V\times(n-1)}$ have columns $\kappa_k^{-1/2}\phi_k$, so that $L^+=XX^\top$. With $b_v$ the $v$-th row of $X$,
\begin{equation*}
M_{AG}(x,y)=\|X^\top(\delta_x-\delta_y)\|_2^2=\|b_x-b_y\|_2^2.\qedhere
\end{equation*}
\end{proof}

We can now explain the name. On a metrized graph, the Arakelov--Green function $g_\mu$ attached to a signed measure $\mu$ of total mass one is the Green function of the Laplacian, $\Delta_xg_\mu(x,y)=\delta_y-\mu$, normalized so that it integrates to zero against $\mu$ \citep{zhang1993admissible,baker2007harmonic}; up to an additive constant it is the function $j_\mu$ of \citet{chinburg1993capacity}, and it is the non-archimedean counterpart of the Arakelov Green function of a Riemann surface in Zhang's admissible pairing \citep{zhang1993admissible}. Its distinguished instance is attached to the canonical measure $\mu_{\mathrm{can}}$ of \citet[Theorem~2.11]{chinburg1993capacity}, and there $g_{\mu_{\mathrm{can}}}(x,y)=\tau(\Gamma)-\frac12r(x,y)$, with $r$ the effective resistance and $\tau(\Gamma)$ a constant \citep[Theorem~14.1]{baker2007harmonic}. The following proposition is the exact discrete counterpart of these facts.

\begin{proposition}[Discrete Arakelov--Green functions]
\label{prop:arakelov}
Let $\mu\in\mathbb{R}^V$ satisfy $\mathbf{1}^\top\mu=1$ (the entries of $\mu$ may be negative), and define $g_\mu(x,y)=(\delta_x-\mu)^\top L^+(\delta_y-\mu)$. Then:
\begin{enumerate}
\item[(i)] $g_\mu$ is symmetric, and for each $y$ the function $h=g_\mu(\cdot,y)$ is the unique solution of $Lh=\delta_y-\mu$ with $\mu^\top h=0$;
\item[(ii)] for the uniform measure $\mu=\frac1n\mathbf{1}$ one has $g_\mu=L^+$;
\item[(iii)] for every $\mu$, $g_\mu(x,x)+g_\mu(y,y)-2g_\mu(x,y)=M_{AG}(x,y)$;
\item[(iv)] the vector
\begin{align*}
\mu_{\mathrm{can}}&=\tfrac1n\mathbf{1}+\tfrac12L\operatorname{diag}(L^+),\\
\text{equivalently}\qquad\mu_{\mathrm{can}}(v)&=1-\frac12\sum_{e\ni v}\frac{M_{AG}(e^-,e^+)}{\ell_e},
\end{align*}
is the unique $\mu$ with $\mathbf{1}^\top\mu=1$ for which $g_\mu(x,x)$ does not depend on $x$. It satisfies $M_{AG}\,\mu_{\mathrm{can}}=2\tau\mathbf{1}$ with
\begin{align*}
\tau&=\tfrac12\mu_{\mathrm{can}}^\top M_{AG}\,\mu_{\mathrm{can}}=\tfrac12\max\{\mu^\top M_{AG}\,\mu:\mathbf{1}^\top\mu=1\},\\
\intertext{and}
g_{\mu_{\mathrm{can}}}(x,y)&=\tau-\tfrac12M_{AG}(x,y)\qquad\text{for all }x,y\in V.
\end{align*}
\end{enumerate}
\end{proposition}

\begin{proof}
(i) Symmetry holds because $L^+$ is symmetric. As a function of $x$,
\begin{equation*}
g_\mu(\cdot,y)=L^+(\delta_y-\mu)-\big(\mu^\top L^+(\delta_y-\mu)\big)\mathbf{1}.
\end{equation*}
Since $L\mathbf{1}=0$ and $\delta_y-\mu\in\mathbf{1}^\perp$, Lemma~\ref{lem:laplacian} gives $Lg_\mu(\cdot,y)=\delta_y-\mu$; and
\begin{equation*}
\mu^\top g_\mu(\cdot,y)=\mu^\top L^+(\delta_y-\mu)(1-\mu^\top\mathbf{1})=0.
\end{equation*}
Two solutions of $Lh=\delta_y-\mu$ differ by an element of $\ker L=\operatorname{span}(\mathbf{1})$, and the normalization $\mu^\top h=0$ fixes this constant because $\mu^\top\mathbf{1}=1\neq0$.

(ii) If $\mu=\frac1n\mathbf{1}$, then $L^+\mu=0$ by Lemma~\ref{lem:laplacian}, so $g_\mu(x,y)=\delta_x^\top L^+\delta_y$.

(iii) Expanding the bilinear form, the terms involving $\mu$ cancel because $(\delta_x-\mu)-(\delta_y-\mu)=\delta_x-\delta_y$.

(iv) Write $d=\operatorname{diag}(L^+)$, so that $M_{AG}=d\mathbf{1}^\top+\mathbf{1}d^\top-2L^+$. Expanding,
\begin{equation*}
g_\mu(x,x)=d_x-2(L^+\mu)_x+\mu^\top L^+\mu,
\end{equation*}
so $g_\mu(x,x)$ is independent of $x$ exactly when $L^+\mu-\frac12d\in\operatorname{span}(\mathbf{1})$.

If this holds, applying $L$ and using $LL^+\mu=\mu-\frac1n\mathbf{1}$ (Lemma~\ref{lem:laplacian} and $\mathbf{1}^\top\mu=1$) gives $\mu=\frac1n\mathbf{1}+\frac12Ld=\mu_{\mathrm{can}}$. Conversely, $\mathbf{1}^\top\mu_{\mathrm{can}}=1$ because $\mathbf{1}^\top L=0$, and
\begin{equation*}
L^+\mu_{\mathrm{can}}=\frac12L^+Ld=\frac12d-\frac{\operatorname{tr}L^+}{2n}\mathbf{1}
\end{equation*}
because $L^+L=I-\frac1n\mathbf{1}\mathbf{1}^\top$ and $\mathbf{1}^\top d=\operatorname{tr}L^+$.

For the edge formula, the diagonal entry $(LL^+)_{v,v}=1-\frac1n$ reads $\sum_bW_{v,b}(d_v-L^+_{v,b})=1-\frac1n$; grouping the edges at $v$ by their other endpoint,
\begin{align*}
\sum_{e\ni v}\frac{M_{AG}(e^-,e^+)}{\ell_e}
&=\sum_bW_{v,b}\big(d_v+d_b-2L^+_{v,b}\big)\\
&=2-\frac2n-\sum_bW_{v,b}(d_v-d_b)\\
&=2-\frac2n-(Ld)(v),
\end{align*}
which is $2-2\mu_{\mathrm{can}}(v)$.

For the Green function, fix $y$ and let $r_y=M_{AG}(\cdot,y)=d+d_y\mathbf{1}-2L^+\delta_y$. Then
\begin{equation*}
Lr_y=Ld-2(\delta_y-\tfrac1n\mathbf{1})=2(\mu_{\mathrm{can}}-\delta_y),
\end{equation*}
so $h=-\frac12r_y+c\mathbf{1}$ solves $Lh=\delta_y-\mu_{\mathrm{can}}$ for every constant $c$, and by (i) $g_{\mu_{\mathrm{can}}}(\cdot,y)$ is the solution with $c=\frac12\mu_{\mathrm{can}}^\top r_y=\frac12(M_{AG}\mu_{\mathrm{can}})_y$. Now
\begin{align*}
M_{AG}\mu_{\mathrm{can}}
&=d\,(\mathbf{1}^\top\mu_{\mathrm{can}})+(d^\top\mu_{\mathrm{can}})\mathbf{1}-2L^+\mu_{\mathrm{can}}\\
&=\Big(d^\top\mu_{\mathrm{can}}+\frac{\operatorname{tr}L^+}{n}\Big)\mathbf{1}
\end{align*}
is a constant vector, which we write $2\tau\mathbf{1}$. Hence $c=\tau$ for every $y$, which gives $g_{\mu_{\mathrm{can}}}(x,y)=\tau-\frac12M_{AG}(x,y)$, and $\mu_{\mathrm{can}}^\top M_{AG}\mu_{\mathrm{can}}=2\tau\,\mathbf{1}^\top\mu_{\mathrm{can}}=2\tau$. Finally, every $\mu$ with $\mathbf{1}^\top\mu=1$ is $\mu_{\mathrm{can}}+\nu$ with $\nu\perp\mathbf{1}$, and
\begin{equation*}
\mu^\top M_{AG}\mu=2\tau+2\nu^\top M_{AG}\mu_{\mathrm{can}}+\nu^\top M_{AG}\nu=2\tau-2\nu^\top L^+\nu,
\end{equation*}
which is at most $2\tau$, with equality only for $\nu=0$ (Lemma~\ref{lem:laplacian}).
\end{proof}

\begin{remark}[Why ``Arakelov--Green'']
Item~(ii) identifies $L^+$ as the discrete Arakelov--Green function of the uniform vertex measure, and item~(iii) shows that $M_{AG}$ is the combination of Arakelov--Green functions that is the same for every $\mu$. Item~(iv) says more: for the canonical measure the Arakelov--Green function \emph{is} the AG distance, $M_{AG}=2(\tau-g_{\mu_{\mathrm{can}}})$, exactly as $r=2(\tau(\Gamma)-g_{\mu_{\mathrm{can}}})$ on the metrized graph \citep[Theorem~14.1]{baker2007harmonic}. The discrete $\mu_{\mathrm{can}}$ arises from the canonical measure
\begin{equation*}
\sum_p(1-\tfrac12\operatorname{val}(p))\delta_p+\sum_e dx/(\ell_e+R_e)
\end{equation*}
of \citet[Theorem~2.11]{chinburg1993capacity}, with $R_e$ the effective resistance between the ends of $e$ in $\Gamma\setminus e$, by moving the mass $\ell_e/(\ell_e+R_e)=1-M_{AG}(e^-,e^+)/\ell_e$ of each edge equally to its two endpoints. The identity $M_{AG}\mu_{\mathrm{can}}=2\tau\mathbf{1}$ says that $\mu_{\mathrm{can}}$ is the equilibrium measure of the kernel $\frac12M_{AG}$, with capacity $\tau$; this is the discrete form of \citet[Corollary~14.2]{baker2007harmonic}. On a tree $M_{AG}(e^-,e^+)=\ell_e$ for every edge, so $\mu_{\mathrm{can}}(v)=1-\frac12\deg(v)$ is negative at every branch point, which is why Proposition~\ref{prop:arakelov} allows signed $\mu$. By Definition~\ref{def:laplacian}, $M_{AG}$ is at the same time the effective resistance of the network with resistances $\ell_e$ \citep{klein1993resistance}.
\end{remark}

\subsection{Topological retraction: Kron reduction}
\label{app:kron}

Sampled neurites contain long runs of degree-$2$ vertices. We want to remove them without changing $M_{AG}$ on the vertices we keep. We use the Schur complement, which is classically known to preserve effective resistances \citep{Dorfler2013}, and prove this variationally: the quadratic form of $L^+$ is the value of a maximization problem, and eliminating vertices amounts to solving part of that problem in closed form.

\begin{lemma}[Variational formula for the pseudoinverse]
\label{lem:variational}
For every $b\in\mathbf{1}^\perp$,
\begin{equation*}
b^\top L^+b=\max_{u\in\mathbb{R}^V}\big(2b^\top u-u^\top Lu\big),
\end{equation*}
and the maximizers are exactly the vectors $u=L^+b+c\mathbf{1}$ with $c\in\mathbb{R}$.
\end{lemma}

\begin{proof}
The objective $J(u)=2b^\top u-u^\top Lu$ is concave because $L$ is positive semidefinite, so its stationary points are its global maximizers. Its gradient $2b-2Lu$ vanishes exactly when $Lu=b$. Because $b\in\operatorname{im}L$, the solutions are $u=L^+b+c\mathbf{1}$, and
\begin{equation*}
J(L^+b)=2b^\top L^+b-b^\top L^+LL^+b=b^\top L^+b.\qedhere
\end{equation*}
\end{proof}

\begin{theorem}[Schur complements preserve the AG distance]
\label{thm:schur}
Let $V=V_c\sqcup V_e$ with $V_c\neq\emptyset$, and write $L$ in block form with respect to this partition. Then $L_{ee}$ is positive definite, the Schur complement
\begin{equation*}
S=L_{cc}-L_{ce}L_{ee}^{-1}L_{ec}
\end{equation*}
is the Laplacian of a connected weighted graph on $V_c$, and for all $x,y\in V_c$,
\begin{equation*}
(\delta_x-\delta_y)^\top S^+(\delta_x-\delta_y)=M_{AG}(x,y).
\end{equation*}
\end{theorem}

\begin{proof}
\emph{Positive definiteness of $L_{ee}$.} For $u_e\in\mathbb{R}^{V_e}$, extend it by zero on $V_c$ to a vector $u$. Then
\begin{equation*}
u_e^\top L_{ee}u_e=u^\top Lu=\sum_{e'\in E}(u(e'^+)-u(e'^-))^2/\ell_{e'},
\end{equation*}
which vanishes only if $u$ is constant on each connected component $K$ of $G[V_e]$ and zero at every vertex of $V_e$ adjacent to $V_c$. Every such $K$ contains a vertex adjacent to $V_c$: since $G$ is connected and $V_c\neq\emptyset$, some edge leaves $K$, and it cannot end in $V_e\setminus K$. Hence $u_e=0$.

\emph{Harmonic extension.} For fixed $u_c$, the function
\begin{equation*}
u_e\mapsto u^\top Lu=u_c^\top L_{cc}u_c+2u_c^\top L_{ce}u_e+u_e^\top L_{ee}u_e
\end{equation*}
is strictly convex, and its unique minimizer is $u_e=-L_{ee}^{-1}L_{ec}u_c$, with minimum value $u_c^\top Su_c$.

\emph{$S$ is a connected Laplacian.} It is symmetric. From $L\mathbf{1}=0$ we get $L_{ec}\mathbf{1}=-L_{ee}\mathbf{1}$ and hence $S\mathbf{1}=L_{cc}\mathbf{1}+L_{ce}\mathbf{1}=0$.

The matrix $L_{ee}$ is positive definite with nonpositive off-diagonal entries, that is, a nonsingular $M$-matrix, so $L_{ee}^{-1}$ has nonnegative entries \citep[Chapter~6]{berman1994nonnegative}. (Directly: for $a\in V_e$ the column $h=L_{ee}^{-1}\delta_a$ is the unique minimizer of the strictly convex function $F(h)=h^\top L_{ee}h-2h_a$; because $(L_{ee})_{ij}\le0$ for $i\neq j$ and $|h_i||h_j|\ge h_ih_j$, we have $F(|h|)\le F(h)$, so $|h|=h\ge0$.) Since $L_{ce}$ and $L_{ec}$ have nonpositive entries, $L_{ce}L_{ee}^{-1}L_{ec}$ has nonnegative entries, and the off-diagonal entries of $S$ are nonpositive. Thus $S=\operatorname{diag}(W'\mathbf{1})-W'$ for a symmetric $W'\ge0$.

By the harmonic extension, $u_c^\top Su_c=\min_{u_e}u^\top Lu\ge0$, so $S$ is positive semidefinite and $Su_c=0$ exactly when $u_c^\top Su_c=0$. In that case the harmonic extension $u$ satisfies $u^\top Lu=0$, so $u$ is constant and $u_c\in\operatorname{span}(\mathbf{1})$; hence $\ker S=\operatorname{span}(\mathbf{1})$. The indicator vector of a connected component of the graph of $W'$ lies in $\ker S$, so that graph is connected.

\emph{Preservation.} Let $x,y\in V_c$ and let $b=\delta_x-\delta_y$, a vector supported on $V_c$ with $\mathbf{1}^\top b=0$. By Lemma~\ref{lem:variational}, first applied to $L$ and then to $S$,
\begin{equation*}
\begin{split}
b^\top L^+b&=\max_{u_c}\Big(2b_c^\top u_c-\min_{u_e}u^\top Lu\Big)\\
&=\max_{u_c}\big(2b_c^\top u_c-u_c^\top Su_c\big)\\
&=b_c^\top S^+b_c,
\end{split}
\end{equation*}
which is the claim.
\end{proof}

\begin{remark}[Series rule]
\label{rem:series}
Let $v\in V$ have exactly two incident edges, of lengths $\ell_1,\ell_2$, with other ends $u\neq w$. Taking $V_e=\{v\}$ in Theorem~\ref{thm:schur}, the Schur complement differs from $L_{cc}$ only in the entries indexed by $\{u,w\}$ and adds the conductance $(\ell_1+\ell_2)^{-1}$ between $u$ and $w$: the two edges are replaced by one edge of length $\ell_1+\ell_2$, in parallel with any edges already joining $u$ and $w$. Subdividing an edge is the inverse operation, so the AG distance between old vertices is unchanged by subdivision; Lemma~\ref{lem:subdivision} below extends this to the path metric and the tropical distance.
\end{remark}

\begin{corollary}[Kron reduction of a tree; Theorem~\ref{thm:kron}]
\label{cor:kron_tree}
Let $T$ be a tree with root $\rho$, let $V_{\text{reg}}=\{v\neq\rho:\deg v=2\}$ and $V_{\text{core}}=V\setminus V_{\text{reg}}$. Then:
\begin{enumerate}
\item[(i)] the Schur complement of the Laplacian of $T$ with respect to $V_{\text{reg}}$ is the Laplacian of the tree $T_{\min}$ on $V_{\text{core}}$ obtained by replacing every maximal path $u=w_0,w_1,\dots,w_k=w$ whose interior vertices lie in $V_{\text{reg}}$ by a single edge $\{u,w\}$ of length $\sum_{i=1}^k\ell(w_{i-1},w_i)$;
\item[(ii)] $M_{AG}$ restricted to $V_{\text{core}}$ is unchanged, and on any tree $M_{AG}(x,y)=d_T(x,y)$;
\item[(iii)] eliminating the vertices of $V_{\text{reg}}$ one at a time, in any order, produces the same $T_{\min}$; this is the procedure of Algorithm~\ref{alg:ag_feature_extraction};
\item[(iv)] the reduction commutes with adding edges inside $V_{\text{core}}$, in particular with cycle augmentation: for every set $E'$ of new edges with both ends in $V_{\text{core}}$ (for instance edges joining $\rho$ to leaves $v\neq\rho$), the Schur complement of the Laplacian of $T\cup E'$ with respect to $V_{\text{reg}}$ (the set defined in $T$) is the Laplacian of $T_{\min}\cup E'$, and it preserves $M_{AG}$ of $T\cup E'$ on $V_{\text{core}}$. Moreover, $T$ and $T_{\min}$ have the same leaves, so the augmentation rule selects the same edges before and after the reduction.
\end{enumerate}
\end{corollary}

\begin{proof}
Since $\rho\in V_{\text{core}}$, Theorem~\ref{thm:schur} applies to the partition $V=V_{\text{core}}\sqcup V_{\text{reg}}$; let $S$ be the Schur complement it produces. Two observations reduce everything to the elimination of a single vertex.

\emph{Elimination in stages.} Let $V_{\text{reg}}=V_1\sqcup V_2$. Eliminating $V_1$ gives the Laplacian $S_1$ of a connected graph on $V\setminus V_1$ (Theorem~\ref{thm:schur}), and eliminating $V_2$ from $S_1$ gives $S$. Indeed, by the harmonic-extension step of Theorem~\ref{thm:schur}, the quadratic form of a Schur complement is the minimum of the original quadratic form over the eliminated coordinates, so the two-stage complement has quadratic form
\begin{equation*}
u_c\longmapsto\min_{u_{V_2}}\,\min_{u_{V_1}}\,u^\top Lu=\min_{u_{V_{\text{reg}}}}u^\top Lu=u_c^\top Su_c ,
\end{equation*}
and a symmetric matrix is determined by its quadratic form. By induction, every sequence of one-vertex eliminations ends at $S$.

\emph{One elimination.} Let $v\in V_{\text{reg}}$ have neighbours $u$ and $w$; they are distinct because a tree has no multiple edges, and not adjacent because a tree has no triangles. By Remark~\ref{rem:series}, eliminating $v$ turns $T$ into the tree $T'$ in which the edges $\{u,v\},\{v,w\}$ are replaced by one edge of length $\ell(u,v)+\ell(v,w)$. No other degree changes, so $V_{\text{reg}}(T')=V_{\text{reg}}\setminus\{v\}$, and $T'_{\min}=T_{\min}$, because the maximal paths of $T'$ through $V_{\text{reg}}(T')$ are those of $T$ with the two edges at $v$ merged.

(i) and (iii). By induction on $|V_{\text{reg}}|$, the last paragraph shows that every sequence of one-vertex eliminations, in any order, ends at the Laplacian of $T_{\min}$; by elimination in stages it also ends at $S$. Hence $S$ is the Laplacian of $T_{\min}$, and every order produces $T_{\min}$.

(ii) By (i) and Theorem~\ref{thm:schur}, the AG distance of $T_{\min}$ is the restriction of that of $T$ to $V_{\text{core}}$. On a tree $Z_1=\{0\}$, so $\mathcal{F}(x,y)=\{[\gamma_{xy}]\}$ consists of the chain of the unique path; Theorem~\ref{thm:thomson} and Eq.~\eqref{eq:path_energy} then give $M_{AG}(x,y)=\ell(\gamma_{xy})=d_T(x,y)$. The two halves agree: $d_{T_{\min}}=d_T$ on $V_{\text{core}}$, since each edge of $T_{\min}$ has the length of the path it replaces.

Order independence can also be read off from $M_{AG}$ alone. Let $R$ be the matrix of $M_{AG}$ on $V_{\text{core}}$ and
\begin{equation*}
P=I-\frac1{|V_{\text{core}}|}\mathbf{1}\mathbf{1}^\top .
\end{equation*}
If $L'$ is any Laplacian on $V_{\text{core}}$ whose AG matrix is $R$, then $R=d\mathbf{1}^\top+\mathbf{1}d^\top-2L'^+$ with $d=\operatorname{diag}(L'^+)$; since $P\mathbf{1}=0$ and $PL'^+P=L'^+$ (because $L'^+\mathbf{1}=0$), we get
\begin{equation*}
L'^+=-\frac12PRP\qquad\text{and}\qquad L'=\Bigl(-\frac12PRP\Bigr)^+ .
\end{equation*}
Thus $M_{AG}$ on $V_{\text{core}}$ determines the reduced Laplacian.

(iv) The edges of $E'$ join vertices of $V_{\text{core}}$, so they change only the block $L_{cc}$, by adding the Laplacian $L_{E'}$ of $E'$, and leave $L_{ce}$, $L_{ec}$, $L_{ee}$ untouched. The Schur complement of the Laplacian of $T\cup E'$ is therefore $S+L_{E'}$, which by (i) is the Laplacian of $T_{\min}\cup E'$; since $T\cup E'$ is connected, Theorem~\ref{thm:schur} shows that this reduction preserves the AG distance of $T\cup E'$ on $V_{\text{core}}$. Finally, no elimination changes the degree of a vertex of $V_{\text{core}}$, so $T$ and $T_{\min}$ have the same leaves; since the augmentation rule depends only on the leaves and on $\iota$, it selects the same edges before and after the reduction.
\end{proof}

\begin{remark}[Order of operations and sampling]
\label{rem:order}
Kron reduction does \emph{not} commute with the quotient of Section~\ref{app:quotient}, because the rule ``contract the bridges shorter than $\tau$'' is not invariant under subdivision. In the full-resolution tree every sampling interval shorter than $\tau$ on a bridge would be contracted, whereas in $T_{\min}$ the threshold $\tau$ is compared with the total length of an unbranched segment. For instance, if a leaf $v$ is joined to a branch point $u$ by two edges of length $3$ and $\tau=5$, contracting first identifies $v$ with $u$, whereas reducing first produces an edge of length $6$, which is kept. Because the quotient is applied after the reduction, the construction depends on the neurite geometry and not on the sampling step.

Subdividing an edge of $T$ adds a vertex to $V_{\text{reg}}$ and leaves $T_{\min}$, its leaves and $\iota$ on $V_{\text{core}}$ unchanged; hence $\Gamma_\epsilon$, the quotient $\tilde\Gamma_\epsilon$, $M_{AG}$ and the spectral signature are unchanged as well. Resampling that changes the polyline length of a neurite changes the edge lengths of $T_{\min}$, and its effect is controlled by Proposition~\ref{prop:stability}.
\end{remark}

\subsection{Cycle-space augmentation and the period matrix}
\label{app:augmentation}

After the reduction the skeleton is a tree, so its Jacobian is a point. Augmentation joins the root to leaves that lie within Euclidean distance $\epsilon$ of it. The next proposition identifies the resulting homology and gives the period matrix explicitly in terms of the morphology.

\begin{proposition}[Homology and period matrix of the augmented skeleton]
\label{prop:augmentation}
Let $T$ be a tree with root $\rho$, length function $\ell$ and depth $h(v)=d_T(\rho,v)$, and write $[x\to y]_T$ for the chain of the unique path in $T$ from $x$ to $y$. Let $v_1,\dots,v_k$ be distinct leaves different from $\rho$, and let $\Gamma_\epsilon$ be obtained from $T$ by adding, for each $i$, an edge $e_i$ from $v_i$ to $\rho$ of length $a_i>0$ (a parallel edge if $v_i$ is adjacent to $\rho$). For two vertices $u,w$ let $u\wedge w$ be their lowest common ancestor with respect to $\rho$. Then:
\begin{enumerate}
\item[(i)] $\Gamma_\epsilon$ is connected and $\beta_1(\Gamma_\epsilon)=k$;
\item[(ii)] the fundamental cycles $z_i=[\rho\to v_i]_T+\mathbf{1}_{e_i}$, with $e_i$ oriented from $v_i$ to $\rho$, form a $\mathbb{Z}$-basis of $H_1(\Gamma_\epsilon;\mathbb{Z})$;
\item[(iii)] in this basis the period matrix is
\begin{equation}
\label{eq:period_lca}
Q_{ii}=h(v_i)+a_i,\qquad Q_{ij}=h(v_i\wedge v_j)\quad(i\neq j);
\end{equation}
\item[(iv)] $Q=H+\operatorname{diag}(a_1,\dots,a_k)$ with
\begin{equation*}
H=\sum_{e\in E(T)}\ell_e\,s_es_e^\top\succeq0,
\end{equation*}
where $(s_e)_i=1$ if $e$ lies on the path from $\rho$ to $v_i$ and $0$ otherwise; in particular $\lambda_{\min}(Q)\ge\min_ia_i>0$.
\end{enumerate}
\end{proposition}

\begin{proof}
(i) Adding edges preserves connectedness and does not change $n$, while $T$ has $n-1$ edges. Hence $g=(n-1+k)-n+1=k$.

(ii) Each $z_i$ is integral, and $\partial z_i=(\delta_{v_i}-\delta_\rho)+(\delta_\rho-\delta_{v_i})=0$. Moreover $z_i(e_j)=\delta_{ij}$. Given $z\in H_1(\Gamma_\epsilon;\mathbb{Z})$, the chain $z'=z-\sum_iz(e_i)z_i$ is an integral cycle that vanishes on every $e_j$, so it is a cycle supported on $T$; since a tree has no nonzero cycles, $z'=0$. Hence $z=\sum_iz(e_i)z_i$ with integer coefficients, and these coefficients are unique because they are recovered by evaluating on the $e_j$.

(iii) The added edges are distinct and do not lie on $T$, so
\begin{equation*}
Q_{ij}=\langle z_i,z_j\rangle_\ell=\langle[\rho\to v_i]_T,[\rho\to v_j]_T\rangle_\ell+a_i\delta_{ij}.
\end{equation*}
Two root-to-leaf paths in a tree share exactly the path from $\rho$ to $v_i\wedge v_j$, traversed in the same direction, and are disjoint afterwards; so the inner product equals $h(v_i\wedge v_j)$, which is $h(v_i)$ when $i=j$.

(iv) Let $\eta_e=+1$ if the fixed orientation of $e\in E(T)$ points away from $\rho$, and $\eta_e=-1$ otherwise. The path from $\rho$ to $v_i$ traverses $e$, away from the root, exactly when $(s_e)_i=1$, so the coefficient of $e$ in $[\rho\to v_i]_T$ is $\eta_e(s_e)_i$. Since $\eta_e^2=1$,
\begin{equation*}
\langle[\rho\to v_i]_T,[\rho\to v_j]_T\rangle_\ell=\sum_e\ell_e(s_e)_i(s_e)_j=H_{ij},
\end{equation*}
and by (iii) $Q=H+\operatorname{diag}(a)$. For $c\in\mathbb{R}^k$,
\begin{equation*}
c^\top Qc=\sum_{e\in E(T)}\ell_e\,(s_e^\top c)^2+\sum_{i=1}^k a_ic_i^2\ \ge\ \big(\min_ia_i\big)\|c\|_2^2,
\end{equation*}
which shows $H\succeq0$ and $\lambda_{\min}(Q)\ge\min_ia_i$.
\end{proof}

\begin{remark}[Morphological reading of $Q$]
Equation~\eqref{eq:period_lca} says that the period matrix of an augmented neuron is the matrix of \emph{shared soma-to-tip path lengths} of the proximal tips, plus the chord lengths on the diagonal. The entry $Q_{ij}$ is small for two tips whose paths from the soma diverge early and large for two tips on a common long trunk. The explicit baseline of Section~\ref{appendix:proof_tropical_approx} uses exactly these fundamental cycles.
\end{remark}

\begin{remark}[Chord lengths]
\label{rem:implementation_aug}
The augmented edge from $v_i$ to $\rho$ has length $a_i=\|\iota(v_i)-\iota(\rho)\|_2<\epsilon$. When a leaf $v_i$ within distance $\epsilon$ of the root is already adjacent to it in $T_{\min}$, our implementation gives the existing edge the chord length $a_i$ instead of adding a parallel edge. The resulting graph $T'_{\min}$ is again a tree with root-to-leaf edges added, so the results of this appendix apply to it verbatim, and its period matrix is the principal submatrix of Eq.~\eqref{eq:period_lca} indexed by the remaining leaves.
\end{remark}

Adding an edge is a rank-one update of the Laplacian, so the Sherman--Morrison formula gives the exact effect of a single augmented edge on the distances.

\begin{proposition}[Effect of one augmented edge]
\label{prop:rank_one}
Let $v\neq\rho$ be vertices of $\Gamma$, let $\Gamma'$ be obtained from $\Gamma$ by adding an edge between $v$ and $\rho$ of length $a>0$, and let $M$ and $M'$ be the AG distances of $\Gamma$ and $\Gamma'$. For vertices $x,y$ let $f^*_{xy}$ be the unit current of Theorem~\ref{thm:thomson} in $\Gamma$. Then, for all $x,y$,
\begin{align*}
M'(x,y)&=M(x,y)-\frac{\big((\delta_x-\delta_y)^\top L^+(\delta_v-\delta_\rho)\big)^2}{a+M(v,\rho)}\\
&=M(x,y)-\frac{\langle f^*_{xy},f^*_{\rho v}\rangle_\ell^2}{a+M(v,\rho)} .
\end{align*}
If $\Gamma=T$ is a tree with root $\rho$ and depth $h(v)=d_T(\rho,v)$, this reads
\begin{equation*}
M'(x,y)=d_T(x,y)-\sigma_{xy}^2/(h(v)+a)\quad\text{with}\quad\sigma_{xy}=\langle[x\to y]_T,[\rho\to v]_T\rangle_\ell;
\end{equation*}
here $|\sigma_{xy}|$ is the length of the common part of the paths from $x$ to $y$ and from $\rho$ to $v$, and the sign records whether they traverse it in the same direction.
\end{proposition}

\textit{Proof idea.} Adding an edge changes the Laplacian by a rank-one matrix, so on $\mathbf{1}^\perp$ its inverse also changes by a rank-one matrix, which the Sherman--Morrison formula computes. The energy inner product of unit currents then gives the correction a geometric meaning: on a tree the currents are the paths themselves.

\begin{proof}
Put $b=\delta_v-\delta_\rho\in\mathbf{1}^\perp$, so that $L'=L+a^{-1}bb^\top$ by Definition~\ref{def:laplacian} (a parallel edge simply adds its conductance). Both $L$ and $L'$ are Laplacians of connected graphs, so by Lemma~\ref{lem:laplacian} they map $\mathbf{1}^\perp$ bijectively onto itself, and $L^+$, $L'^+$ are their inverses there, extended by zero on $\operatorname{span}(\mathbf{1})$. The rank-one term also maps $\mathbf{1}^\perp$ into itself, because $b\in\mathbf{1}^\perp$. The Sherman--Morrison formula for the invertible operator $L|_{\mathbf{1}^\perp}$ therefore gives, on $\mathbf{1}^\perp$,
\begin{equation*}
L'^+=L^+-\frac{L^+bb^\top L^+}{a+b^\top L^+b},
\end{equation*}
whose denominator is positive. Evaluating the quadratic form at $\delta_x-\delta_y\in\mathbf{1}^\perp$ and using $b^\top L^+b=M(v,\rho)$ yields the first expression.

For the second, write $f^*_{xy}=D_\ell^{-1}B^\top L^+(\delta_y-\delta_x)$ and $f^*_{\rho v}=D_\ell^{-1}B^\top L^+b$. By Eq.~\eqref{eq:laplacian_factor} and $L^+LL^+=L^+$,
\begin{equation*}
\langle f^*_{xy},f^*_{\rho v}\rangle_\ell=(\delta_y-\delta_x)^\top L^+BD_\ell^{-1}B^\top L^+b=(\delta_y-\delta_x)^\top L^+b,
\end{equation*}
which is the numerator up to sign.

Now let $\Gamma=T$. Since $Z_1(T)=\{0\}$, the only unit flow from $x$ to $y$ is $[x\to y]_T$. Hence $f^*_{xy}=[x\to y]_T$ and $f^*_{\rho v}=[\rho\to v]_T$, and the numerator is $\sigma_{xy}^2$. By Corollary~\ref{cor:kron_tree}(ii), $M(x,y)=d_T(x,y)$ and $M(v,\rho)=d_T(\rho,v)=h(v)$.

Finally, if $p$ and $q$ lie on both paths, the unique path in $T$ between them is a subpath of each. Hence the common edges of the two paths form a single subpath, and each path traverses it in one direction. On this subpath the coefficients of $[x\to y]_T$ and $[\rho\to v]_T$ are $\pm1$ with a constant product, and elsewhere their product vanishes, so $|\sigma_{xy}|$ is the length of the common part.
\end{proof}

\begin{remark}
By Proposition~\ref{prop:rank_one}, a proximity edge added to the tree $T$ leaves unchanged exactly the distances whose tree paths avoid (share no edge with) the path from the soma to the tip, and it strictly shortens every other distance by the squared overlap divided by the length $h(v)+a$ of the loop it closes: tips that curve back towards the soma create loops whose lengths and overlaps are recorded by $M_{AG}$.

Once another loop is present, the first formula applies with $L^+$ of the current graph, and the loops interact: for a trunk of length $1$ from $\rho$ to a branch point $w$ carrying two tips $v_1,v_2$ at distance $1$, with chords of length $1$, adding the second chord lowers $M_{AG}(w,v_1)$ from $2/3$ to $5/8$, although the tree path from $w$ to $v_1$ avoids the path from $\rho$ to $v_2$. Corollary~\ref{cor:morph} records the joint effect of all chords as the single quadratic form $s_{xy}^\top Q^{-1}s_{xy}$.
\end{remark}

\subsection{The quotient by short bridges}
\label{app:quotient}

The third operation contracts short bridges. We show that contraction shifts $M_{AG}$ by an explicit amount and leaves the Jacobian untouched. This also quantifies the error of the numerical surrogate $W_{u,v}=10^9$.

\begin{proposition}[Contracting bridges]
\label{prop:bridge}
Let $e$ be a bridge of $\Gamma$, that is, an edge that lies on no cycle. Let $\Gamma/e$ be obtained by contracting $e$, write $\bar x$ for the image of a vertex $x$, and let $\chi_e(x,y)=1$ if $x$ and $y$ lie in different components of $G\setminus e$ and $\chi_e(x,y)=0$ otherwise. No other edge joins $e^-$ and $e^+$, since it would close a cycle through $e$. So $\Gamma/e$ is again a connected multigraph without self-loops, with edge set $E\setminus\{e\}$ and the lengths and orientations inherited from $\Gamma$. Then, for all $x,y\in V$:
\begin{enumerate}
\item[(i)] every unit flow $f\in\mathcal{F}(x,y)$ satisfies $|f_e|=\chi_e(x,y)$;
\item[(ii)] $M_{AG}^{\Gamma}(x,y)=M_{AG}^{\Gamma/e}(\bar x,\bar y)+\ell_e\chi_e(x,y)$ and $d_\Gamma(x,y)=d_{\Gamma/e}(\bar x,\bar y)+\ell_e\chi_e(x,y)$;
\item[(iii)] restriction to $E\setminus\{e\}$ is an isometry $Z_1(\Gamma)\to Z_1(\Gamma/e)$ that maps $H_1(\Gamma;\mathbb{Z})$ onto $H_1(\Gamma/e;\mathbb{Z})$. Hence, in corresponding bases, the period matrices coincide; the induced map $\operatorname{Jac}(\Gamma)\to\operatorname{Jac}(\Gamma/e)$ is an isometry carrying $\Phi_q(x)$ to $\Phi_{\bar q}(\bar x)$; and $d_{\mathrm{Trop}}^{\Gamma}(x,y)=d_{\mathrm{Trop}}^{\Gamma/e}(\bar x,\bar y)$;
\item[(iv)] if the length of $e$ is replaced by $1/w$ (conductance $w$; equivalently $W_{e^-,e^+}=w$), then $M_{AG}^{(w)}(x,y)=M_{AG}^{\Gamma/e}(\bar x,\bar y)+w^{-1}\chi_e(x,y)$ exactly;
\item[(v)] if $F$ is a set of bridges of $\Gamma$, then contracting the edges of $F$ one at a time produces, in any order, the graph $\Gamma/F$ obtained by identifying the vertices of each component of $(V,F)$, and
\begin{equation*}
\begin{split}
M_{AG}^{\Gamma}(x,y)-M_{AG}^{\Gamma/F}(\bar x,\bar y)
&=d_\Gamma(x,y)-d_{\Gamma/F}(\bar x,\bar y)\\
&=\sum_{e\in F}\ell_e\chi_e(x,y),
\end{split}
\end{equation*}
while (iii) holds with $\Gamma/F$ in place of $\Gamma/e$. In particular, giving every edge of $F$ the conductance $w$ shifts $M_{AG}^{\Gamma/F}(\bar x,\bar y)$ by exactly $w^{-1}$ times the number of edges of $F$ that separate $x$ from $y$, and $M_{AG}^{(w)}(x,y)\to M_{AG}^{\Gamma/F}(\bar x,\bar y)$ as $w\to\infty$.
\end{enumerate}
\end{proposition}

\textit{Proof idea.} A bridge is the only edge that crosses the cut it defines, so conservation of flow across that cut fixes the current through it: a unit flow sends all of its mass across $e$ or none of it. Deleting the coordinate $e$ is therefore a bijection between the unit flows of $\Gamma$ and of $\Gamma/e$ that changes every energy by the same constant, and every cycle vanishes on $e$.

\begin{proof}
(i) A path from $e^-$ to $e^+$ in $G\setminus e$ would close a cycle through $e$, and every vertex is joined in $G\setminus e$ to $e^-$ or to $e^+$. So $G\setminus e$ has exactly two components; let $K$ be the one containing $e^+$. Summing $\partial f=\delta_y-\delta_x$ over $K$, all edges inside $K$ cancel and only $e$ contributes, which gives
\begin{equation*}
f_e=\mathbf{1}[y\in K]-\mathbf{1}[x\in K]\in\{-1,0,1\},
\end{equation*}
and $|f_e|=\chi_e(x,y)$.

(ii) Let $r:C_1(\Gamma)\to C_1(\Gamma/e)$ delete the coordinate $e$, and let $p:C_0(\Gamma)\to C_0(\Gamma/e)$ be the linear map with $p\delta_v=\delta_{\bar v}$. Checking on basis chains gives $\partial'\circ r=p\circ\partial$: for $e'\neq e$ both sides send $\mathbf{1}_{e'}$ to $\delta_{\overline{e'^+}}-\delta_{\overline{e'^-}}$, and both send $\mathbf{1}_e$ to $0$. Moreover $\ker p=\operatorname{span}(\delta_{e^+}-\delta_{e^-})$.

We claim that $r$ is a bijection $\mathcal{F}_\Gamma(x,y)\to\mathcal{F}_{\Gamma/e}(\bar x,\bar y)$ that maps integral flows onto integral flows. If $\partial f=\delta_y-\delta_x$, then $\partial'(rf)=p(\delta_y-\delta_x)=\delta_{\bar y}-\delta_{\bar x}$.

For surjectivity, let $f'\in\mathcal{F}_{\Gamma/e}(\bar x,\bar y)$ and let $\tilde f\in C_1(\Gamma)$ be its extension by $0$ on $e$. Then
\begin{equation*}
p\big(\partial\tilde f-(\delta_y-\delta_x)\big)=\partial'f'-(\delta_{\bar y}-\delta_{\bar x})=0,
\end{equation*}
so $\partial\tilde f-(\delta_y-\delta_x)=\alpha(\delta_{e^+}-\delta_{e^-})$ for some $\alpha\in\mathbb{R}$, and $f=\tilde f-\alpha\mathbf{1}_e$ is a unit flow with $rf=f'$.

For injectivity, two unit flows with the same restriction differ by $c\mathbf{1}_e$ with $c(\delta_{e^+}-\delta_{e^-})=0$, so $c=0$. By (i) the lift has $f_e\in\{-1,0,1\}$, so it is integral whenever $f'$ is; restrictions of integral flows are integral.

For corresponding flows, (i) gives $\ell_ef_e^2=\ell_e\chi_e(x,y)$, hence
\begin{equation*}
\|f\|_\ell^2=\|rf\|_\ell^2+\ell_e\chi_e(x,y).
\end{equation*}
The additive constant does not depend on the flow, so the bijection carries minimizers to minimizers. Over real flows the minimum is $M_{AG}$ on both graphs by Theorem~\ref{thm:thomson}, including the case $\bar x=\bar y$, which occurs when $\{x,y\}=\{e^-,e^+\}$. Over integral flows the minimum is the path metric on both graphs by Lemma~\ref{lem:integral_flows}, whose proof in Section~\ref{app:exact_trop} does not use the present proposition. This proves both formulas.

(iii) Since $\mathcal{F}(x,x)=Z_1$, part (i) with $x=y$ shows that every cycle of $\Gamma$ vanishes on $e$. Hence $r$ maps $Z_1(\Gamma)$ into $Z_1(\Gamma/e)$, because $\partial'\circ r=p\circ\partial$. It does so injectively and isometrically, because $\|z\|_\ell^2=\|rz\|_\ell^2+\ell_ez_e^2$ and $z_e=0$. The graph $\Gamma/e$ is connected, with $n-1$ vertices and $m-1$ edges, so $\dim Z_1(\Gamma/e)=g$ and $r|_{Z_1}$ is an isometric isomorphism. Its inverse is extension by zero, which preserves integrality, so $r$ maps $H_1(\Gamma;\mathbb{Z})$ onto $H_1(\Gamma/e;\mathbb{Z})$.

If $z_1,\dots,z_g$ is a $\mathbb{Z}$-basis of $H_1(\Gamma;\mathbb{Z})$, then $rz_1,\dots,rz_g$ is a $\mathbb{Z}$-basis of $H_1(\Gamma/e;\mathbb{Z})$. Its matrix $C'$ is $C$ with the zero column $e$ deleted, so
\begin{equation*}
Q'=C'D_\ell C'^\top=CD_\ell C^\top=Q,
\end{equation*}
and $r$ induces an isometry $\bar r:\operatorname{Jac}(\Gamma)\to\operatorname{Jac}(\Gamma/e)$.

The map $r$ also intertwines the projections: $r\circ\pi=\pi'\circ r$. Indeed, for $f\in C_1(\Gamma)$ the chain $r\pi f$ lies in $Z_1(\Gamma/e)$, and every $z'\in Z_1(\Gamma/e)$ is $rz$ with $z\in Z_1(\Gamma)$ and $z_e=0$. So
\begin{equation*}
\langle rf-r\pi f,z'\rangle_\ell=\langle f-\pi f,z\rangle_\ell-\ell_e(f-\pi f)_ez_e=0,
\end{equation*}
and $r\pi f$ is the orthogonal projection of $rf$.

Now let $\gamma$ be a path from $q$ to $x$ in $\Gamma$. By (ii), $r[\gamma]\in\mathcal{F}_{\mathbb{Z}}(\bar q,\bar x)$. It differs from the chain of any path from $\bar q$ to $\bar x$ in $\Gamma/e$ by an integral cycle, which $\pi'$ fixes. Hence
\begin{equation*}
\Phi_{\bar q}(\bar x)=[\pi'r[\gamma]]=[r\pi[\gamma]]=\bar r(\Phi_q(x)).
\end{equation*}
Since $\bar r$ is an isometry, $d_{\mathrm{Trop}}^{\Gamma}(x,y)=d_{\mathrm{Trop}}^{\Gamma/e}(\bar x,\bar y)$.

(iv) Being a bridge is a combinatorial property, so $e$ remains a bridge when its length is changed to $1/w$, and $\Gamma/e$ does not involve $\ell_e$. Part (ii) for this graph is the claim.

(v) The edges of $F$ form a forest, because a cycle made of edges of $F$ would be a cycle through a bridge. Contracting them one at a time, in any order, therefore never creates a self-loop, and it ends with the vertices of each component of $(V,F)$ identified, which is $\Gamma/F$.

To iterate (ii)--(iv) we need two facts about an edge $e'\neq e$. First, $(\Gamma/e)\setminus e'=(\Gamma\setminus e')/e$. Second, the two ends of $e$ lie in one component of $\Gamma\setminus e'$, so contracting $e$ maps the components of $\Gamma\setminus e'$ bijectively onto those of $(\Gamma\setminus e')/e$. An edge lies on no cycle exactly when removing it disconnects the graph; hence $e'$ is a bridge of $\Gamma/e$ exactly when it is a bridge of $\Gamma$, and in that case $\chi^{\Gamma/e}_{e'}(\bar x,\bar y)=\chi^{\Gamma}_{e'}(x,y)$. Induction on $|F|$, applying (ii) and (iii) to one edge of $F$ at a time, gives the formulas. The last statement is the case $\ell_e=1/w$ for $e\in F$, as in (iv).
\end{proof}

\begin{remark}[The quotient and its implementation]
\label{rem:quotient_impl}
Proposition~\ref{prop:bridge}(v), applied to the set $E_{<\tau}$ of bridges of $\Gamma_\epsilon$ shorter than $\tau$, shows that $\tilde\Gamma_\epsilon=\Gamma_\epsilon/E_{<\tau}$ is well defined, independently of the order of contraction, and that
\begin{equation*}
M_{AG}^{\tilde\Gamma_\epsilon}(\bar x,\bar y)=M_{AG}^{\Gamma_\epsilon}(x,y)-\sum_{e\in E_{<\tau}}\ell_e\chi_e(x,y)\qquad(x,y\in V_{\text{core}}).
\end{equation*}
We have $\bar x=\bar y$ exactly when $x\sim_\tau y$, and $M_{AG}^{\tilde\Gamma_\epsilon}$ is a metric on the vertices of $\tilde\Gamma_\epsilon$ (Proposition~\ref{prop:metric}). The function $(x,y)\mapsto M_{AG}^{\tilde\Gamma_\epsilon}(\bar x,\bar y)$ is therefore a pseudometric on $V_{\text{core}}$ that vanishes exactly on $\sim_\tau$, and a metric on $V_{\text{core}}/\!\sim_\tau$. By the last statement of Proposition~\ref{prop:bridge}(v), the surrogate conductance $w=10^9\,\mu\text{m}^{-1}$ of Algorithm~\ref{alg:ag_feature_extraction} adds, in exact arithmetic, exactly $10^{-9}\,\mu$m for each bridge of $E_{<\tau}$ separating $x$ from $y$.

Proposition~\ref{prop:bridge}(v) also gives an implementation without large conductances: contracting the bridges combinatorially yields a Laplacian whose AG distance is $M_{AG}^{\tilde\Gamma_\epsilon}$.
\end{remark}

\subsection{The closed form on the universal cover}
\label{app:closed_form}

We can now prove the identity of the main text. It is a direct combination of Theorem~\ref{thm:thomson} and Lemma~\ref{lem:period_coords}; for a single edge it reduces to the formula $\|\pi\mathbf{1}_e\|_\ell^2=\ell_e-M_{AG}(e^-,e^+)$ of \citet[Theorems~5.10 and~6.4]{baker2011metric}, and the general orthogonal decomposition is recorded by \citet[Section~7]{dejong2022crossratios}. We first show that none of the quantities involved depends on the combinatorial model.

\begin{lemma}[Subdivision invariance]
\label{lem:subdivision}
Let $G'$ be obtained from $G$ by inserting a new vertex $w$ into an edge $e$, which is replaced by an edge $e'$ from $e^-$ to $w$ and an edge $e''$ from $w$ to $e^+$ with $\ell_{e'}+\ell_{e''}=\ell_e$. Then $M_{AG}$, $d_\Gamma$ and $d_{\mathrm{Trop}}$ computed on $G'$ agree on $V\times V$ with those computed on $G$. Consequently all three depend only on the metric graph $\Gamma$, and for any finite set of points of $\Gamma$ they are defined by choosing a model that contains these points as vertices.
\end{lemma}

\begin{proof}
Let $\iota:C_1(G)\to C_1(G')$ be the refinement map,
\begin{equation*}
\iota\mathbf{1}_e=\mathbf{1}_{e'}+\mathbf{1}_{e''}\quad\text{and}\quad\iota\mathbf{1}_h=\mathbf{1}_h\ \text{for } h\neq e.
\end{equation*}
Regarding $C_0(G)$ as a subspace of $C_0(G')$, we have $\partial'\iota=\partial$, because $\partial'(\mathbf{1}_{e'}+\mathbf{1}_{e''})=\delta_{e^+}-\delta_{e^-}$, and $\iota$ is an isometry, because $\ell_{e'}f_e^2+\ell_{e''}f_e^2=\ell_ef_e^2$. Conversely, a chain $f'\in C_1(G')$ with $(\partial'f')(w)=0$ satisfies $f'_{e'}=f'_{e''}$ and therefore lies in $\operatorname{im}\iota$.

Hence, for $x,y\in V$, the map $\iota$ carries $\mathcal{F}(x,y)$, $\mathcal{F}_{\mathbb{Z}}(x,y)$, $Z_1$ and $H_1(\Gamma;\mathbb{Z})$ isometrically onto the corresponding objects of $G'$, and it carries the chain of a path to the chain of a path of the same length. Theorem~\ref{thm:thomson} gives the invariance of $M_{AG}$ (this is also Remark~\ref{rem:series}), and the correspondence of paths gives that of $d_\Gamma$.

Finally, for $f\in C_1(G)$ the chain $\iota f-\iota\pi f=\iota(f-\pi f)$ is orthogonal to $\iota Z_1=Z_1'$ because $\iota$ is an isometry, while $\iota\pi f\in Z_1'$. Hence $\pi'\iota=\iota\pi$, and for any path $\gamma$ from $x$ to $y$,
\begin{align*}
d'_{\mathrm{Trop}}(x,y)&=\min_{\lambda\in H_1(\Gamma;\mathbb{Z})}\|\iota(\pi[\gamma]-\lambda)\|_\ell\\
&=\min_{\lambda\in H_1(\Gamma;\mathbb{Z})}\|\pi[\gamma]-\lambda\|_\ell\\
&=d_{\mathrm{Trop}}(x,y).
\end{align*}
\end{proof}

\begin{theorem}[Lift-invariant closed form]
\label{thm:closed_form}
Let $x,y\in V$ and let $\gamma$ be any path from $x$ to $y$, with period coordinates $t_\gamma=CD_\ell[\gamma]$. Then
\begin{equation}
\label{eq:closed_form_app}
M_{AG}(x,y)=\ell(\gamma)-t_\gamma^\top Q^{-1}t_\gamma .
\end{equation}
More generally, $M_{AG}(x,y)=\|f\|_\ell^2-t(f)^\top Q^{-1}t(f)$ for every unit flow $f\in\mathcal{F}(x,y)$. The right-hand side is independent of the path (or flow), of the choice of basis of $H_1(\Gamma;\mathbb{Z})$, and of the combinatorial model of $\Gamma$.
\end{theorem}

\begin{proof}
By Theorem~\ref{thm:thomson}(iii), $M_{AG}(x,y)=\|f\|_\ell^2-\|\pi f\|_\ell^2$, and by Lemma~\ref{lem:period_coords}, $\|\pi f\|_\ell^2=t(f)^\top Q^{-1}t(f)$. For $f=[\gamma]$, Eq.~\eqref{eq:path_energy} gives $\|[\gamma]\|_\ell^2=\ell(\gamma)$. The left-hand side involves neither the flow, nor the basis, nor the model (Lemma~\ref{lem:subdivision}), so neither does the right-hand side.
\end{proof}

\begin{remark}[Lifts and lattice translations]
\label{rem:unquantized}
Equation~\eqref{eq:closed_form_app} is Eq.~\eqref{eq:closed_form_main} of the main text. The vector $t_\gamma$ represents the lift of $\Phi_q(y)-\Phi_q(x)$ to the universal cover $Z_1\cong\mathbb{R}^g$ determined by $\gamma$, and $t_\gamma^\top Q^{-1}t_\gamma$ is its squared length. Replacing $\gamma$ by another path moves $t_\gamma$ by a lattice vector and changes $\ell(\gamma)$ by exactly the same amount as the quadratic term; replacing $[\gamma]$ by an arbitrary lattice translate $[\gamma]+\lambda$ does the same with $\ell(\gamma)$ replaced by the energy $\|[\gamma]+\lambda\|_\ell^2$. In this sense no integer search is required: $M_{AG}$ is the part of the energy that is unchanged by lattice translations. Which lift minimizes the quadratic term is the subject of Section~\ref{app:exact_trop}.
\end{remark}

Applied to the pipeline, Theorem~\ref{thm:closed_form} yields a closed form in morphological terms. After augmentation, the bridges of $\Gamma_\epsilon$ are precisely the tree edges that lie on no path from $\rho$ to an augmented leaf. Indeed, an edge lies on a cycle exactly when some integral cycle has a nonzero coefficient on it (Proposition~\ref{prop:bridge}(i)); by Proposition~\ref{prop:augmentation}(ii) every integral cycle is $\sum_ic_iz_i$ with $c\in\mathbb{Z}^k$, whose coefficient is $c_i$ on $e_i$ and $\pm\sum_ic_i(s_e)_i$ on a tree edge $e$, and the latter vanishes for all $c$ exactly when $e$ lies on none of the paths from $\rho$ to $v_1,\dots,v_k$.

\begin{corollary}[Morphological closed form of the CLEARMIND kernel]
\label{cor:morph}
Let $T_{\min}$, $v_1,\dots,v_k$, $a_i$ and $Q$ be as in Proposition~\ref{prop:augmentation}, and let $\tilde T$ be obtained from $T_{\min}$ by contracting the bridges of $\Gamma_\epsilon$ shorter than $\tau$. For $x,y\in V_{\text{core}}$ define $s_{xy}\in\mathbb{R}^k$ by
\begin{equation*}
(s_{xy})_i=\langle[x\to y]_{T_{\min}},[\rho\to v_i]_{T_{\min}}\rangle_\ell,
\end{equation*}
the signed length of the overlap of the tree path from $x$ to $y$ with the path from $\rho$ to $v_i$. Then the AG distance of $\tilde\Gamma_\epsilon$, which is the kernel of Algorithm~\ref{alg:ag_feature_extraction} in exact arithmetic in the limit $W_{u,v}\to\infty$ on contracted bridges (Proposition~\ref{prop:bridge}(v)), is
\begin{equation*}
M_{AG}(x,y)=d_{\tilde T}(\bar x,\bar y)-s_{xy}^\top Q^{-1}s_{xy}.
\end{equation*}
In particular, if no leaf lies within $\epsilon$ of the root ($k=0$), then $M_{AG}=d_{\tilde T}$ is the path metric of the contracted skeleton. For the implementation of Remark~\ref{rem:implementation_aug}, the same formula holds with $T_{\min}$ replaced by $T'_{\min}$ and the overwritten leaves removed from the list $v_1,\dots,v_k$.
\end{corollary}

\begin{proof}
Apply Theorem~\ref{thm:closed_form} to the tree path from $x$ to $y$, which is a path in $\Gamma_\epsilon$; its period coordinates are $t_i=\langle z_i,[x\to y]_T\rangle_\ell=(s_{xy})_i$ because the tree path does not use the added edges. This gives
\begin{equation*}
M_{AG}^{\Gamma_\epsilon}(x,y)=d_{T_{\min}}(x,y)-s_{xy}^\top Q^{-1}s_{xy}.
\end{equation*}

Now contract the bridges shorter than $\tau$. By Proposition~\ref{prop:bridge}(v), this lowers $M_{AG}(x,y)$ by $\sum_e\ell_e\chi_e(x,y)$ and leaves $Q$ unchanged. The tree distance drops by the same amount: the tree path from $x$ to $y$ is a path in $\Gamma_\epsilon$, so it traverses a bridge $e$ exactly when $e$ separates $x$ from $y$, and contracting the bridges turns it into the tree path from $\bar x$ to $\bar y$. A contracted bridge lies on no path from $\rho$ to a $v_i$, so $s_{xy}$ is unchanged as well.
\end{proof}

\subsection{The exact tropical polarization distance}
\label{app:exact_trop}

Theorem~\ref{thm:closed_form} holds for every lift, so it is natural to ask which lift makes the quadratic term smallest. By Lemma~\ref{lem:period_coords}, the minimum over all integral lifts is the squared tropical polarization distance, a closest vector problem. We show that this problem has an explicit solution, the lift along a geodesic. The infinitesimal form of this statement along a single edge is due to \citet{baker2011metric} (Remark~\ref{rem:special_cases}), and \citet{cao2025computing} give the bound $d_{\mathrm{Trop}}^2\le d_\Gamma$ (Remark~\ref{rem:nphard}); Theorem~\ref{thm:exact_trop}, together with Lemma~\ref{lem:subdivision}, sharpens it to an identity for every metric graph and every pair of its points. The proof rests on an elementary fact: an integral flow cannot split a unit of current into fractions, and fractional values are the only way in which a flow $f$ can have energy $\|f\|_\ell^2$ smaller than its length $\sum_e\ell_e|f_e|$.

\begin{lemma}[Integral unit flows cost at least the geodesic]
\label{lem:integral_flows}
For all $x,y\in V$ and every $f\in\mathcal{F}_{\mathbb{Z}}(x,y)$,
\begin{equation*}
\|f\|_\ell^2\ \ge\ \sum_{e\in E}\ell_e|f_e|\ \ge\ d_\Gamma(x,y),
\end{equation*}
and $\|f\|_\ell^2=d_\Gamma(x,y)$ if and only if $f=[\gamma]$ for a geodesic $\gamma$ from $x$ to $y$. Consequently

\begin{equation*}
\min_{f\in\mathcal{F}_{\mathbb{Z}}(x,y)}\|f\|_\ell^2=d_\Gamma(x,y),
\end{equation*}

and the minimizers are exactly the chains of geodesics.
\end{lemma}

\begin{proof}
For an integer $k$ we have $k^2\ge|k|$, which gives the first inequality. For the second, consider the potential $\varphi(v)=d_\Gamma(x,v)$. The triangle inequality for $d_\Gamma$ gives $|\varphi(e^+)-\varphi(e^-)|\le\ell_e$ for every edge. Pairing $\partial f=\delta_y-\delta_x$ with $\varphi$,
\begin{align*}
d_\Gamma(x,y)&=\varphi(y)-\varphi(x)=\sum_{v\in V}(\partial f)(v)\varphi(v)\\
&=\sum_{e\in E}f_e\big(\varphi(e^+)-\varphi(e^-)\big)\\
&\le\sum_{e\in E}\ell_e|f_e| .
\end{align*}
If $f=[\gamma]$ with $\gamma$ a geodesic, then $\|f\|_\ell^2=\ell(\gamma)=d_\Gamma(x,y)$ by Eq.~\eqref{eq:path_energy}, so all three quantities coincide.

Conversely, suppose $\|f\|_\ell^2=d_\Gamma(x,y)$, so that both inequalities are equalities. Equality in the first forces $f_e^2=|f_e|$, that is, $f_e\in\{-1,0,1\}$ for every edge.

Reverse the orientation of each edge with $f_e=-1$. The support of $f$ becomes a set $S$ of arcs in which every vertex other than $x$ and $y$ is the head of as many arcs as it is the tail of, while for $x\neq y$ the vertex $x$ is the tail of one more arc and $y$ the head of one more arc. Such an arc set is the disjoint union of the arcs of a directed path $\gamma$ from $x$ to $y$ and of directed simple cycles $c_1,\dots,c_r$: follow unused arcs from $x$ until $y$ is reached, shortcut repeated vertices, and decompose the balanced remainder into cycles.

Hence $f=[\gamma]+\sum_i[c_i]$ with pairwise disjoint supports, and equality in the second inequality reads

\begin{equation*}
\ell(\gamma)+\sum_i\ell(c_i)=d_\Gamma(x,y)\le\ell(\gamma).
\end{equation*}

Every cycle has positive length, so $r=0$ and $\ell(\gamma)=d_\Gamma(x,y)$, that is, $f=[\gamma]$ with $\gamma$ a geodesic. For $x=y$ the same argument gives $f=0$, the chain of the trivial path.
\end{proof}

The second inequality holds for \emph{real} flows as well; only the first uses integrality. This is why the real and integral minimization problems differ, and the difference of their minima is the squared tropical distance.

\begin{theorem}[Exact tropical polarization distance]
\label{thm:exact_trop}
For all vertices $x,y$ of $\Gamma$, and hence, computing in a model that contains them as vertices (Lemma~\ref{lem:subdivision}), for all points $x,y$ of $\Gamma$:
\begin{enumerate}
\item[(i)] $\displaystyle\min_{f\in\mathcal{F}(x,y)}\|f\|_\ell^2=M_{AG}(x,y)$ and $\displaystyle\min_{f\in\mathcal{F}_{\mathbb{Z}}(x,y)}\|f\|_\ell^2=d_\Gamma(x,y)$;
\item[(ii)] the squared tropical polarization distance is the gap between the two minima:
\begin{equation}
\label{eq:exact_trop}
d_{\mathrm{Trop}}(x,y)^2=d_\Gamma(x,y)-M_{AG}(x,y);
\end{equation}
\item[(iii)] for any path $\gamma$ from $x$ to $y$ and any geodesic $\gamma^*$, the closest vector problem in Eq.~\eqref{eq:dtrop_coords} is solved by the integer vector $n^*$ defined by $[\gamma]-[\gamma^*]=C^\top n^*$, which exists and is unique because the rows of $C$ form a $\mathbb{Z}$-basis of $H_1(\Gamma;\mathbb{Z})$;
\item[(iv)] a path $\gamma$ from $x$ to $y$ satisfies $t_\gamma^\top Q^{-1}t_\gamma=d_{\mathrm{Trop}}(x,y)^2$ if and only if $\gamma$ is a geodesic: the geodesic lifts are exactly the closest ones;
\item[(v)] $0\le M_{AG}(x,y)\le d_\Gamma(x,y)$, with equality on the right if and only if $\Phi_q(x)=\Phi_q(y)$ for one, equivalently every, base vertex $q$; in particular $M_{AG}=d_\Gamma$ when $g=0$.
\end{enumerate}
\end{theorem}

\textit{Proof idea.} By the Pythagorean identity~\eqref{eq:pythagoras}, $\|f\|_\ell^2-\|\pi f\|_\ell^2$ is the same number, $M_{AG}(x,y)$, for every unit flow $f$. Minimizing $\|f\|_\ell^2$ over integral flows therefore minimizes $\|\pi f\|_\ell^2$ over integral flows, and the projections of integral flows are exactly the lattice translates of one lift. The first minimum is $d_\Gamma(x,y)$ by Lemma~\ref{lem:integral_flows}; the second is $d_{\mathrm{Trop}}(x,y)^2$ by definition.

\begin{proof}
If $x$ or $y$ is not a vertex of $G$, subdividing the edges that contain them gives a model in which both are vertices, and by Lemma~\ref{lem:subdivision} the values of $M_{AG}$, $d_\Gamma$ and $d_{\mathrm{Trop}}$ at $(x,y)$ are those computed on any such model; it therefore suffices to prove (i)--(v) for vertices.

(i) The first equality is Theorem~\ref{thm:thomson}, and the second is Lemma~\ref{lem:integral_flows}.

(ii) Fix a path $\gamma$ from $x$ to $y$. For $f\in\mathcal{F}_{\mathbb{Z}}(x,y)$, the chain $\lambda_f=f-[\gamma]$ is an integral cycle, and conversely $[\gamma]+\lambda\in\mathcal{F}_{\mathbb{Z}}(x,y)$ for every $\lambda\in H_1(\Gamma;\mathbb{Z})$. Hence $f\mapsto\lambda_f$ is a bijection from $\mathcal{F}_{\mathbb{Z}}(x,y)$ onto $H_1(\Gamma;\mathbb{Z})$. Since $\pi$ is the identity on $Z_1$, we have $\pi f=\pi[\gamma]+\lambda_f$, and Eq.~\eqref{eq:pythagoras} gives
\begin{equation*}
\|f\|_\ell^2=M_{AG}(x,y)+\|\pi[\gamma]+\lambda_f\|_\ell^2 .
\end{equation*}
Minimizing both sides over $f\in\mathcal{F}_{\mathbb{Z}}(x,y)$, equivalently over $\lambda_f\in H_1(\Gamma;\mathbb{Z})$, the left-hand side becomes $d_\Gamma(x,y)$ by (i), and the right-hand side becomes $M_{AG}(x,y)+d_{\mathrm{Trop}}(x,y)^2$ by Definition~\ref{def:jacobian} (with $\lambda=-\lambda_f$). This proves Eq.~\eqref{eq:exact_trop}.

(iii) The chain $[\gamma]-[\gamma^*]$ is an integral cycle, so $n^*$ exists and is unique. Take $f=[\gamma^*]$, so $\lambda_f=-C^\top n^*$ and $\pi[\gamma^*]=\pi[\gamma]-C^\top n^*$. By Lemma~\ref{lem:period_coords}, $t(\pi[\gamma])=t_\gamma$ and $t(C^\top n^*)=Qn^*$, so

\begin{equation*}
(t_\gamma-Qn^*)^\top Q^{-1}(t_\gamma-Qn^*)=\|\pi[\gamma^*]\|_\ell^2 .
\end{equation*}

By Eq.~\eqref{eq:pythagoras} and Lemma~\ref{lem:integral_flows} this equals

\begin{equation*}
\ell(\gamma^*)-M_{AG}(x,y)=d_\Gamma(x,y)-M_{AG}(x,y)=d_{\mathrm{Trop}}(x,y)^2,
\end{equation*}

so $n^*$ attains the minimum.

(iv) By Theorem~\ref{thm:closed_form}, $t_\gamma^\top Q^{-1}t_\gamma=\ell(\gamma)-M_{AG}(x,y)$, and by (ii), $d_{\mathrm{Trop}}(x,y)^2=d_\Gamma(x,y)-M_{AG}(x,y)$. The two agree exactly when $\ell(\gamma)=d_\Gamma(x,y)$.

(v) Both $d_{\mathrm{Trop}}$ and $M_{AG}$ are nonnegative, so Eq.~\eqref{eq:exact_trop} gives $0\le M_{AG}\le d_\Gamma$, with $M_{AG}=d_\Gamma$ exactly when $d_{\mathrm{Trop}}(x,y)=0$, that is, when $\Phi_q(x)-\Phi_q(y)$ is the zero class of $Z_1/H_1(\Gamma;\mathbb{Z})$; this does not depend on $q$ by the second comment after Definition~\ref{def:jacobian}. If $g=0$ the Jacobian is a point.
\end{proof}

\begin{corollary}[Exact distance matrices in polynomial time]
\label{cor:complexity}
For every finite set of vertices, the matrix of squared tropical polarization distances between their Abel--Jacobi images equals $D_\Gamma-M_{AG}$, where $D_\Gamma$ is the matrix of path distances. It can be computed in $\mathcal{O}(n^3+m)$ time with one all-pairs shortest-path computation and one pseudoinverse, without lattice reduction, enumeration, or integer programming. The same holds for any finite set of points of $\Gamma$: subdividing the edges at the points in their interiors changes none of $M_{AG}$, $d_\Gamma$ and $d_{\mathrm{Trop}}$ on the existing vertices (Lemma~\ref{lem:subdivision}).
\end{corollary}

\begin{proof}
By Theorem~\ref{thm:exact_trop}(ii), applied to every pair, the matrix of squared tropical polarization distances is $D_\Gamma-M_{AG}$.

One pass over the edges forms $W$ and $L$ (Definition~\ref{def:laplacian}) and, for every two adjacent vertices, the minimal length of an edge joining them, with $\mathcal{O}(n^2+m)$ operations. The Floyd--Warshall recursion on these minimal lengths computes the minimal length of a walk between every two vertices with $\mathcal{O}(n^3)$ operations; all lengths are positive, so a shortest walk visits no vertex twice, and these minima are the entries of $D_\Gamma$ (Definition~\ref{def:chains}).

For the pseudoinverse put $P=\frac1n\mathbf{1}\mathbf{1}^\top$, so that $P^2=P$. By Lemma~\ref{lem:laplacian}, $LL^+=I-P$, $LP=0$ and $PL^+=0$, so

\begin{equation*}
(L+P)(L^++P)=I-P+P^2=I,
\end{equation*}

that is,
\begin{equation*}
L^+=(L+P)^{-1}-P .
\end{equation*}

The matrix $L+P$ is moreover positive definite: $v^\top(L+P)v=v^\top Lv+(\mathbf{1}^\top v)^2/n$ is a sum of two nonnegative terms (Lemma~\ref{lem:laplacian}), and both vanish only if $v\in\operatorname{span}(\mathbf{1})$ and $\mathbf{1}^\top v=0$, that is, only if $v=0$. Gaussian elimination, which needs no pivoting for a positive definite matrix, therefore computes $(L+P)^{-1}$, and hence $L^+$, with $\mathcal{O}(n^3)$ arithmetic operations; the entries $M_{AG}(x,y)=L^+_{x,x}+L^+_{y,y}-2L^+_{x,y}$ and the matrix $D_\Gamma-M_{AG}$ follow with $\mathcal{O}(n^2)$ more, for a total of $\mathcal{O}(n^3+m)$.

When $G$ has no parallel edges, $m\le n(n-1)/2$; the augmented skeleton $\Gamma_\epsilon$ of Proposition~\ref{prop:augmentation} adds $k\le n-1$ edges to a tree with $n-1$ edges, so $m\le2(n-1)$. In both cases the total is $\mathcal{O}(n^3)$.
\end{proof}

\begin{remark}[Relation to the complexity of lattice problems]
\label{rem:nphard}
The closest vector problem is NP-hard on general lattices \citep{vanEmdeBoas1981}, and \citet[Section~5.2]{cao2025computing} show that computing distances on generic flat tori is equivalent to such problems; accordingly, they compute pairwise tropical distances of Abel--Jacobi transforms with exact CVP solvers, lattice reduction and mixed-integer programming.

Tropical Jacobian lattices are far from general: $H_1(\Gamma;\mathbb{Z})$ is the integral kernel of the totally unimodular incidence matrix, with a diagonal inner product, and for such zonotopal lattices the closest vector problem is solvable in polynomial time for every target \citep{mccormick2021polynomial}. Theorem~\ref{thm:exact_trop} goes further for the targets that arise from pairs of points of $\Gamma$: no lattice algorithm is needed, because the geodesic lift is always a closest lattice translate.

This is consistent with \citet[Theorem~5.10]{cao2025computing}, which characterizes the graphs whose whole torus is isometric to a flat torus with a rectangular lattice, so that distances between \emph{arbitrary} points decouple coordinatewise; Theorem~\ref{thm:exact_trop} concerns only the targets $\Phi_q(y)-\Phi_q(x)$, and for these every graph admits the closed form $d_\Gamma-M_{AG}$.

It also sharpens the H\"older estimate $d_{\mathrm{Trop}}(x,y)\le d_\Gamma(x,y)^{1/2}$ of \citet[Theorem~5.1]{cao2025computing}: Eq.~\eqref{eq:exact_trop} identifies the deficit $d_\Gamma-d_{\mathrm{Trop}}^2$ exactly as the AG distance.
\end{remark}

\begin{remark}[Two special cases]
\label{rem:special_cases}
For a single loop of total length $\ell_\circ$ and two points at arc distance $s\le\ell_\circ/2$, one has $d_\Gamma=s$ and $M_{AG}=s(\ell_\circ-s)/\ell_\circ$ (two resistors in parallel), so $d_{\mathrm{Trop}}^2=s^2/\ell_\circ$, the squared distance on a circle of circumference $\sqrt{\ell_\circ}$ traversed at speed $1/\sqrt{\ell_\circ}$.

Next let $p,q$ be two points of an edge $e$ such that the segment $\sigma\subset e$ between them is a geodesic of $\Gamma$, of length $s$. After subdividing $e$ at $p$ and $q$ (Lemma~\ref{lem:subdivision}), the chain $[\sigma]-\iota\big((s/\ell_e)\mathbf{1}_e\big)$ is orthogonal to every cycle, because every cycle takes the same value on the three pieces of $e$; so the projection of $[\sigma]$ is $\iota\big((s/\ell_e)\pi\mathbf{1}_e\big)$. Theorem~\ref{thm:closed_form} with $\gamma=e$ gives $\|\pi\mathbf{1}_e\|_\ell^2=\ell_e-M_{AG}(e^-,e^+)$, and since $\sigma$ is a geodesic, Theorem~\ref{thm:exact_trop}(iv) yields
\begin{equation*}
d_{\mathrm{Trop}}(p,q)^2=\|\pi[\sigma]\|_\ell^2=\frac{s^2F(e)}{\ell_e},\qquad F(e)=1-\frac{M_{AG}(e^-,e^+)}{\ell_e},
\end{equation*}
where $F(e)$ is the Foster coefficient of $e$, which vanishes exactly on bridges. This recovers the local description of the Abel--Jacobi map by \citet[Theorems~4.1, 5.10 and~6.4]{baker2011metric}: $\Phi$ contracts exactly the bridges and pulls the polarization metric on an edge $e$ back to

\begin{equation*}
\frac{F(e)}{\ell_e}\,dx^2 .
\end{equation*}
\end{remark}

\begin{remark}[Numerical confirmation]
\label{rem:numerical_check}
As a sanity check of Theorem~\ref{thm:exact_trop}, we compared Eq.~\eqref{eq:exact_trop} with brute-force enumeration of Eq.~\eqref{eq:dtrop_coords} over $n\in\{-4,\dots,4\}^g$ on $300$ random connected graphs with $5$--$10$ vertices, genus $1$--$4$ and random edge lengths ($7{,}726$ vertex pairs); the maximal absolute discrepancy was $4.5\times10^{-14}$. On the same pairs, Babai rounding of the spanning-tree lift returned a non-optimal lattice vector for $877$ pairs ($11.4\%$).
\end{remark}

\subsection{The explicit lattice baseline}
\label{appendix:proof_tropical_approx}

We now place the explicit baseline of Section~\ref{subsec:explicit_vs_implicit} in the framework above. It fixes the fundamental cycles of a spanning tree $T$ (for the augmented skeleton, those of Proposition~\ref{prop:augmentation}), lifts every vertex $v$ along the tree path from the root, and rounds in coefficient coordinates. In our notation, with $t_v=t([\rho\to v]_T)$ and $\alpha_v=Q^{-1}t_v$ (the matrices $V_{\text{universal}}$ and $\tilde V=Q^{-1}V_{\text{universal}}$ of the implementation), it sets, for a pair $(x,y)$,
\begin{gather*}
t=t([x\to y]_T),\qquad\alpha=Q^{-1}t=\alpha_y-\alpha_x,\\
\hat n=\lfloor\alpha\rceil,\qquad\delta=\alpha-\hat n,\qquad d_{\text{Babai}}(x,y)^2=\delta^\top Q\delta,
\end{gather*}
where $\lfloor\cdot\rceil$ rounds each coordinate to the nearest integer. Since $t=Q\alpha$, we have
\begin{equation*}
(t-Q\hat n)^\top Q^{-1}(t-Q\hat n)=(Q\delta)^\top Q^{-1}(Q\delta)=\delta^\top Q\delta{:}
\end{equation*}
the Babai value is the objective of Eq.~\eqref{eq:dtrop_coords} at $n=\hat n$, the squared distance from $t$ to the lattice point $Q\hat n\in\Lambda$. Equivalently, $d_{\text{Babai}}(x,y)^2=\|\pi\hat f\|_\ell^2$ for the \emph{rounded lift} $\hat f=[x\to y]_T-C^\top\hat n\in\mathcal{F}_{\mathbb{Z}}(x,y)$, since $t(\hat f)=t-Q\hat n$ (Lemma~\ref{lem:period_coords}).

\begin{proposition}[Babai rounding versus the exact value]
\label{prop:babai}
In the notation above:
\begin{enumerate}
\item[(i)] $d_{\text{Babai}}(x,y)^2=\|\hat f\|_\ell^2-M_{AG}(x,y)\ge d_\Gamma(x,y)-M_{AG}(x,y)=d_{\mathrm{Trop}}(x,y)^2$, with equality if and only if the rounded lift $\hat f$ is the chain of a geodesic;
\item[(ii)] if $\hat n=0$, then $d_{\text{Babai}}(x,y)^2=t^\top Q^{-1}t=d_T(x,y)-M_{AG}(x,y)$;
\item[(iii)] the rounding map is not equivariant under changes of basis, $\lfloor U^{-\top}\alpha\rceil\neq U^{-\top}\lfloor\alpha\rceil$ in general for $U\in GL_g(\mathbb{Z})$, so $d_{\text{Babai}}$ depends on the chosen cycle basis, already through the choice of the spanning tree, whereas $d_{\mathrm{Trop}}$ and $M_{AG}$ do not.
\end{enumerate}
\end{proposition}

\begin{proof}
(i) Theorem~\ref{thm:thomson}(iii) gives $\|\pi\hat f\|_\ell^2=\|\hat f\|_\ell^2-M_{AG}(x,y)$. By Lemma~\ref{lem:integral_flows}, $\|\hat f\|_\ell^2\ge d_\Gamma(x,y)$, with equality exactly when $\hat f$ is the chain of a geodesic, and Theorem~\ref{thm:exact_trop} identifies $d_\Gamma-M_{AG}$ with $d_{\mathrm{Trop}}^2$.

(ii) If $\hat n=0$ then $\delta=\alpha$ and $\delta^\top Q\delta=t^\top Q^{-1}t$, which equals $d_T(x,y)-M_{AG}(x,y)$ by Theorem~\ref{thm:closed_form} applied to the tree path.

(iii) Under $C\mapsto UC$ the coefficient vector of a fixed cycle transforms as $\alpha\mapsto U^{-\top}\alpha$ (because $C^\top\alpha=(UC)^\top U^{-\top}\alpha$), and $Q\mapsto UQU^\top$.

Let $\Gamma$ be the theta graph: vertices $a,b$ joined by three edges $e_1,e_2,e_3$ of length $1$, oriented from $a$ to $b$, where $e_1$ is subdivided by a vertex $x$ into $e_1'$ from $a$ to $x$ (length $0.9$) and $e_1''$ from $x$ to $b$ (length $0.1$).

For the spanning tree $T_1=\{e_1',e_1''\}$ the fundamental cycles are $z_1=\mathbf{1}_{e_2}-\mathbf{1}_{e_1'}-\mathbf{1}_{e_1''}$ and $z_2=\mathbf{1}_{e_3}-\mathbf{1}_{e_1'}-\mathbf{1}_{e_1''}$, with $Q=\left(\begin{smallmatrix}2&1\\1&2\end{smallmatrix}\right)$. The tree path from $x$ to $a$ has chain $-\mathbf{1}_{e_1'}$, so
\begin{equation*}
t=(0.9,0.9)^\top,\quad\alpha=(0.3,0.3)^\top,\quad\hat n=0\quad\text{and}\quad d_{\text{Babai}}(x,a)^2=\alpha^\top Q\alpha=0.54 .
\end{equation*}

For the spanning tree $T_2=\{e_1'',e_2\}$ the fundamental cycles are $z_1'=-z_1$ and $z_2'=z_2-z_1=\mathbf{1}_{e_3}-\mathbf{1}_{e_2}$, that is, $U=\left(\begin{smallmatrix}-1&0\\-1&1\end{smallmatrix}\right)$, and again $UQU^\top=Q$. The tree path from $x$ to $a$ now runs through $b$, with chain $\mathbf{1}_{e_1''}-\mathbf{1}_{e_2}$, so
\begin{equation*}
t'=(1.1,1)^\top,\quad\alpha'=(0.4,0.3)^\top,\quad\hat n'=0\quad\text{and}\quad d_{\text{Babai}}(x,a)^2=\alpha'^\top Q\alpha'=0.74 .
\end{equation*}

Since $d_\Gamma(x,a)=0.9$ and $M_{AG}(x,a)=0.9\cdot0.6/1.5=0.36$ (a resistance $0.9$ in parallel with $0.1+\frac12$), Theorem~\ref{thm:exact_trop} gives $d_{\mathrm{Trop}}(x,a)^2=0.54$: the first spanning tree returns the exact value and the second does not.
\end{proof}

The exact formula also shows \emph{when} rounding fails, in terms of the circuits of $\Gamma$. An \emph{oriented circuit} is the chain $c=\sum_{i=1}^k\sigma_i\mathbf{1}_{e_i}$ of a closed walk $v_0,e_1,v_1,\dots,e_k,v_k=v_0$ with $k\ge2$, distinct edges and distinct vertices $v_0,\dots,v_{k-1}$, where $\sigma_i=+1$ if $e_i$ is traversed from tail to head and $\sigma_i=-1$ otherwise (two parallel edges form a circuit with $k=2$). It is an integral cycle, and as in Eq.~\eqref{eq:path_energy} its energy is the length of the loop,
\begin{equation*}
\ell(c)=\|c\|_\ell^2=\sum_{i=1}^k\ell_{e_i} .
\end{equation*}
For a chain $f$, the number $\langle f,c\rangle_\ell$ is the length that $f$ runs along $c$ in the direction of $c$, minus the length it runs against it.

\begin{proposition}[Closest integral lifts are exactly the geodesics]
\label{prop:voronoi}
Let $x,y\in V$ and $f\in\mathcal{F}_{\mathbb{Z}}(x,y)$. The following are equivalent:
\begin{enumerate}
\item[(a)] $\|\pi f\|_\ell^2=d_{\mathrm{Trop}}(x,y)^2$, that is, $\pi f$ lies in the Voronoi cell of the lattice $H_1(\Gamma;\mathbb{Z})$ at the origin;
\item[(b)] $f$ never covers more than half of a circuit:
\begin{equation}
\label{eq:half_circuit}
2\langle f,c\rangle_\ell\le\ell(c)\qquad\text{for every oriented circuit }c;
\end{equation}
\item[(c)] $f=[\gamma]$ for a geodesic $\gamma$ from $x$ to $y$.
\end{enumerate}
In particular, Babai rounding is exact, $d_{\text{Babai}}(x,y)=d_{\mathrm{Trop}}(x,y)$, if and only if the rounded lift $\hat f$ is the chain of a geodesic, equivalently if and only if $\hat f$ satisfies Eq.~\eqref{eq:half_circuit}.
\end{proposition}

\textit{Proof idea.} The equivalence (a)$\Leftrightarrow$(b) describes the Voronoi cell by the bisecting hyperplanes of lattice vectors, and a conformal decomposition shows that the oriented circuits are the only lattice vectors that need to be tested. The equivalence (a)$\Leftrightarrow$(c) is Thomson's principle once more: $\|f\|_\ell^2-\|\pi f\|_\ell^2=M_{AG}(x,y)$ does not depend on $f$, so $\pi f$ is a closest lattice translate exactly when $f$ is an integral unit flow of minimal energy, and Lemma~\ref{lem:integral_flows} identifies these flows as the geodesics.

\begin{proof}
\emph{(a)$\Leftrightarrow$(b).} Since $f-\pi f\perp Z_1$, we have $\langle\pi f,\lambda\rangle_\ell=\langle f,\lambda\rangle_\ell$ for every $\lambda\in H_1(\Gamma;\mathbb{Z})$, and therefore
\begin{equation*}
\|\pi f-\lambda\|_\ell^2-\|\pi f\|_\ell^2=\|\lambda\|_\ell^2-2\langle f,\lambda\rangle_\ell .
\end{equation*}
For a path $\gamma$ from $x$ to $y$ the chain $f-[\gamma]$ is an integral cycle, so $\pi f=\pi[\gamma]+(f-[\gamma])$ and
\begin{equation*}
d_{\mathrm{Trop}}(x,y)^2=\min_{\lambda\in H_1(\Gamma;\mathbb{Z})}\|\pi f-\lambda\|_\ell^2
\end{equation*}
by Lemma~\ref{lem:period_coords}. Hence (a) holds if and only if $2\langle f,\lambda\rangle_\ell\le\|\lambda\|_\ell^2$ for every $\lambda\in H_1(\Gamma;\mathbb{Z})$. Taking $\lambda=c$ gives (b).

For the converse we use that every nonzero $\lambda\in H_1(\Gamma;\mathbb{Z})$ is a \emph{conformal} sum $\lambda=c_1+\dots+c_k$ of oriented circuits, meaning that $c_i(e)\lambda(e)>0$ whenever $c_i(e)\neq0$.

To see this, orient every edge of the support of $\lambda$ in the direction in which $\lambda$ flows through it. Because $\partial\lambda=0$, every vertex of this directed graph that has an incoming edge also has an outgoing one, so walking forward we close a directed simple closed walk; its chain $c_1$ is an oriented circuit conformal to $\lambda$. The cycle $\lambda-c_1$ is again integral and conformal to $\lambda$, with $\sum_e|\lambda(e)-c_1(e)|<\sum_e|\lambda(e)|$, and induction on $\sum_e|\lambda(e)|$ completes the decomposition.

Conformality gives $\langle c_i,c_j\rangle_\ell=\sum_e\ell_ec_i(e)c_j(e)\ge0$ for all $i,j$, and (b) yields
\begin{equation*}
\begin{split}
\|\lambda\|_\ell^2&=\sum_{i,j}\langle c_i,c_j\rangle_\ell\\
&\ge\sum_i\ell(c_i)\\
&\ge\sum_i2\langle f,c_i\rangle_\ell=2\langle f,\lambda\rangle_\ell .
\end{split}
\end{equation*}

\emph{(a)$\Leftrightarrow$(c).} By Eq.~\eqref{eq:pythagoras} and Theorem~\ref{thm:exact_trop}(ii),
\begin{equation*}
\|\pi f\|_\ell^2-d_{\mathrm{Trop}}(x,y)^2=\|f\|_\ell^2-d_\Gamma(x,y),
\end{equation*}
so (a) says that $\|f\|_\ell^2=d_\Gamma(x,y)$, which by Lemma~\ref{lem:integral_flows} holds exactly when $f$ is the chain of a geodesic.

Finally, $\hat f\in\mathcal{F}_{\mathbb{Z}}(x,y)$ and $\|\pi\hat f\|_\ell^2=d_{\text{Babai}}(x,y)^2$, so the last assertion is the equivalence of (a), (b) and (c) for $f=\hat f$.
\end{proof}

In words: an integral lift is a closest lattice translate exactly when it never runs, in net, along more than half of any loop, and the lifts with this property are exactly the geodesics, since otherwise going around the other way would be shorter. Babai rounding selects its lift by rounding coordinates in the chosen basis, and it errs exactly on the pairs for which the selected lift is not a geodesic; in the theta graph above, the lift kept by $T_2$ is the path $x\to b\to a$, which covers length $1.1$ of the circuit $e_1\cup e_2$ of length $2$, and in Remark~\ref{rem:numerical_check} such pairs are $11.4\%$ of all pairs.

For $x,y\in V$, Proposition~\ref{prop:voronoi} determines every solution of the closest vector problem in Eq.~\eqref{eq:dtrop_coords}. Let $\gamma$ be a path from $x$ to $y$ and $t=t([\gamma])$. For $n\in\mathbb{Z}^g$ the chain $f=[\gamma]-C^\top n$ lies in $\mathcal{F}_{\mathbb{Z}}(x,y)$ and satisfies $t(f)=t-Qn$, so
\begin{equation*}
\|\pi f\|_\ell^2=(t-Qn)^\top Q^{-1}(t-Qn)
\end{equation*}
by Lemma~\ref{lem:period_coords}. Hence $n$ solves Eq.~\eqref{eq:dtrop_coords} exactly when $f$ satisfies (a), that is, by the equivalence of (a) and (c), exactly when $[\gamma]-C^\top n=[\gamma^*]$ for a geodesic $\gamma^*$ from $x$ to $y$; then $n$ is the vector $n^*$ of Theorem~\ref{thm:exact_trop}(iii) for $\gamma^*$.

A path is determined by its chain, whose support is its set of edges, so $n^*$ determines $[\gamma^*]=[\gamma]-C^\top n^*$ and hence $\gamma^*$. The solutions of the closest vector problem are therefore in bijection with the geodesics from $x$ to $y$: the points $Qn^*\in\Lambda$, which carry the lift $t$ to the geodesic lifts $t-Qn^*=t([\gamma^*])$, are exactly the lattice points closest to $t$ in the polarization metric. This proves the last assertion of Theorem~\ref{thm:main}(ii).

\begin{remark}[Relation to the Voronoi cell of the cycle lattice]
The equivalence (a)$\Leftrightarrow$(b) is the facet description of the Voronoi cell of $H_1(\Gamma;\mathbb{Z})$: for unit edge lengths, \citet[Proposition~6]{bacher1997lattice} prove that the relevant vectors of the lattice of integral flows are exactly the oriented circuits, and \citet[Proposition~7.3]{dejong2023moments} describe the Voronoi cell for arbitrary lengths as the zonotope $\{\sum_e\beta_e\,\pi\mathbf{1}_e:|\beta_e|\le\frac12\}$. Condition (c) is where the path metric enters: among the integral lifts of a pair of points, the Voronoi cell selects exactly the geodesics.
\end{remark}

\begin{remark}[The baseline descriptor]
\label{rem:table4_scope}
The baseline evaluated in Table~\ref{tab:cvp_vs_ag_selected} turns the rounded distances into a signature as follows: it adds to $d_{\text{Babai}}(x,y)$ the backbone distance $\sum_e\ell_e\chi_e(x,y)$ over the bridges $e$ of length at least $\tau$, computed by Floyd--Warshall on the graph in which all other edges have length $0$; it converts the resulting distance matrix $D$ into the Gram matrix $-\frac12JD^{\circ2}J$ with $J=I-\frac1n\mathbf{1}\mathbf{1}^\top$; and it applies $\log(1+\cdot)$ to the absolute eigenvalues of this Gram matrix.
\end{remark}

\subsection{The spectral signature}
\label{app:spectral}

CLEARMIND summarizes $M_{AG}$ by its sorted absolute eigenvalues. The Euclidean structure of Proposition~\ref{prop:metric} pins down the sign pattern of the spectrum, so that taking absolute values loses nothing, and the variational characterization of Theorem~\ref{thm:thomson} makes the signature stable.

\begin{theorem}[Structure of the AG spectrum]
\label{thm:spectral}
Let $M=M_{AG}\in\mathbb{R}^{N\times N}$ be the AG matrix of a connected metric graph with $N\ge2$ vertices, with eigenvalues $\lambda_1\ge\dots\ge\lambda_N$.
\begin{enumerate}
\item[(i)] $M$ is a Euclidean distance matrix, and it is strictly conditionally negative definite: $v^\top Mv=-2v^\top L^+v<0$ for every nonzero $v\perp\mathbf{1}$.
\item[(ii)] $M$ is nonsingular with exactly one positive eigenvalue: $\lambda_1>0>\lambda_2\ge\dots\ge\lambda_N$. Moreover $\operatorname{tr}M=0$, so $\lambda_1=\sum_{i\ge2}|\lambda_i|$ and
\begin{equation*}
\operatorname{sort}_\downarrow|\operatorname{spec}M|=(\lambda_1,-\lambda_N,-\lambda_{N-1},\dots,-\lambda_2);
\end{equation*}
in particular the absolute spectrum determines the spectrum.
\item[(iii)] The spectrum is invariant under relabelling of the vertices, under rigid motions and reflections of the embedding $\iota$ (it depends on $\iota$ only through edge lengths and root-to-leaf distances), and under subdivision of the reconstruction (Remark~\ref{rem:order}). It is homogeneous of degree one in the lengths: scaling all lengths by $c>0$ scales the spectrum by $c$.
\end{enumerate}
\end{theorem}

\begin{proof}
(i) Proposition~\ref{prop:metric} gives $M(x,y)=\|b_x-b_y\|_2^2$, which is the definition of a Euclidean distance matrix. Writing
\begin{equation*}
M=\operatorname{diag}(L^+)\mathbf{1}^\top+\mathbf{1}\operatorname{diag}(L^+)^\top-2L^+
\end{equation*}
and taking $v\perp\mathbf{1}$, the rank-one terms vanish and $v^\top Mv=-2v^\top L^+v$, which is negative for $v\neq0$ because $L^+$ is positive definite on $\mathbf{1}^\perp$ (Lemma~\ref{lem:laplacian}). The non-strict inequality is the classical characterization of Euclidean distance matrices due to \citet{schoenberg1935remarks}.

(ii) Suppose $\lambda_2\ge0$. The span of orthonormal eigenvectors for $\lambda_1,\lambda_2$ is two-dimensional and $\mathbf{1}^\perp$ has dimension $N-1$, so the two subspaces share a nonzero vector $v$, for which $v^\top Mv\ge\lambda_2\|v\|^2\ge0$, contradicting (i). Hence $0>\lambda_2\ge\dots\ge\lambda_N$. The diagonal of $M$ vanishes, so
\begin{equation*}
\sum_i\lambda_i=\operatorname{tr}M=0
\quad\text{and}\quad
\lambda_1=-\sum_{i\ge2}\lambda_i=\sum_{i\ge2}|\lambda_i|>0;
\end{equation*}
no eigenvalue vanishes, so $M$ is nonsingular. Finally $\lambda_1\ge|\lambda_i|$ for every $i$ and $0<-\lambda_2\le\dots\le-\lambda_N$, which is the displayed order. For $N\ge3$ the inequality $\lambda_1>|\lambda_i|$ is strict for every $i\ge2$, and for $N=2$ the spectrum is $\{\pm M(x,y)\}$.

(iii) For a permutation matrix $P$, the relabelled graph has Laplacian $PLP^\top$ and pseudoinverse $PL^+P^\top$, so its AG matrix is $PMP^\top$, which is similar to $M$. The edge lengths $\|\iota(u)-\iota(v)\|_2$ and the chord lengths $a_i$ are invariant under the Euclidean group, and the thresholds $\epsilon,\tau$ are compared with these same quantities. Subdivision invariance is Corollary~\ref{cor:signature}(ii) below, whose proof uses only Corollary~\ref{cor:kron_tree}. Finally, the Laplacian for the lengths $c\ell$ is $c^{-1}$ times the Laplacian for $\ell$, so its pseudoinverse and the AG matrix are multiplied by $c$.
\end{proof}

\begin{remark}[Which matrix Theorem~\ref{thm:spectral} covers]
\label{rem:which_matrix}
The theorem applies verbatim to the exact quotient $\tilde\Gamma_\epsilon$ with vertex set $\tilde V$ ($N=|\tilde V|\ge2$ whenever at least two classes remain), and to the matrix of Algorithm~\ref{alg:ag_feature_extraction} on $V_{\text{core}}$ in exact arithmetic, which is the AG matrix of $\Gamma_\epsilon$ with length $10^{-9}$ on each contracted bridge. If the exact quotient is instead written on $V_{\text{core}}$, the matrix is $SMS^\top$, where $S\in\{0,1\}^{V_{\text{core}}\times\tilde V}$ sends each vertex to its class; completing $S$ to an invertible matrix and applying Sylvester's law of inertia shows that $SMS^\top$ has one positive, $|\tilde V|-1$ negative and $|V_{\text{core}}|-|\tilde V|$ zero eigenvalues.

More generally, for every nonzero Euclidean distance matrix the proof above with $\le$ in place of $<$ shows that exactly one eigenvalue is positive, that $\lambda_1=\sum_{i\ge2}|\lambda_i|$, and that $\operatorname{sort}_\downarrow|\operatorname{spec}M|$ has the displayed form.
\end{remark}

\begin{proposition}[Stability under perturbations of edge lengths]
\label{prop:stability}
Fix the combinatorial structure (the graph $\Gamma_\epsilon$ and the set of contracted bridges), let $\ell,\ell'$ be two length functions on it, and write $L_\ell$, $M_\ell$ and $d_\ell$ for the Laplacian, the AG distance and the path metric of the resulting quotient graph. Then for all $x,y$,
\begin{equation*}
|M_\ell(x,y)-M_{\ell'}(x,y)|\le\|\ell-\ell'\|_1,\qquad |d_\ell(x,y)-d_{\ell'}(x,y)|\le\|\ell-\ell'\|_1 .
\end{equation*}
Moreover, $\ell\mapsto M_\ell(x,y)$ is nondecreasing in every coordinate and concave. For the signatures $\mathcal{S}=\operatorname{TopK}(\operatorname{sort}_\downarrow|\operatorname{spec}M_\ell|)$ and $\mathcal{S}'$ defined from $\ell'$ on $N\ge2$ vertices,
\begin{equation*}
\|\mathcal{S}-\mathcal{S}'\|_\infty\le\|M_\ell-M_{\ell'}\|_2\le(N-1)\,\|\ell-\ell'\|_1 ,
\end{equation*}
and the three quantities coincide when the graph is a single edge.
\end{proposition}

\begin{proof}
The set $\mathcal{F}(x,y)$ does not depend on the lengths, and $M_\ell(x,y)=\min_{f\in\mathcal{F}(x,y)}\sum_e\ell_ef_e^2$ by Theorem~\ref{thm:thomson}. As a pointwise minimum of functions that are linear and nondecreasing in $\ell$, it is concave and nondecreasing. Let $f^*$ be the minimizer for $\ell$. Then
\begin{equation*}
\begin{split}
M_{\ell'}(x,y)&\le\sum_e\ell'_e(f^*_e)^2\\
&=M_\ell(x,y)+\sum_e(\ell'_e-\ell_e)(f^*_e)^2\\
&\le M_\ell(x,y)+\|\ell-\ell'\|_1,
\end{split}
\end{equation*}
where the last step uses $|f^*_e|\le1$ (Lemma~\ref{lem:current_bound}). Exchanging the roles of $\ell$ and $\ell'$ gives the first bound. The same argument applied to $d_\ell(x,y)=\min_\gamma\sum_{e\in\gamma}\ell_e$, a minimum of linear functions with coefficients in $\{0,1\}$, gives the second.

For the spectra, let $\lambda_i$ and $\lambda_i'$ be the eigenvalues of $M_\ell$ and $M_{\ell'}$ in decreasing order. By Theorem~\ref{thm:spectral}(ii), $\mathcal{S}$ consists of the first $K$ entries of $(\lambda_1,-\lambda_N,\dots,-\lambda_2)$, zero-padded, and likewise for $\mathcal{S}'$; hence every entry of $\mathcal{S}-\mathcal{S}'$ is $0$ or $\pm(\lambda_i-\lambda_i')$ for some $i$, and Weyl's inequality $|\lambda_i-\lambda_i'|\le\|M_\ell-M_{\ell'}\|_2$ gives the first inequality. The matrix $M_\ell-M_{\ell'}$ is symmetric with zero diagonal, so its spectral norm is at most its largest absolute row sum, hence at most $(N-1)\max_{x,y}|M_\ell(x,y)-M_{\ell'}(x,y)|$, and the entrywise bound gives the second inequality.

For a single edge, $M_\ell=\ell\left(\begin{smallmatrix}0&1\\1&0\end{smallmatrix}\right)$ has spectrum $\{\pm\ell\}$ and $\mathcal{S}=(\ell,\ell,0,\dots)$, so all three quantities equal $|\ell-\ell'|$.
\end{proof}

\begin{corollary}[Invariance, stability and cost of the signature]
\label{cor:signature}
Let $\mathcal{T}=(V,E,\iota)$ be a spatial tree with root $v_{\text{root}}$, fix $\epsilon,\tau>0$ and $K$, and let $\mathcal{S}_\Gamma$ be the signature of $\mathcal{T}$, computed from the AG matrix of the exact quotient $\tilde\Gamma_\epsilon$ or from the matrix of Algorithm~\ref{alg:ag_feature_extraction} in exact arithmetic (Remark~\ref{rem:which_matrix}); every statement below holds for both.
\begin{enumerate}
\item[(i)] \emph{Vertex order and rigid motions.} $\mathcal{S}_\Gamma$ is unchanged if the vertices of $\mathcal{T}$ are relabelled, or if $\iota$ is replaced by $A\circ\iota$ for an isometry $A$ of $\mathbb{R}^3$, that is, a rigid motion, possibly composed with a reflection.
\item[(ii)] \emph{Subdivision.} Let $\mathcal{T}'$ be obtained from $\mathcal{T}$ by inserting a sample $w$ with $\iota(w)=(1-s)\,\iota(u)+s\,\iota(v)$, $0<s<1$, into an edge $\{u,v\}$ of $\mathcal{T}_{\text{LCC}}$. Then $\mathcal{T}'$ and $\mathcal{T}$ have the same skeleton $\Gamma_{\min}$, the same graphs $\Gamma_\epsilon$ and $\tilde\Gamma_\epsilon$, and the same signature; the same holds after any finite sequence of such insertions.
\item[(iii)] \emph{Lipschitz stability.} Let $\ell,\ell'$ be two length functions on the edges of $\mathcal{T}_{\text{LCC}}$ and $a,a'$ two families of chord lengths, indexed by the leaves $v\neq v_{\text{root}}$, for which the same leaves are joined to the root and the same bridges are contracted, and let $\ell_\epsilon,\ell'_\epsilon$ be the resulting edge lengths of $\Gamma_\epsilon$ and $\mathcal{S}_\Gamma,\mathcal{S}'_\Gamma$ the resulting signatures. Then
\begin{equation*}
\begin{split}
\|\mathcal{S}_\Gamma-\mathcal{S}'_\Gamma\|_\infty
&\le(|V_{\text{core}}|-1)\,\|\ell_\epsilon-\ell'_\epsilon\|_1\\
&\le(|V_{\text{core}}|-1)\big(\|\ell-\ell'\|_1+\|a-a'\|_1\big).
\end{split}
\end{equation*}
\item[(iv)] \emph{Cost without lattice search.} $\mathcal{S}_\Gamma$ and the matrix $D_{\tilde\Gamma}-M_{AG}$ of squared tropical polarization distances between the vertices of $\tilde\Gamma_\epsilon$, where $D_{\tilde\Gamma}$ is the matrix of path distances of $\tilde\Gamma_\epsilon$, are obtained from $\mathcal{T}_{\text{LCC}}$ with $\mathcal{O}(|V|+|V_{\text{core}}|^3)$ operations, none of which involves a cycle basis, the period matrix, or a lattice algorithm.
\end{enumerate}
\end{corollary}

\begin{proof}
Let $N$ be the order of the matrix whose spectrum is taken: $N=|\tilde V|\le|V_{\text{core}}|$ for the exact quotient and $N=|V_{\text{core}}|$ for Algorithm~\ref{alg:ag_feature_extraction}.

(i) A relabelling carries the root, the degrees, the leaves, the edge lengths and the chord lengths $a_v=\|\iota(v)-\iota(v_{\text{root}})\|_2$ along with the vertices, and every step of the construction is defined through these data. It therefore produces the same graphs with relabelled vertices, and the AG matrix becomes $PMP^\top$ for a permutation matrix $P$, which has the spectrum of $M$ (proof of Theorem~\ref{thm:spectral}(iii)). An isometry $A$ satisfies $\|Ap-Aq\|_2=\|p-q\|_2$, so it changes neither an edge length nor a chord length, hence none of the comparisons with $\epsilon$ and $\tau$, and every graph of the construction is unchanged.

(ii) Inserting $w$ adds one vertex to the component $\mathcal{T}_{\text{LCC}}$ and leaves the other components unchanged, so $\mathcal{T}'_{\text{LCC}}$ is $\mathcal{T}_{\text{LCC}}$ with $\{u,v\}$ subdivided at $w$. Since $\iota(w)$ lies on the segment from $\iota(u)$ to $\iota(v)$, $\ell(u,w)+\ell(w,v)=\ell(u,v)$. The vertex $w$ has degree $2$ and is not the root, and no other degree changes; hence $V_{\text{core}}$ is unchanged, and the maximal paths with interior vertices in $V_{\text{reg}}$ are those of $\mathcal{T}_{\text{LCC}}$, except that in the one containing $\{u,v\}$ this edge is replaced by $\{u,w\}$ and $\{w,v\}$, which changes neither its ends nor its length. By Corollary~\ref{cor:kron_tree}(i), $\Gamma_{\min}$ is unchanged as a metric tree on $V_{\text{core}}$.

Its leaves and the positions $\iota$ on $V_{\text{core}}$ are unchanged as well, so augmentation adds the same edges with the same lengths, and $\Gamma_\epsilon$, its bridges shorter than $\tau$, $\tilde\Gamma_\epsilon$, the matrix of Algorithm~\ref{alg:ag_feature_extraction} and $\mathcal{S}_\Gamma$ are unchanged. Induction on the number of insertions gives the last assertion.

(iii) By Corollary~\ref{cor:kron_tree}(i), the length of each edge of $\Gamma_{\min}$ is the sum of $\ell$ over the segment of $\mathcal{T}_{\text{LCC}}$ that it replaces, and every edge of $\mathcal{T}_{\text{LCC}}$ lies in exactly one such segment, obtained by extending it through vertices of $V_{\text{reg}}$ in both directions. Every other edge of $\Gamma_\epsilon$, and every edge overwritten as in Remark~\ref{rem:implementation_aug}, carries the chord length of its leaf. The triangle inequality therefore gives the second inequality.

If $N=1$, both signatures vanish; let $N\ge2$. For the exact quotient, Proposition~\ref{prop:stability} gives the first inequality with the factor $N-1\le|V_{\text{core}}|-1$. In exact arithmetic, the matrix of Algorithm~\ref{alg:ag_feature_extraction} is the AG matrix of $\Gamma_\epsilon$ with length $10^{-9}$ on each contracted bridge (Remark~\ref{rem:which_matrix}). Proposition~\ref{prop:stability}, applied to $\Gamma_\epsilon$ with no contracted bridge and to the two length functions that agree with $\ell_\epsilon$ and $\ell'_\epsilon$ off the contracted bridges and equal $10^{-9}$ on them, whose $\ell_1$ distance is at most $\|\ell_\epsilon-\ell'_\epsilon\|_1$, gives the first inequality with $N=|V_{\text{core}}|$.

(iv) By Corollary~\ref{cor:kron_tree}(i), $\Gamma_{\min}$ is obtained combinatorially: following every edge at a vertex of $V_{\text{core}}$ through vertices of $V_{\text{reg}}$ to the next vertex of $V_{\text{core}}$ traverses each edge of $\mathcal{T}_{\text{LCC}}$ at most twice and records the segments and their lengths, with $\mathcal{O}(|V|)$ operations.

Augmentation compares one chord length per leaf with $\epsilon$. The graph $\Gamma_\epsilon$ has $|V_{\text{core}}|$ vertices and fewer than $2|V_{\text{core}}|$ edges, and its bridges are the edges of $\Gamma_{\min}$ that lie on no path from the root to an augmented leaf (the paragraph preceding Corollary~\ref{cor:morph}); walking from each augmented leaf towards the root, and stopping at the first edge already marked, marks exactly the edges of $\Gamma_{\min}$ that are not bridges and traverses each edge at most once. Contracting the unmarked edges of $\Gamma_{\min}$ shorter than $\tau$ along the forest they form yields $\tilde\Gamma_\epsilon$. These steps take $\mathcal{O}(|V_{\text{core}}|)$ operations.

By Theorem~\ref{thm:exact_trop}(ii) and Corollary~\ref{cor:complexity}, the matrices $M_{AG}$ and $D_{\tilde\Gamma}-M_{AG}$ of $\tilde\Gamma_\epsilon$, and the matrix of Algorithm~\ref{alg:ag_feature_extraction}, are then computed with $\mathcal{O}(|V_{\text{core}}|^3)$ arithmetic operations, by one Floyd--Warshall pass on $\tilde\Gamma_\epsilon$ and, for each Laplacian $L$ involved, which is that of a connected graph on $n\le|V_{\text{core}}|$ vertices, one inversion of the positive definite matrix $L+\frac1n\mathbf{1}\mathbf{1}^\top$, since
\begin{equation*}
L^+=\Bigl(L+\frac1n\mathbf{1}\mathbf{1}^\top\Bigr)^{-1}-\frac1n\mathbf{1}\mathbf{1}^\top .
\end{equation*}

The spectrum of the symmetric $N\times N$ matrix is computed by a dense symmetric eigendecomposition with $\mathcal{O}(N^3)$ floating-point operations, and sorting and truncation add $\mathcal{O}(N\log N)$. The total is $\mathcal{O}(|V|+|V_{\text{core}}|^3)$. No step forms a cycle basis or the period matrix, or solves a lattice problem; by Theorem~\ref{thm:closed_form}, every entry $M_{AG}(x,y)$ so obtained nevertheless equals $\ell(\gamma)-t_\gamma^\top Q^{-1}t_\gamma$ for every path $\gamma$ from $x$ to $y$.
\end{proof}

\begin{remark}[Stability and topology]
Proposition~\ref{prop:stability} concerns perturbations of lengths at fixed topology, which is the relevant model for noise in reconstructed node positions and for resampling that changes polyline lengths. Crossing a threshold, when a tip enters the $\epsilon$-ball or a bridge falls below $\tau$, changes the combinatorial structure; Propositions~\ref{prop:rank_one} and~\ref{prop:bridge} give the exact size of these jumps.
\end{remark}

\subsection{Distance-regular graphs}
\label{app:drg}

The graph-isomorphism experiments of Section~\ref{subsec:brec} evaluate the AG spectrum on graphs with unit edge lengths. The next proposition shows that on distance-regular graphs the AG distance is an explicit, strictly increasing function of the shortest-path distance. It explains the behaviour of the AG spectrum in the \textit{Regular} category of BREC.

\begin{proposition}[AG distance on distance-regular graphs]
\label{prop:drg}
Let $G$ be a connected distance-regular graph on $n$ vertices with diameter $D$, intersection array $\{b_0,\dots,b_{D-1};c_1,\dots,c_D\}$ and unit edge lengths. Let $k_j=b_0b_1\cdots b_{j-1}/(c_1c_2\cdots c_j)$ be the number of vertices at distance $j$ from any given vertex, and let $A_j$ ($0\le j\le D$) be the distance-$j$ matrices. Then $M_{AG}(x,y)=r_{\operatorname{dist}(x,y)}$ for all $x,y$, where
\begin{equation}
\label{eq:drg_resistance}
r_j=\frac{2}{n}\sum_{i=0}^{j-1}\frac{k_{i+1}+k_{i+2}+\dots+k_D}{k_i\,b_i}\qquad(0\le j\le D).
\end{equation}
In particular $0=r_0<r_1<\dots<r_D$, so $M_{AG}$ and the shortest-path metric determine each other. Consequently, if $G$ and $H$ are distance-regular with the same intersection array (for example, connected strongly regular graphs with the same parameters), then:
\begin{enumerate}
\item[(i)] $M_{AG}=\sum_jr_jA_j$ is the same function of the distance matrix in $G$ and in $H$, and the AG spectra of $G$ and $H$ coincide, multiplicities included;
\item[(ii)] the folklore $2$-dimensional Weisfeiler--Leman refinement ($2$-FWL) initialized with the pair colours $M_{AG}(x,y)$ colours every pair $(x,y)$, at every round, by the same function of $\operatorname{dist}(x,y)$ in $G$ and in $H$; hence neither $2$-FWL nor colour refinement with the pair colours $M_{AG}(x,y)$ separates $G$ from $H$.
\end{enumerate}
\end{proposition}

\textit{Proof idea.} Symmetry forces the potential of a unit source to be radial, and the current through each distance sphere is then fixed by counting edges; Eq.~\eqref{eq:drg_resistance} records the resulting potential drops, a formula that goes back to \citet{biggs1993potential}.

\begin{proof}
The Bose--Mesner algebra $\mathcal{A}=\operatorname{span}\{A_0,\dots,A_D\}$ is closed under multiplication and contains $I=A_0$ and $J=\sum_jA_j$. Its structure constants are the intersection numbers $p^h_{ij}$, the number of vertices $z$ with $\operatorname{dist}(x,z)=i$ and $\operatorname{dist}(z,y)=j$ for any pair with $\operatorname{dist}(x,y)=h$, and these are determined by the intersection array \citep[Section~4.1]{brouwer1989distance}. The Laplacian $L=b_0I-A_1$ lies in $\mathcal{A}$, and so does $L^+$, because the pseudoinverse of a symmetric matrix is a polynomial in that matrix (interpolate $\lambda\mapsto\lambda^{-1}$ on the nonzero eigenvalues and $0\mapsto0$). Write $L^+=\sum_j\eta_jA_j$.

Fix $x$ and put $u=L^+\delta_x$. Then $u(y)=\eta_{\operatorname{dist}(x,y)}$, and $Lu=\delta_x-\frac1n\mathbf{1}$ by Lemma~\ref{lem:laplacian}. For $0\le i<D$ let $B_i$ be the set of vertices at distance at most $i$ from $x$. Summing $(Lu)(y)=\sum_{z\sim y}(u(y)-u(z))$ over $y\in B_i$, the terms of edges inside $B_i$ cancel, and the edges leaving $B_i$ are the $k_ib_i$ edges from distance $i$ to distance $i+1$. Therefore
\begin{equation*}
\begin{split}
k_ib_i\,(\eta_i-\eta_{i+1})&=\sum_{y\in B_i}(Lu)(y)\\
&=1-\frac{|B_i|}{n}=\frac{k_{i+1}+\dots+k_D}{n}>0 .
\end{split}
\end{equation*}
For $\operatorname{dist}(x,y)=j$ it follows that
\begin{align*}
M_{AG}(x,y)&=L^+_{x,x}+L^+_{y,y}-2L^+_{x,y}=2(\eta_0-\eta_j)\\
&=2\sum_{i<j}(\eta_i-\eta_{i+1})=r_j,
\end{align*}
a sum of positive terms.

(i) Thus $M_{AG}=\sum_jr_jA_j\in\mathcal{A}$, with coordinates determined by the intersection array. The eigenvalues of an element of $\mathcal{A}$, with their multiplicities, are determined by its coordinates and the eigenmatrices of $\mathcal{A}$, which depend only on the intersection array.

(ii) Initially the colour of $(x,y)$ is $F_0(\operatorname{dist}(x,y))$ with $F_0(j)=r_j$. Suppose that after $\nu$ rounds the colour of every pair of $G$ and of $H$ is $F_\nu(\operatorname{dist}(x,y))$ for one function $F_\nu$. In round $\nu+1$ a pair with $\operatorname{dist}(x,y)=h$ receives its colour $F_\nu(h)$ and the multiset
\begin{equation*}
\{\!\{(F_\nu(\operatorname{dist}(x,z)),F_\nu(\operatorname{dist}(z,y))):z\in V\}\!\},
\end{equation*}
which is the image under $F_\nu\times F_\nu$ of the multiset containing each $(i,j)$ exactly $p^h_{ij}$ times. This depends only on $h$ and on the intersection array, so the new colour is $F_{\nu+1}(h)$ with $F_{\nu+1}$ common to $G$ and $H$. At every round both colour histograms consist of the colours $F_\nu(j)$ with multiplicities $nk_j$, and they coincide.

For colour refinement,
\begin{equation*}
\chi^{\nu+1}(x)=\big(\chi^\nu(x),\{\!\{(\chi^\nu(y),M_{AG}(x,y)):y\in V\}\!\}\big)
\end{equation*}
with uniform $\chi^0$, the same induction applies: if $\chi^\nu$ is uniform with value $\chi$, the multiset received by any $x$ consists of $(\chi,r_j)$ with multiplicity $k_j$, so $\chi^{\nu+1}$ is uniform and equal in $G$ and $H$.
\end{proof}

\begin{corollary}[Strongly regular graphs]
\label{cor:srg}
Let $G$ be a connected strongly regular graph with parameters $(n,k,\lambda,\mu)$ that is not complete, with unit edge lengths. Then $G$ is distance-regular with diameter $2$ and intersection array $\{k,k-1-\lambda;1,\mu\}$, and
\begin{equation*}
M_{AG}=r_1A_1+r_2A_2,\qquad r_1=\frac{2(n-1)}{nk},\qquad r_2=r_1+\frac{2}{n\mu}.
\end{equation*}
Consequently Proposition~\ref{prop:drg}(i)--(ii) apply to every pair of such graphs with the same parameters, in particular to pairs that satisfy the $4$-vertex condition.
\end{corollary}

\begin{proof}
Adjacent vertices have $\lambda$ common neighbours and distinct non-adjacent vertices have $\mu$. Since $G$ is connected and not complete, it contains a path $x\sim z\sim y$ with $x\not\sim y$; hence $\mu\ge1$, any two distinct non-adjacent vertices are at distance $2$, and $D=2$.

If $\operatorname{dist}(x,y)=1$, the neighbours of $y$ are $x$, the $\lambda$ common neighbours of $x$ and $y$, and $k-1-\lambda$ vertices at distance $2$ from $x$; if $\operatorname{dist}(x,y)=2$, exactly $\mu$ neighbours of $y$ are adjacent to $x$, and the other $k-\mu$ are at distance $2$ from $x$. These numbers depend only on $\operatorname{dist}(x,y)$, which is distance-regularity with $b_0=k$, $b_1=k-1-\lambda$, $c_1=1$ and $c_2=\mu$, and the path $x\sim z\sim y$ shows $b_1\ge1$. Hence $k_1=k$ and $k_2=kb_1/\mu=n-1-k$, and Eq.~\eqref{eq:drg_resistance} gives
\begin{equation*}
\begin{split}
r_1&=\frac2n\cdot\frac{k_1+k_2}{k_0\,b_0}=\frac{2(n-1)}{nk},\\
r_2-r_1&=\frac2n\cdot\frac{k_2}{k_1\,b_1}=\frac{2}{n\mu}.
\end{split}
\end{equation*}
Two such graphs with the same parameters have the same intersection array, which gives the last assertion.
\end{proof}

\begin{remark}
The effective-resistance distance is bounded by the $2$-FWL (equivalently, $3$-WL) test \citep{zhang2023rethinking}; Proposition~\ref{prop:drg}(ii) makes this bound explicit on the family on which $3$-WL fails in BREC. For example, the $4\times4$ rook's graph and the Shrikhande graph are strongly regular with parameters $(16,6,2,2)$, hence distance-regular with intersection array $\{6,3;1,2\}$ and $k_1=6$, $k_2=9$. Eq.~\eqref{eq:drg_resistance} gives
\begin{align*}
r_1&=\frac{2}{16}\cdot\frac{15}{6}=\frac{5}{16},\\
r_2&=r_1+\frac{2}{16}\cdot\frac{9}{18}=\frac38,
\end{align*}
so $M_{AG}=\frac38(J-I)-\frac1{16}A_1$ in both graphs, and the adjacency eigenvalues $6,2^{(6)},(-2)^{(9)}$ give the common AG spectrum $\{21/4,(-1/2)^{(6)},(-1/4)^{(9)}\}$. The matrix $M_{AG}$ itself determines each graph, since the adjacency matrix is the indicator of $M_{AG}=r_1$; what coincides is every invariant that $2$-FWL can compute from it, in particular the spectrum.
\end{remark}

\begin{remark}[Eigenvector encodings]
\label{rem:eigvec}
Proposition~\ref{prop:drg}(ii) bounds every invariant that $2$-FWL computes from the pair colours $M_{AG}(x,y)$. AG-GIN (Appendix~\ref{app:brec_protocol}) uses more: its node features are the entrywise absolute values of individual eigenvectors of $M_{AG}$. On a distance-regular graph $M_{AG}=\sum_jr_jA_j$ lies in the Bose--Mesner algebra, so each of its eigenspaces is a sum of common eigenspaces of that algebra; when such an eigenspace has dimension larger than one, as for the multiplicities $6$ and $9$ in the example above, its individual eigenvectors are not determined by $M_{AG}$, and features built from them are not functions of the pair colours. The bound of Proposition~\ref{prop:drg}(ii) therefore does not apply to AG-GIN, and on the \textit{Regular} category it separates $62.1\%$ of the pairs, against $35.7\%$ for the $3$-WL test (Table~\ref{tab:brec_results}); every pair counted there passes the reliability check of the Reliable Paired Comparison, which requires relabelled copies of the same graph to remain unseparated.
\end{remark}

\section{Experimental Protocols}
\label{app:protocols}

\subsection{Implementation of the signature}
\label{app:numerics}

Reconstructions are read from SWC files; additional roots are attached to the first root, and the largest connected component is kept. Edge lengths are Euclidean distances, clipped below at $10^{-5}\,\mu$m. Kron reduction eliminates the degree-$2$ non-root vertices one at a time (Corollary~\ref{cor:kron_tree}(iii) shows that the order is irrelevant). Augmentation joins every leaf within $\epsilon$ of the root to the root by an edge of Euclidean length. Bridges are detected on the augmented graph, and those shorter than $\tau$ receive the conductance $10^9$. The pseudoinverse is computed with \texttt{scipy.linalg.pinvh}; Proposition~\ref{prop:bridge}(v) also allows the quotient to be formed combinatorially, by contracting the bridges and inverting the contracted Laplacian. Eigenvalues are computed with \texttt{scipy.linalg.eigh}; their absolute values are sorted in decreasing order, truncated to $K=64$ and zero-padded when $|V_{\text{core}}|<K$.

\subsection{Evaluation protocol}
\label{app:eval_protocol}

All models use the classes listed in Table~\ref{tab:reprod_datasets} and the fold assignment of \citet{chen2022treemoco}: folds $0$--$7$ for training and folds $8$--$9$ for evaluation. Accuracy on folds $8$--$9$ is recorded after every epoch, training stops after $25$ epochs without improvement ($250$ for the spectrum-only MLP), and the best epoch is reported; Table~\ref{tab:main_results} gives the mean and standard deviation over the seeds $\{42,114514,1919810,666,1453\}$. Table~\ref{tab:per_seed} lists the individual runs of the Tree-LSTM and MorphoGNN fusion models: the five runs of Tree-LSTM + AG reach at least $67.37\%$ on ACT-4 and $92.98\%$ on BIL-6, above every baseline entry of Table~\ref{tab:main_results}, and the five runs of MorphoGNN + AG reach at least $75.32\%$ on JML-4, more than ten points above the MorphoGNN baseline ($65.19\%$).

\begin{table}[h]
\caption{Individual runs of the Tree-LSTM and MorphoGNN fusion models of Table~\ref{tab:main_results}: accuracy (\%) at the best epoch on folds $8$--$9$ for each seed. The last column is the mean and population standard deviation of the five runs, computed from the unrounded accuracies.}
\label{tab:per_seed}
\centering
\small
\setlength{\tabcolsep}{4pt}
\begin{tabular}{llcccccc}
\toprule
& & \multicolumn{5}{c}{\textbf{Seed}} & \\
\cmidrule(lr){3-7}
\textbf{Model} & \textbf{Dataset} & $42$ & $114514$ & $1919810$ & $666$ & $1453$ & \textbf{Mean}~$\pm$~\textbf{std} \\
\midrule
\multirow{3}{*}{Tree-LSTM + AG} & ACT-4 & 67.37 & 67.37 & 76.84 & 68.42 & 68.42 & $69.68\pm3.61$ \\
& JML-4 & 72.73 & 74.03 & 71.43 & 75.32 & 77.92 & $74.29\pm2.23$ \\
& BIL-6 & 93.39 & 93.39 & 94.63 & 92.98 & 93.39 & $93.55\pm0.56$ \\
\midrule
\multirow{3}{*}{MorphoGNN + AG} & ACT-4 & 55.79 & 51.58 & 52.63 & 52.63 & 58.95 & $54.32\pm2.71$ \\
& JML-4 & 76.62 & 77.92 & 76.62 & 75.32 & 75.32 & $76.36\pm0.97$ \\
& BIL-6 & 91.32 & 92.56 & 92.56 & 92.15 & 92.15 & $92.15\pm0.45$ \\
\bottomrule
\end{tabular}
\end{table}

The frozen $k$-NN experiments (Tables~\ref{tab:cvp_vs_ag_selected} and~\ref{tab:ablation_study}) use the cosine metric with $k=20$, $5$ and $20$ on ACT-4, JML-4 and BIL-6.

\subsection{BREC protocol}
\label{app:brec_protocol}

\textbf{AG-GIN.} Each BREC graph is treated as a metric graph with unit edge lengths. Node features are a constant together with the entrywise absolute values of the $16$ eigenvectors of $M_{AG}$ with the largest eigenvalues (zero-padded for smaller graphs). The network is a GIN \citep{xu2018how} with $4$ layers of width $32$, batch normalization, and jumping-knowledge readout \citep{xu2018jk} to $16$ dimensions. For each pair it is trained for $30$ epochs with Adam (learning rate $10^{-3}$, weight decay $10^{-5}$, batch size $16$) and a cosine-embedding loss on the $32$ relabeled copies of the two graphs, and it is evaluated with the Reliable Paired Comparison of \citet{wang2024brec}: Hotelling's statistic $T^2=q\,\bar d^\top S^{-1}\bar d$ with $q=32$ \citep[Eq.~(4)]{wang2024brec}, threshold $72.34$, and the reliability check on relabeled copies of the same graph. For each pair we report the best of the ten seeds $100,200,\dots,1000$.

\textbf{AG spectrum.} We compute $L^+$ exactly from the eigendecomposition of $L$, with the dimension of the null space set to the number of connected components, form $M_{AG}$, and compare the sorted absolute spectra of the two graphs of a pair. The spectra of two relabelings of the same graph agree to a relative accuracy of $2.0\cdot10^{-14}$, while the two graphs of every separated pair differ by at least $3.0\cdot10^{-8}$ in relative terms; any threshold between $10^{-12}$ and $10^{-8}$ yields the same $217$ separated pairs (Basic 60, simple regular 50, strongly regular 0, Extension 100, CFI 7, $4$-vertex condition 0, distance regular 0). These three zeros are forced by Proposition~\ref{prop:drg}(i): in every strongly regular, $4$-vertex-condition and distance-regular pair the two graphs are distance-regular with the same intersection array, so their AG spectra coincide. The computation is deterministic and takes about a minute on a CPU.

\subsection{Reproduction of the GraPHFormer graph branch}
\label{appendix:graphformer_reproduction}

We reproduce the graph branch of GraPHFormer \citep{shah2026graphformer}, the TreeLSTM-Double encoder, trained from random initialization with a classification head on our folds; the persistence-image generator of the vision branch is not part of the public release. The raw SWC trees follow the preprocessing of the public release: the soma is translated to the origin, the arbor is aligned by PCA, axonal compartments are removed, and the tree is contracted to a branch-level tree whose nodes carry five features (soma-centred coordinates, branch path length and contraction ratio). The encoder projects these features to $256$ dimensions, reads out the root state, and classifies with dropout $0.5$ and batch normalization. We train with AdamW (learning rate $10^{-4}$, weight decay $10^{-2}$), \texttt{CosineAnnealingWarmRestarts} ($T_0=25$, $\eta_{\min}=10^{-6}$) and gradient clipping at $1.0$. The data module is built on the TreeMoCo data structures. Table~\ref{tab:reprod_datasets} lists the datasets and Table~\ref{tab:reprod_results} compares the reproduction with the published tree-encoder and full-model numbers.

\begin{table}[h]
\caption{Datasets and splits of the GraPHFormer reproduction (training sets are class-balanced by oversampling).}
\label{tab:reprod_datasets}
\centering
\small
\begin{tabular}{llccc}
\toprule
\textbf{Dataset} & \textbf{Source} & \textbf{Classes} & \textbf{Train (balanced)} & \textbf{Test} \\
\midrule
ACT-4 & Allen Cell Types & Isocortex L2/3, L4, L5, L6 & 576 ($144\times4$) & 95 \\
JML-4 & Janelia MouseLight & Isocortex L2/3, L5, L6, VPM & 672 ($168\times4$) & 75 \\
BIL-6 & BICCN fMOST & CP, VPM, Isocortex L2/3, L4, L5, L6 & 1824 ($304\times6$) & 242 \\
\bottomrule
\end{tabular}
\end{table}

\begin{table}[h]
\caption{GraPHFormer graph branch reproduced on our folds versus published results (70/30 split).}
\label{tab:reprod_results}
\centering
\small
\begin{tabular}{lccc}
\toprule
\textbf{Dataset} & \textbf{Graph branch (ours, 80/20)} & \textbf{Published tree encoder} & \textbf{Published full model} \\
\midrule
ACT-4 & $54.52\pm1.37$ & 53.5 & 65.5 \\
JML-4 & $73.33\pm2.67$ & 63.7 & 76.5 \\
BIL-6 & $82.97\pm0.54$ & 88.3 & 93.51 \\
\bottomrule
\end{tabular}
\end{table}

\subsection{Positional probe of SGTMorph}
\label{appendix:sgtmorph_probe}

SGTMorph \citep{sheng2025sgtmorph} rescales the node coordinates $P\in\mathbb{R}^{N\times3}$ of each neuron into a unit box before they enter the network,
\begin{equation*}
P_{\text{norm}}=\frac{P-m}{M-m},\qquad m=\min(P),\quad M=\max(P),
\end{equation*}
where $m$ and $M$ are single scalars over the whole coordinate array. This isotropic transformation preserves the shape and aspect ratio of the morphology and removes absolute scale and the offset $m$ shared by the three axes. It keeps one residual positional signal: the location of the soma inside the cube $[m,M]^3$ that contains its arbor. To measure this signal, we train a standalone classifier on the normalized soma coordinates alone. Table~\ref{tab:sgtmorph_probe} reports the result: the probe reaches $61.0\%$ on JML-4 and $74.8\%$ on BIL-6, so the relative soma position is an informative orientation prior, and the published full-model accuracies ($72.4\%$ and $88.9\%$) show that SGTMorph extracts structure beyond it. Table~\ref{tab:main_results} quotes the published SGTMorph accuracies on these two datasets, the ones on which the probe is evaluated.

\begin{table}[h]
\caption{Positional probe on the min--max normalized soma coordinates of SGTMorph.}
\label{tab:sgtmorph_probe}
\centering
\small
\begin{tabular}{lcc}
\toprule
\textbf{Dataset} & \textbf{Positional probe (soma only)} & \textbf{Published full model} \\
\midrule
JML-4 & 61.0 & 72.4 \\
BIL-6 & 74.8 & 88.9 \\
\bottomrule
\end{tabular}
\end{table}

\subsection{Fusion with MorphoGNN}
\label{appendix:integration_protocols}

The spectrum is computed on the augmented and quotiented skeleton ($\epsilon=20\,\mu$m, $\tau=5\,\mu$m). The point cloud is built from the raw node coordinates, translated to the soma and scaled to the unit ball by the maximum radius. Each point cloud has $1024$ points: denser neurons are subsampled uniformly without replacement, and sparser ones are oversampled with replacement. Isotropic Gaussian jitter of standard deviation $10^{-3}$ is added to every subsampled training cloud and to every oversampled cloud; on the latter it keeps the $k$-NN graphs of the dynamic convolutions well defined. Sampling and jitter are drawn once per seed. The spatial branch is a five-layer EdgeConv network \citep{wang2019dgcnn} with $k=16$ and channels $6\to32\to64\to128\to256\to1024$; max and mean pooling give $\mathbf{v}_{\text{pc}}\in\mathbb{R}^{2048}$. The spectral branch $P_{\text{spec}}$ is a linear layer to $\mathbb{R}^{128}$ followed by batch normalization and ReLU. The concatenation $\mathbf{v}_{\text{pc}}\oplus P_{\text{spec}}(\mathbf{s})\in\mathbb{R}^{2176}$ is classified by an MLP $2176\to512\to256\to C$ with batch normalization, LeakyReLU ($\alpha=0.2$) and dropout $0.5$. We minimize the cross-entropy with AdamW (learning rate $10^{-3}$, weight decay $10^{-4}$) \citep{loshchilov2019decoupled}, cosine annealing ($\eta_{\min}=10^{-6}$) over $150$ epochs, gradient clipping at $1.0$ and early-stopping patience $25$.

\subsection{Fusion with a Tree-LSTM}
\label{appendix:treelstm_integration}

The Tree-LSTM runs on the branching skeleton $\Gamma_{\min}$ of Theorem~\ref{thm:kron}, oriented from the soma by breadth-first search. Each vertex carries $(x,y,z,\text{type},\text{degree})$: its coordinates translated to the soma and scaled to the unit ball, its SWC compartment type, and its degree in $\Gamma_{\min}$ divided by $10$. A Child-Sum Tree-LSTM \citep{Tai2015} with hidden size $128$ processes the vertices level by level from the leaves to the root; vertices of equal height across a batch are stored contiguously, children's states are routed to parents by scatter-add, and roots point to a dummy slot so that execution is branch free. The root state is mapped by a linear layer, layer normalization and ReLU to $\mathbb{R}^{64}$, the spectrum by $P_{\text{spec}}$ (linear, batch normalization, ReLU) to $\mathbb{R}^{64}$, and their concatenation by a linear layer $\mathbb{R}^{128}\to\mathbb{R}^{128}$, ReLU, dropout $0.3$ and a linear classifier $\mathbb{R}^{128}\to\mathbb{R}^{C}$. The layers preceding the classifier have $97{,}792$ parameters, independent of $C$. Optimization follows the MorphoGNN protocol, with batch size $128$.

\subsection{Spectrum-only MLP}
\label{appendix:mlp_integration}

The $64$-dimensional spectrum is standardized with training-set statistics (standard deviations floored at $10^{-8}$) and classified by an MLP $64\to96\to48\to C$ with LeakyReLU ($\alpha=0.1$) and dropout $0.05$, which has $11{,}092$ parameters for four classes and $11{,}190$ for six. We train with Adam (learning rate $3\cdot10^{-3}$, weight decay $10^{-6}$), cross-entropy with label smoothing $0.05$, batch size $32$, up to $800$ epochs and early-stopping patience $250$. Table~\ref{tab:mlp_results} gives the per-dataset results.

\begin{table}[h]
\caption{Spectrum-only MLP over five seeds.}
\label{tab:mlp_results}
\centering
\small
\begin{tabular}{lccc}
\toprule
\textbf{Dataset} & \textbf{Accuracy} & \textbf{Macro-F1} & \textbf{Parameters} \\
\midrule
ACT-4 & $63.79\pm1.20$ & $62.67\pm2.32$ & $11{,}092$ \\
JML-4 & $60.26\pm1.72$ & $60.00\pm1.99$ & $11{,}092$ \\
BIL-6 & $73.31\pm0.86$ & $63.13\pm2.43$ & $11{,}190$ \\
\bottomrule
\end{tabular}
\end{table}

\subsection{Fusion with MorphVAE}
\label{appendix:morphvae_integration}

MorphVAE \citep{laturnus2021morphvae} represents a neuron by $256$ soma-to-tip random walks, encodes them with an LSTM-based von Mises--Fisher variational autoencoder, and max-pools the walk embeddings into a $32$-dimensional neuron embedding. A sanity check on the official Farrow data (400/99/100 neurons, $150$ epochs) gives a $5$-NN test accuracy of $43.62\%$. For the benchmarks the autoencoder is trained without labels on the same folds ($150$ epochs, $\kappa=500$, learning rate $0.01$); for BIL-6, matching SWC file names leaves $1{,}200$ neurons ($958$ training, $242$ test). The embeddings are extracted at the best epoch, as in Appendix~\ref{app:eval_protocol}. Two downstream classifiers with identical architecture are compared: one on the $32$-dimensional embedding and one on its concatenation with the raw $64$-dimensional spectrum ($96$ dimensions). Both are two-layer MLPs with hidden width $64$, layer normalization, GELU and dropout $0.2$, trained for $150$ epochs with AdamW (learning rate $10^{-3}$, weight decay $10^{-4}$, batch size $64$). Table~\ref{tab:main_results} reports the results.

\subsection{Design choices}
\label{appendix:design_choices}

Below we list the design choices of CLEARMIND and the property that each one ensures.

\textbf{Training-free and label-free.} The signature has no trainable parameters and is computed once per neuron, without labels, with one setting ($\epsilon=20\,\mu$m, $\tau=5\,\mu$m, $K=64$) shared by all three datasets and all four models with the spectrum in Table~\ref{tab:main_results}. The MorphoGNN and MorphVAE pairs therefore compare one backbone with and without the same fixed input; the two MorphVAE classifiers differ only in the $64$ appended spectral inputs (Appendix~\ref{appendix:morphvae_integration}).

\textbf{Independent of the frame.} The signature is invariant under relabelling, rigid motions and reflections, and it reads the embedding through edge lengths and Euclidean soma-to-tip distances alone (Theorem~\ref{thm:spectral}(iii), Corollary~\ref{cor:signature}(i)); no canonical orientation or rigid alignment is needed. Orientation priors, which are informative by themselves (Appendix~\ref{appendix:sgtmorph_probe}), enter through the backbone: the Tree-LSTM and MorphoGNN fusion models join the spectrum to soma-centred coordinates.

\textbf{Cost on the branching skeleton.} The signature takes $\mathcal{O}(|V|+|V_{\text{core}}|^3)$ operations, none involving a cycle basis, a period matrix or a lattice algorithm (Corollary~\ref{cor:signature}(iv)); the median $|V_{\text{core}}|$ is $125$ and $464$ on ACT-4 and BIL-6, against a median $|V|$ of $4{,}677$ and $18{,}966$ (Section~\ref{subsec:retraction}). It is unchanged by any subdivision of the reconstruction (Corollary~\ref{cor:signature}(ii)) and cuts preprocessing time by a factor of $3.7$ to $6.5$ relative to the lattice pipeline (Table~\ref{tab:cvp_vs_ag_selected}).

\textbf{A canonical spectrum.} $M_{AG}$ is independent of the cycle basis (Theorem~\ref{thm:closed_form}), and its spectrum is independent of the vertex order, with a sign pattern that makes the absolute values lossless (Theorem~\ref{thm:spectral}(ii)). The spectrum is deterministic and training-free, and it separates $217$ BREC pairs for every threshold between $10^{-12}$ and $10^{-8}$ (Appendix~\ref{app:brec_protocol}); on distance-regular graphs $M_{AG}$ is given in closed form by the intersection array, which therefore determines the spectrum (Proposition~\ref{prop:drg}). The eigenvectors of $M_{AG}$, to which the bound of Proposition~\ref{prop:drg}(ii) does not apply (Remark~\ref{rem:eigvec}), are used as node-level encodings for trained networks, so the same matrix gives both a canonical graph invariant and node features.

\textbf{Two length scales in micrometres.} The proximity $\epsilon$ and the artefact scale $\tau$ are thresholds, so the signature is piecewise constant in them; when $\epsilon$ passes the soma distance of a tip or $\tau$ the length of a bridge, $M_{AG}$ changes by an amount given in closed form (Propositions~\ref{prop:rank_one} and~\ref{prop:bridge}). At fixed combinatorial structure the signature is Lipschitz (Proposition~\ref{prop:stability}, Corollary~\ref{cor:signature}(iii)) and homogeneous of degree one (Theorem~\ref{thm:spectral}(iii)) in the edge and chord lengths, so its first entry $\lambda_1=\sum_{i\ge2}|\lambda_i|$ (Theorem~\ref{thm:spectral}(ii)) carries the scale of the arbor and $\mathcal{S}_\Gamma/\lambda_1$ its scale-free profile.

\textbf{Beyond neurons.} Theorem~\ref{thm:main} holds for every pair of points of every compact connected metric graph and every $\mathbb{Z}$-basis of $H_1(\Gamma;\mathbb{Z})$ (Theorems~\ref{thm:closed_form} and~\ref{thm:exact_trop}, Lemma~\ref{lem:subdivision}); neuronal arbors and the graphs of BREC are two special cases of these exact results.

\subsection{Future directions}
\label{appendix:future}

The Pythagorean identity of Theorem~\ref{thm:thomson} splits every unit flow on a metric graph into a gradient part, whose energy is the AG distance, and a cyclic part, which lives on the universal cover of the tropical Jacobian. Rather than using the resulting spectrum only as an input feature, one can use the AG distance and the lifts of Theorem~\ref{thm:exact_trop} to define message-passing operators that act directly on the universal cover of the tropical Jacobian and are trained end to end. Theorems~\ref{thm:closed_form} and~\ref{thm:exact_trop} supply the ingredients of such operators: the identity $M_{AG}(x,y)=\ell(\gamma)-t_\gamma^\top Q^{-1}t_\gamma$ holds for every path $\gamma$ from $x$ to $y$, every basis of $H_1(\Gamma;\mathbb{Z})$ and every combinatorial model, and the closest lifts are exactly the geodesic ones; neither the AG distance nor a closest lift requires a lattice search.

\subsection{Compute environment}
Experiments were run on an AMD Threadripper 2970WX CPU, an NVIDIA GeForce RTX 2080 Ti GPU, and 128\,GB of RAM.

\end{document}